\documentclass{article}
\usepackage{fullpage,amsmath,amsfonts,amsthm,bbm}
\usepackage{subcaption}
\usepackage{caption}
\usepackage{float}
\usepackage{natbib}
\setcitestyle{numbers,square}
\usepackage{comment}
\usepackage{graphicx} %
\usepackage{enumitem}
\usepackage{amssymb}
\usepackage{amsmath}
\usepackage{xcolor}
\usepackage{hyperref}       %
\definecolor{linkblue}{RGB}{0,51,153}
\hypersetup{
  colorlinks=true,
  linkcolor=linkblue,   %
  citecolor=linkblue,   %
  urlcolor=linkblue  %
}
\usepackage{makecell}
\usepackage{multirow}
\usepackage{bm}
\usepackage{fullpage}
\usepackage{booktabs}    

\usepackage{thmtools,thm-restate}

\declaretheorem[name=Lemma]{lemma}

\newtheorem{cor}{Corollary}
\newtheorem{remark}{Remark}
\newtheorem{assumption}{Assumption}

\newcommand{\mc}{\mathcal} %

\newcommand{\sign}{\text{sign}}

\newcommand{\de}{\textup{d}} %

\newcommand{\sph}{\mathbb{S}} %
\newcommand{\ind}{\mathbbm{1}} %
\def\c{{\textsf{c}}} %
\def\bzero{{\mathbf{0}}} %
\def\bone{{\mathbf{1}}} %

\def\s{k} %
\def\bA{\boldsymbol{A}}
\def\bX{\boldsymbol{X}}
\def\bZ{\boldsymbol{Z}}

\def\by{\boldsymbol{y}}

\def\bI{\boldsymbol{I}}
\def\bv{\boldsymbol{v}}
\def\bw{\boldsymbol{w}}
\def\bR{\boldsymbol{R}}
\def\bS{\boldsymbol{S}}
\def\bN{\boldsymbol{N}}
\def\bmu{\boldsymbol{\mu}}
\def\bpsi{\boldsymbol{\psi}}
\def\bnu{\boldsymbol{\nu}}
\def\bu{\boldsymbol{u}}
\def\bSigma{\boldsymbol{\Sigma}}
\def\bGamma{\boldsymbol{\Gamma}}
\def\bvareps{\boldsymbol{\varepsilon}}
\def\E{\mathbb{E}} %
\def\P{\mathbb{P}} %
\def\C{{\textsf{c}}} %

\newcommand{\bbR}{\mathbb{R}}

\newcommand{\bx}{\boldsymbol{x}}
\newcommand{\wt}{\widetilde}
\newcommand{\wh}{\widehat}
\newcommand{\ee}{\mc{E}}
\newcommand{\lle}{\wh{\mc L}}
\newcommand{\wb}{\bm w^*}
\newcommand{\wpop}{\bm w^*_{\mc{L}}}
\newcommand{\gw}{g_{\bm{w}}}
\newcommand{\tw}{t_{\bm{w}}}
\newcommand{\snr}{\normso{\bmu}}
\newcommand{\geta}{\bm \Gamma_{\eta}}

\DeclareMathOperator*{\argmin}{arg\,min}

\newcommand{\norm}[1]{\left\lVert#1\right\rVert}
\newcommand{\normso}[1]{\left\lVert#1\right\rVert_{\bm \Sigma^{-1}}}
\newcommand{\norms}[1]{\left\lVert#1\right\rVert_{\bm \Sigma}}
\newcommand{\normg}[1]{\left\lVert#1\right\rVert_{\bm \Gamma}}

\title{Minimax Optimal Early-Stopped Gradient Descent for\\Gaussian Mixture Classification}
\author{Alex Buna \and Shirley Xiaoqi Liu \and Patrick Rebeschini } 
\date{}

\begin{document}

\maketitle
\makeatletter
\begingroup
\renewcommand{\thefootnote}{}
\long\def\@makefntext#1{\noindent #1}
\footnotetext{The authors are with the Department of Statistics, University of Oxford.\\
Emails: \texttt{alex.bunamarginean@spc.ox.ac.uk}, \texttt{shirley.liu@stats.ox.ac.uk}, and \texttt{patrick.rebeschini@stats.ox.ac.uk}.}
\endgroup
\makeatother

\begin{abstract}
    In overparameterised classification, training data can be linearly separable even when the underlying distribution is not. In this setting, gradient descent (GD) on the logistic loss diverges in norm while converging in direction to a max-margin interpolating classifier, whose implicit bias can be statistically suboptimal. In this work, we show that early stopping can overcome this suboptimality: in a Gaussian mixture model with label-flipping noise, GD stopped at an appropriate oracle time achieves minimax-optimal excess zero-one risk for covariance spectra with fast and continuous decay, including polynomial and exponential spectral decays. Our analysis combines a sharp upper bound for the early-stopped iterate with a matching statistical lower bound over arbitrary classifiers, yielding optimal rates that are validated by experiments. A central technical contribution is a new calibration result that converts excess logistic risk into excess zero-one risk; it handles the model misspecification induced by the label-flipping noise, and  removes the square-root rate in standard bounds. We also establish a lower bound for linear interpolators, showing that interpolation can require exponentially more samples than early stopping to achieve the same excess risk. 
\end{abstract}
\addtocontents{toc}{\protect\setcounter{tocdepth}{-1}}

\section{Introduction}
Machine learning models are routinely trained  in   overparameterised  regimes, where the number of parameters $d$ exceeds the sample size $n$. %
While interpolation \cite{Muthukumar2021classification} and early-stopping in gradient descent (GD) \cite{wu2025benefits} are both known to generalise in classification settings under suitable conditions,
their relative statistical efficiency  remains less well-understood than in regression settings. In this work, we study  classification in Gaussian mixtures with label-flipping noise, a canonical model that captures key challenges of high-dimensional classification, including misspecification under the logistic loss. In this setting, we show that early-stopped GD, with an appropriately chosen stopping time, is  \emph{minimax-optimal} in certain regimes, whereas interpolating estimators can require exponentially more samples to achieve the same excess risk.

Overparameterisation has been extensively studied in least-squares regression. Foundational studies \cite{bartlett2020benign, belkin2020two,  hastie2022surprises,tsigler2023benign,muthukumar2020harmless} show that  GD on  least-squares   converges to the minimum $\ell_2$-norm interpolator,   which, under sufficient overparameterisation   and  suitable spectral conditions on the data covariance,    achieves  vanishing excess risk as  $n\to \infty$, a phenomenon termed ``benign overfitting''. 
Complementary to these results,  early stopping for GD methods
has been shown to act as an  implicit   regulariser in overparameterised regimes,  
   closely connected to   explicit $\ell_2$-regularisation in least-squares regression \cite{yao2007early,suggala2018connecting}. Under suitable spectral and signal regularity  conditions,   appropriately early-stopped iterates  achieve  statistically optimal  rates for    oracle stopping times \cite{Buhlmann2003boosting,yao2007early,lin2017optimal,pillaud2018statistical} and, despite the added difficulty, data-dependent stopping times \cite{raskutti2014early,wei2019early,averyanov2020early,kanade2023statistical}.   %
   Most recently, it was shown in \citep{wu2025risk} that GD with a data-dependent stopping rule is minimax-optimal over a range of power-law  spectral classes  defined via the capacity and source conditions \cite{caponnetto2007optimal}.%

These  regression  results rely heavily on  closed-form solutions and   linear GD dynamics under least-squares loss. In contrast,  even for  linear classifiers, the logistic loss admits neither property, making the  theory for classification   far less developed  despite   its central role in  machine learning. 
Specifically, in overparameterised logistic regression, GD behaves fundamentally differently from least-squares regression. Its iterates diverge in   $\ell_2$-norm  while converging in direction to  the max-margin 
 on linearly-separable data  \cite{soudry2018implicit,shamir2021gradient}, and   to the same direction up to a bounded~offset~on nonseparable~data~\cite{ji2019implicit}. 
 
 As a consequence of this implicit bias,  classification analyses  \cite{chatterji2021finite, wang2022binary,  cao2021risk, hashimoto2025universality}  depart from those for least-squares regression, taking the max-margin classifier as the key object of study.  
A complementary line of work  \cite{Muthukumar2021classification, hsu2021proliferation,ardeshir2021support}    shows that, under sufficient overparameterisation where all samples become support vectors,  the max-margin classifier  coincides with the minimum $\ell_2$-norm interpolator, enabling tools from least-squares analysis  to transfer to classification \cite{Muthukumar2021classification, hsu2021proliferation, wang2022binary, tsigler2025benign}. Both strands of work characterise  sufficient spectral conditions   under which the max-margin classifier exhibits benign overfitting. 
Moreover, \cite{wang2022binary, tsigler2025benign}  explicitly compare interpolation with the  $\ell_2$-regularised classifier.  
However,  none of these  works on classification  establishes how interpolation compares to early-stopped GD 
in terms of  sample complexity, i.e., the number of samples $n$  required to achieve a target excess zero-one  risk, 
or more fundamentally,  whether interpolation or~early-stopped~GD~is~\emph{statistically~optimal}. 

This question is motivated by the phenomenon illustrated in Figure~\ref{fig: u-shape}. When running GD with the logistic loss for binary classification on linearly separable data, the population excess zero-one risk initially decreases, but then increases again as GD approaches the interpolating regime, illustrating that early stopping  improves generalisation, whereas interpolation can be statistically suboptimal. 

\begin{figure}[t]
        \centering
        \includegraphics[width=\textwidth]{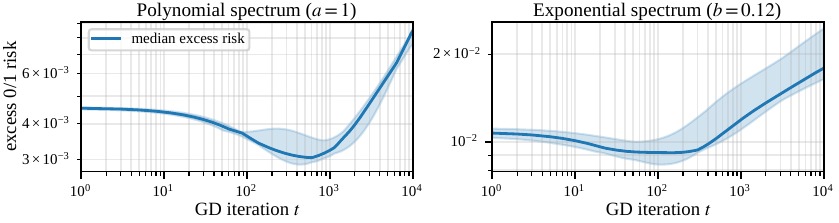}
    \caption{\textbf{Early stopping along the GD trajectory.} Excess zero-one population risk along GD iterates trained on $n=972$ samples from the $d=3888$-dimensional Gaussian mixture model $\bm X\sim \mathcal{N}(\pm\bmu,\bm \Sigma)$ with label-flip probability $p=0.1$ (see \eqref{eq:GMM}). The centres $\bmu$ satisfy $\|\bSigma^{-1}\bmu\|\lesssim 1$. Left: polynomial spectrum $\lambda_i(\bSigma)\asymp i^{-2a}$ with $a=1$. Right: exponential spectrum $\lambda_i(\bSigma)\asymp e^{-bi}$ with $b=0.12$.  The solid curve shows the median over independent repetitions, and the shaded band shows the interquartile range. In both cases, the excess risk first decreases and then increases as GD approaches  interpolation, illustrating the statistical benefit of early stopping.}
    \label{fig: u-shape}
\end{figure}

Characterising this statistical benefit of early stopping is the focus of this work. 
We study a canonical \emph{generative} model in classification: 
Gaussian mixtures  with label-flipping noise. 
In this model, the covariate  vector $\bX\!\in\!\bbR^d$ clusters around  class-conditional means $\pm\bmu\!\in\!\bbR^d$, and~the~observed~label~$Y$~is obtained by flipping a latent    label $\wt{Y}$ independently with probability $p\!\in\![0, 0.5)$.  
Formally, 
 each sample $(\bX, Y)\!\in\! \bbR^d\!\times\! \{\pm 1\}$  is generated by first drawing  a latent label $\wt{Y}\!\in\!\{\pm1\}$~uniformly,~then~setting 
\begin{align}
   \bX=\wt{Y} \bmu + \bvareps, \quad \bvareps\sim \mc{N}(\bzero, \bSigma)\,,\quad \text{and}\quad Y=\begin{cases}
       -\wt{Y}\; &\text{w.p.}\; \;p\,,\\
       \wt{Y}\; &\text{w.p.}\;\; 1-p\,,
   \end{cases}\label{eq:GMM}
\end{align}
where $\bSigma\in \bbR^{d\times d}$ is symmetric positive definite.
A key  challenge in this model is the presence of label-flipping noise, as studied  for Gaussian mixtures in the interpolation regime in  \cite{chatterji2021finite, wang2022binary, hashimoto2025universality, tsigler2025benign}. When $p>0$, the distribution of $Y\mid \bm X$ differs from the sigmoid model and is thus  \emph{misspecified} under the logistic loss. As a result, standard calibration bounds that relate excess zero-one  and logistic risks \citep{zhang2004statistical,bartlett2006convexity} yield suboptimal rates, necessitating new proof techniques.

In this work, we overcome the challenge due to misspecification and prove that, early-stopped GD on the logistic loss is \emph{minimax-optimal} for classification in the model in~\eqref{eq:GMM}  when the covariance spectrum is fast and continuously decaying (FCD), in line with the terminology used in earlier regression literature \cite[Assumption 3]{wu2025risk}. To our knowledge, this is the first result establishing the statistical optimality of early-stopped GD for Gaussian mixture classification.  The FCD class covers a range of spectral decays, including both polynomial and exponential spectra, whereas \cite{wu2025risk}  establishes    minimax rates for GD in  regression within power-law spectral classes,
which correspond to polynomial   decay. Figure~\ref{fig: rates} provides a finite-sample validation of our theory: it plots the empirical excess risk as a function of the sample size under both polynomial and exponential spectra, showing close agreement with the minimax rates  predicted by our  theory,  summarised~in~Table~\ref{tab:summary}. 

\begin{figure}[t]

        \centering
        \includegraphics[width=\textwidth]{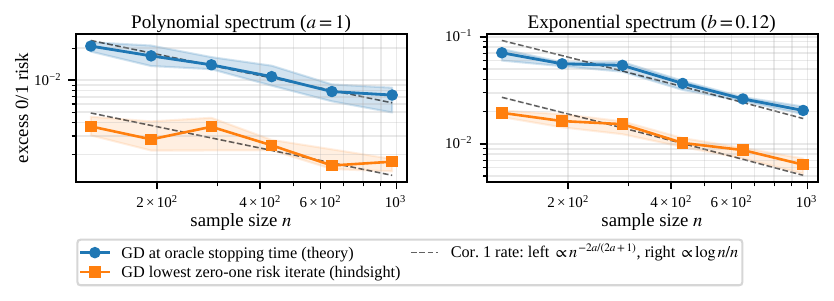}

    \caption{\textbf{Empirical rates for early-stopped GD.}
Excess zero-one population risk as a function of the sample size $n$, under the same experimental setting as  Figure~\ref{fig: u-shape}. The blue curve shows GD at the oracle stopping time used in our theory, while the orange curve shows the GD iterate with the lowest (in hindsight) excess zero-one risk, along the same GD trajectory. Solid curves and shaded bands show medians and interquartile ranges over independent repetitions. 
The thin black dashed lines show the rates from Corollary \ref{cor: polyexp}, $n^{-2a/(2a+1)}$ and $\log n/n$, up to multiplicative constants.}%
    \label{fig: rates}
\end{figure}

\subsection{Main Contributions}
\label{main contributions}

We study GD   under the logistic loss  for Gaussian mixture classification with label-flipping noise. We establish that in certain regimes early stopping is \emph{sufficient and necessary} for statistical optimality:

\begin{itemize}[leftmargin=*]
    \item (Sufficiency) We prove  that for  FCD covariance spectra,  there exists an oracle  stopping time at which   the GD iterate  achieves  minimax-optimal   excess  zero-one  risk 
     (Theorem~\ref{thm: matching bounds}). This is established via an upper bound for the early-stopped GD iterate (Theorem~\ref{thm: upper bound final}) and a  matching statistical lower bound for arbitrary  classifiers 
    (Theorem~\ref{thm: stat lower bound}). 
 To the best of our knowledge, this is the first minimax optimality result for early-stopped GD for Gaussian mixture classification.

    \item For common FCD spectral regimes, specialising Theorem~\ref{thm: matching bounds} yields Corollary \ref{cor: polyexp}, which provides explicit minimax rates for  polynomially and exponentially-decaying spectra. These rates 
    are summarised in Table~\ref{tab:summary}, and are closely matched by  the empirical rates %
    in Figure~\ref{fig: rates}.

    \item A key technical contribution is a calibration result (Lemma~\ref{cor:calibration}) linking excess logistic  and zero-one risks in the Gaussian mixture model with label-flipping noise, which renders the model misspecified under the logistic loss.
    Unlike standard calibration bounds \citep{zhang2004statistical,bartlett2006convexity}, which incur a square-root rate and optimise over all measurable predictors, our result achieves a linear rate and  is restricted to linear classifiers. This sharper calibration   is essential for deriving our~minimax-optimal~upper~bound.
    \item (Necessity) Finally, we prove an excess risk  lower bound   over all  linear interpolators (Theorem~\ref{thm:interpolation}), including  the max-margin classifier that GD converges to in direction. As summarised in Table~\ref{tab:summary},  in  scaling regimes such as $d\!\asymp\!n\log n$,  this interpolation lower bound vanishes strictly more slowly in $n$ than the minimax rate of early-stopped GD for polynomial and exponential spectra. This result builds on and generalises Theorem~4.2~of~\cite{wu2025benefits} by removing its signal-sparsity assumption. This  generalisation  is essential for 
    bringing 
    this interpolation lower bound within the scope of 
    our minimax theory, since the worst-case signal directions that~determine~the~minimax~rates~need~not~be~sparse.

\end{itemize}   

\begin{table}[t]
  \caption{Summary of main results:  minimax-optimal  
  excess zero-one  risk 
  achieved by early-stopped GD at  the 
  oracle  stopping time $\tau$; excess zero-one risk lower bound at interpolation. 
  }
  \label{tab:summary}
  \centering
  \begin{tabular}{lcc}
  \toprule
      Excess zero-one  risk       
      &\makecell{Polynomial decay \\$\lambda_i (\bSigma)\asymp i^{-2a},\, a>\frac{1}{2}$} & \makecell{Exponential decay\\ $\lambda_i(\bSigma)\asymp e^{-bi}$,\, $b>0$} \\
    \midrule
     At oracle stopping time $\tau$   %
     &$\asymp n^{-2a/(2a+1)}$   & $ \asymp\log n/n$   \\
     \midrule
     {At interpolation (for $d\asymp n\log n$)}  %
     & $\gtrsim 1/\log n$  
     & $\gtrsim 1/\log n$ 
     \\
    \bottomrule 
  \end{tabular}
\end{table}

\subsection{Related Work}
A first step towards understanding the statistical benefit of early stopping is taken by  \cite{wu2025benefits}, which analyses early-stopped GD in \emph{well-specified} logistic regression, a \emph{discriminative} model. %
It shows that 
early-stopped GD   incurs only polynomial sample complexity in certain overparameterised regimes where interpolation  requires exponential sample complexity. Our work goes beyond \citep{wu2025benefits} in two directions. First, we study a qualitatively distinct   Gaussian mixture model, a generative  model that becomes misspecified under the logistic loss in the presence of label-flipping noise $p>0$. This misspecification requires  new calibration bounds for the analysis. Second, we show that early-stopped GD is not only   
more sample-efficient than interpolation in common overparametrised regimes, but it is, in fact, minimax-optimal for the class of FCD covariance spectra.

Analyses of GD often characterise problem difficulty through the capacity-source conditions,   e.g., \cite{caponnetto2007optimal}, where   capacity    measures  spectral decay and   source  measures  signal alignment %
via a parameter $r$. 
A closely related work to ours on the regression side  is \cite{wu2025risk}, which establishes the minimax optimality  of early-stopped GD over power-law spectral classes. These classes are defined via the capacity-source conditions,   thus corresponding to polynomial spectral  decay, and allowing a general alignment parameter $r\ge0$. Their minimax rates capture  explicit dependence on $r$.  
Our work is similar  in spirit but focuses on  classification under the  logistic loss. We show minimax optimality of early-stopped GD for a broader class of spectral decays than \cite{wu2025risk}, namely the FCD class, which includes both  polynomial and  exponential  spectral decays. However, our minimax  rates can be seen as corresponding to $r=\frac{1}{2}$ and therefore do not explicitly capture dependence on $r$.

Interpolation in Gaussian mixtures   is studied as a canonical linear classification setting  for benign overfitting in recent works  \cite{chatterji2021finite, Muthukumar2021classification,  wang2022binary, cao2021risk, hashimoto2025universality, tsigler2025benign}; see \cite[Table 1]{hashimoto2025universality} for a summary of results across different regimes depending on the  covariance spectrum, the signal-to-noise ratio,
and label-flip probability. Our work studies the same model with general covariance spectrum, but focuses on  the statistical performance of early-stopped GD, using interpolation as an important point  of comparison. 

In contrast to the finite-sample results reviewed above and developed in this paper, a long line of work \cite{sur2019modern,salehi2019salehi,deng2022model,kini2020analytic,mignacco2020role,kammoun2021precise,salehi2020performance,montanari2025generalization}
studies  overparameterised linear classification in the   proportional \emph{asymptotic} regime, where $d,n\!\to\! \infty$ with $d/n$ converging to a constant. This regime enables precise asymptotic characterisation of 
the generalisation error. However, these results often rely on  systems of equations that are hard to approach analytically and   focus on  isotropic covariance, with less attention given to      anisotropic settings, which are  important for generalisation~of~$\ell_2$-regularised~solutions~\cite{bartlett2020benign, Muthukumar2021classification}. 

\paragraph{Notation:} We use ordinary letters for scalars,   boldface letters for vectors or matrices,  uppercase letters for random variables, and lowercase for their realisations. For a symmetric positive definite matrix  $\bSigma\in\mathbb{R}^{d\times d}$, let  $\operatorname{tr}(\bSigma):=\sum_{i=1}^d \bSigma_{ii}$ denote its trace.  %
For a vector $\bm x\in\mathbb{R}^d$ we denote $\|\bx\|:=\sqrt{\bx^\top \bx}$ and  $\|\bx\|_{\bSigma}:=\sqrt{\bx^\top \bSigma \bx}$. %
For an arbitrary matrix $\mathbf X\in \bbR^{n\times d}$, let $\|\mathbf X\|:=\sup_{\bv\in\bbR^d, \|\bv\|=1}\|\mathbf X\bv\|$ denote its  spectral norm.  %
We use $\C$ with various subscripts to denote positive universal constants independent of $n$ and $d$. 
Throughout, $a\lesssim b$ (resp.\ $a\gtrsim b$) means $a\le \C_1 b$ (resp.\ $a\ge \C_2 b$) for universal constants $\C_1, \C_2>0$, and $a\asymp b$ means $a\lesssim b$ and $a\gtrsim b$ hold simultaneously.

\section{Problem Setup}
\label{problem setup}

We study binary classification in the Gaussian mixtures model defined in \eqref{eq:GMM}. 
For a score function $h:\bbR^d\to \mathbb{R}$ with associated classifier $\sign(h):\bbR^d\to\{\pm 1\}$, 
define its population zero-one risk  as
\begin{align*}
    \mc{E}(h):=\P\left(\sign(h(\bX))\ne Y\right),
\end{align*} 
and   for a linear score parametrised by   $\bm w\in\mathbb{R}^d$,  $h_{\bw}(\bx)=\bx^\top\bw$, 
we abbreviate 
\begin{align*}
    \mc{E}(\bw):=\mc{E}(h_{\bw})\,.
\end{align*}  
It is known (see, e.g., \citep{hastie2009elements}) that the Bayes-optimal score function is linear, unique up to positive scaling, and given by 
\begin{align*}
    h_{\bw^*}\in\argmin_{h:\bbR^d\to\mathbb{R}} \mc{E}(h),%
    \quad \text{where}\quad
\bw^*:= \bSigma^{-1}\bmu.
\end{align*}
Hence, the goal of a training algorithm is to learn a linear classifier $\bw$ that minimises the  excess zero-one risk $\ee(\bm w)-\ee(\bw^*)$.\footnote{If, instead, one is interested in measuring the excess zero-one risk with respect to the latent labels $\wt Y$, note that the identity $\ee(\bm w)=p+(1-2p)\mathbb{P}(\operatorname{sign}(\bm X^\top \bm w)\neq \wt Y)$ gives an exact characterisation. Hence, the Bayes-optimal scores coincide under both risk measures and the excess risks differ only by a factor of $1-2p$. By the same identity, when $p=0.5$, i.e., $Y$ is random noise, $\ee(\bm w)=0.5$ for any $\bw\in\mathbb{R}^d$ and the excess risk is always 0. To avoid such degenerate cases, we fix $p<0.5$.} 

Since the zero-one loss is discontinuous, training commonly  minimises
 the  logistic loss $\ell(z)=\log(1+\exp(-z))$,  a smooth  convex surrogate. 
 As before, we define the logistic risk of a general score $h:\bbR^d\to \mathbb{R}$ or  a linear score $h_{\bw}(\bm x)=\bm x^\top \bw$, respectively, as:
 \begin{align*}
        \mc{L}(h):=\mathbb{E}[\ell(Y  h(\bX) )]\,\quad \text{and}\quad \mc{L}(\bm w):=\mc{L}(h_{\bw})\,.
    \end{align*} 
    
   The following lemma characterises  the  minimiser  
\begin{align*}
        \qquad \bw_{\mc{L}}^*\in\underset{\bm w\in\mathbb{R}^d}{\arg\min}\,\mc{L}(\bm w)\,,
    \end{align*}
 which plays   a central role in  our early-stopped GD analysis:

\begin{restatable}[Population logistic risk minimiser]{lemma}{lemmawpop}
    \label{lemma: w_pop}
   The minimiser of  $ \mc{L}(\bw)$ over $\bbR^d$ is unique 
    and   has the form 
   $\bm w_{\mc{L}}^*=\rho^*\bm w^*$   where $\rho^*\in(0,2]$. %
\end{restatable}

\paragraph{Training via GD:} Let $(\bm X_i,Y_i)_{i=1}^n$ be   $n$ training samples drawn i.i.d.\ from the joint distribution of $(\bX, Y)$ in \eqref{eq:GMM}.  
The most common  training procedure is to minimise   the empirical logistic risk  
    \begin{align*}
\widehat{\mc{L}}(\bm w):=\frac{1}{n}\sum_{i=1}^n\ell(Y_i\bm X_i^\top \bm w) ,
    \end{align*} 
    over $\bw\in\mathbb{R}^d$   via GD, 
   initialised at the origin with a fixed step size $\gamma>0$:
\begin{align*}
    \bm w_0 = \bm 0,\quad \quad  \bm w_{t+1}=\bm w_t-\gamma \nabla \lle (\bm w_{t}) \quad \text{for  $t\geq 0$}\,.
\end{align*}
Let  $\mathbf X_{1:n}\in\mathbb{R}^{n\times d}$ denote  the design matrix whose rows are $\bm X_1^\top, \dots, \bX_n^\top$, and let    $\|\mathbf{X}_{1:n}\|$ denote its spectral norm. To ensure monotone descent, we set  the step size 
to 
$    \gamma\leq \min (1/\beta, 1)$,
where  $\beta:= \norm{\mathbf{X}_{1:n}}^2/n$  is the smoothness constant of $\wh{\mc{L}}$ (see  Lemma \ref{lemma:lle gradient hessian} in Appendix \ref{sec:optimisation_lemmas}). 

\paragraph{Early stopping:} We focus on the overparameterised regime with $ d\ge n$, where two contrasting phenomena arise that explain the empirical observations in Figure \ref{fig: u-shape}.  On one hand,  the training samples $(\bX_i, Y_i)_{i=1}^n$ are almost surely \emph{linearly separable},  that is, there exists some $\bm w\in\mathbb{R}^d$ such that 
    \begin{align*}
\min_{ 1\leq i\leq n} Y_i\bm X_i^\top \bm w>0\,.
    \end{align*}
On the other hand,  the Bayes risk $\mc{E}(\bw^*)>0$, implying the  population distribution is \emph{not} linearly separable, even when $p=0$. In this regime, the implicit bias of GD under the  empirical logistic risk \cite{soudry2018implicit,ji2019implicit} drives the GD iterates to  diverge in norm while converging in direction to the max-margin classifier, an  interpolating solution  that  can incur suboptimal excess risk, as we prove in Section \ref{subsec:interpolation}. To prevent this, we  propose to  stop GD early, at iteration 
    \begin{align}\label{eq:def_tau}
        \tau:=\inf\{t\geq 1:\norm{\bm w_t}\geq 4\norm{\wb}\}.
    \end{align}

\begin{restatable}[Existence of $\tau$]{lemma}{tauexists}
\label{lem:existence_tau}
     Suppose that $n\leq d$.
    Then $\norm{\bm w_t}\rightarrow\infty$ as $t\rightarrow\infty$ almost surely, and so $\tau=\inf\{t\geq 1:\norm{\bm w_t}\geq 4\norm{\wb}\}$ is finite. 
\end{restatable}

Our stopping rule in \eqref{eq:def_tau} ensures that the early-stopped iterate $\bw_\tau$ remains bounded in the sense of Lemma \ref{ESproperties} (i) below. On a technical level, this allows us to localise the iterate $\bw_\tau$, and thus apply a  refined concentration analysis %
 yielding the sharp upper bound   in Theorem~\ref{thm: upper bound final} below. 
At the same time, stopping too early would prevent  the empirical logistic risk $\wh{\mc{L}}(\bw_\tau)$  from decreasing sufficiently, which is needed to ensure   progress towards a good classifier.
As  is customary in risk decomposition, we show that our  stopping time also  guarantees $\wh{\mc{L}}(\bw_\tau)\le\wh{\mc{L}}(\bw^*_{\mc{L}}) $ (Lemma \ref{ESproperties} (ii)).  

\begin{restatable}[Properties of the early-stopped iterate]{lemma}{wtauproperties}
\label{ESproperties}
    The early-stopped iterate $\bm w_\tau$ satisfies the following: (i) $\norm{\bm w_\tau}<4\norm{\wb}+1$; and (ii) $\lle(\bm w_{\tau})\leq\lle(\bm w^*_{\mc{L}})$.
\end{restatable}

Similar oracle stopping rules, typically derived from bias-variance tradeoffs, are commonly used in the statistical analysis of early stopping \citep{Buhlmann2003boosting,yao2007early,lin2017optimal,pillaud2018statistical,wu2025benefits}.

We detail the proof of Lemma~\ref{lemma: w_pop} in Appendix~\ref{sec:proof_logistic_minimiser}
 and the proofs of Lemmas~\ref{lem:existence_tau}--\ref{ESproperties} in Appendix~\ref{sec:optimisation_lemmas}.
\begin{remark}[Scaling regime]
    The condition $d\geq n$ is not inherently needed for our early-stopped GD minimax-optimality results below to hold. The upper bound in Theorem~\ref{thm: upper bound final} below applies to any $\bw\in\mathbb{R}^d$ satisfying the two properties of Lemma \ref{ESproperties}, and in particular, to the early-stopped iterate $\bw_\tau$. The only role of $d\geq n$ is to guarantee that the iterates diverge in norm, hence $\tau$ is finite (Lemma \ref{lem:existence_tau}). However, divergence can still occur, for instance, when $d<n$ and the data has a non-empty separable part \cite{ji2019implicit}. The statistical lower bound in Theorem~\ref{thm: stat lower bound} below does not require $d\geq n$ either.
\end{remark}

\section{Minimax Optimality of Early-Stopped GD}
\label{minimax optimality of early-stopped gd}

In this section, we present our main results,   establishing the statistical optimality of  GD  at the oracle stopping time $\tau$. We first prove  a high-probability upper bound on the excess zero-one risk of the early-stopped GD iterate $\bw_\tau$    (Section \ref{subsec: upper bound}), and then  a lower bound  that applies to  any arbitrary  classifier, including the linear classifier induced by $\bm w_\tau$ (Section \ref{subsec: stat lower bound}). Finally, we show that these bounds match as functions of the sample size $n$, for  the FCD spectral class, 
including polynomially and exponentially-decaying spectra (Section \ref{sec:fast and cont}).
Throughout, our results  depend on the spectrum of $\bSigma$, whose eigenvalues we denote by $\lambda_1\geq\lambda_2\geq\ldots\geq\lambda_d>0$.

\subsection{Upper Bound for Early-Stopped GD}
\label{subsec: upper bound}

We first present our excess zero-one risk upper bound for the early-stopped GD iterate $\bm w_{\tau}$.
This  bound involves  the spectrum of  $\bm \Sigma$ through a soft-thresholding operator, namely the \emph{effective rank} of $\bSigma$, defined for every $\eta>0$ as
\begin{align*}
    r_{\operatorname{eff}}(\eta):=\operatorname{tr}(\bm \Sigma(\bm \Sigma+\eta \bm I_d)^{-1})=\sum_{i=1}^d\frac{\lambda_i}{\lambda_i+\eta}\,.
\end{align*}
This  corresponds to the effective dimension commonly used in analyses of ridge regression,~e.g.,~\cite{caponnetto2007optimal}. We also recall from Lemma \ref{lemma: w_pop} that  the population logistic risk minimiser has the form 
   $\bm w_{\mc{L}}^*=\rho^*\bm w^*$   where $\rho^*\in(0,2]$. These quantities will be used in the theorem below.  

\begin{restatable}[Upper bound for early-stopped GD]{thm}{upperbound}
\label{thm: upper bound final}
    There exist a universal constant $\c_0>0$ and some $\c_1,\c_2,\c_3,\c_4>0$ depending on $\bmu$ and $\bm\Sigma$ only, but not on $n$ or $p$, such that the following holds:
    
    Fix $\delta\!\in\!(0,1)$ and $n\geq 3$. Suppose there exists $\eta_n>0$ satisfying $n\geq\c_0(r_{\operatorname{eff}}(\eta_n)+\!1\!+\log(2/\delta))$ and
    \begin{align}
        \label{thm condition}
        \Psi(\eta_n):=\frac{r_{\operatorname{eff}}(\eta_n)+\log(1/\delta)+\log\log(n)+\c_2}{n}+\c_3\eta_n\leq \c_4(\rho^*)^2\norm{\wb}_{\bSigma}^2\ .
    \end{align}
    Then, with probability at least $1-\delta$,
    \begin{align*}
        \ee (\bm w_{\tau})-\ee(\bw^*) \leq \c_1\frac{1-2p}{(\rho^*)^2\norm{\wb}_{\bSigma}}\Psi(\eta_n).
    \end{align*}
\end{restatable}

    The  proof of Theorem \ref{thm: upper bound final}, including explicit constants, is provided in Appendix \ref{app: upper bound proof},
    with a preview of the
    main technical challenges and   
      new proof techniques  %
      provided in Section \ref{sec: new proof techniques}.
    We now interpret the key quantities  in the   bound to build intuition:
    
        The upper bound depends on $\Psi(\eta_n)$, which must be small enough for condition \eqref{thm condition} to hold. Since the right hand side of \eqref{thm condition} is constant in $n$, this is ensured if $\Psi(\eta_n)\to 0$ as $n\rightarrow\infty$. Moreover, the bound $\ee(\bw_\tau)-\ee(\wb)\lesssim \Psi(\eta_n)$ shows that the same behaviour  suffices for vanishing excess risk. 
        Hence, condition \eqref{thm condition}   imposes no  restrictions beyond those already sufficient for vanishing excess risk.
        
        The  tightness of the upper bound is governed by the choice of $\eta_n$ through $\Psi(\eta_n)$, which consists of  two terms: one involving $r_{\operatorname{eff}}(\eta_n)$ and one linear in $\eta_n$. Since $r_{\operatorname{eff}}(\eta_n)$   decreases with $\eta_n$, this induces a tradeoff that determines the optimal choice of $\eta_n$. For commonly studied spectral decays, such as polynomial ($\lambda_i\asymp i^{-2a}$ for $a>\frac{1}{2}$) and exponential ($\lambda_i\asymp e^{-bi}$ for $b>0$), this tradeoff can be solved  analytically, yielding the 
        optimal rates $n^{-2a/(2a+1)}$ and $\log n/n$, respectively. In Section \ref{sec:fast and cont}, we prove that this optimality extends more generally to fast and continuously decaying spectrum \eqref{eq:fcd}. 

        A natural notion of signal-to-noise ratio (SNR) in the Gaussian mixtures model is $\normso{\bmu}$, which is equal to $\norm{\bw^*}_{\bSigma}$. Our upper bound scales  inversely with the SNR, as  expected: as the noise level increases, the SNR decreases, leading to a looser upper bound, consistent with the deterioration in the performance of the early-stopped GD iterate.

\subsubsection{New Proof Techniques}
\label{sec: new proof techniques}
\label{subsubsec: calibration}

The proof of Theorem \ref{thm: upper bound final} has several technical steps;  we highlight the~ones~capturing~the~main~novelty.

\paragraph{A new calibration result:} Standard calibration arguments \citep{zhang2004statistical,bartlett2006convexity} compare the excess zero-one risk to the excess logistic risk relative to the \emph{unrestricted} logistic optimum:
\begin{align}\label{eq:standard_calibration}
        \ee(\bm w)-\ee(\bw^*)\lesssim \sqrt{\mc L(\bm w)-\min_{h: \bbR^d\to \bbR}\mc L(h)}\,.
    \end{align}
    This bound has two    limitations. First, the square-root dependence leads to   suboptimal rates in the excess zero-one risk guarantees. 
    Second, and more importantly, the minimum of $\mc L(h)$ is taken over \emph{all} measurable functions $h:\mathbb{R}^d\rightarrow\bbR$, rather than  linear score functions. Consequently, the $\mc{L}(\bw)- \min_{h:\bbR^d\to \bbR}\mc{L}(h)$  bound in \eqref{eq:standard_calibration} is strictly looser  than one based on  $\mc{L}(\bw)-  \mc{L}(\wpop)$, with an irreducible difference $\mc{L}(\wpop)-\min_{h:\bbR^d\to \bbR} \mc{L}(h)$. 
    
    In   well-specified   models such as the sigmoid, Proposition 2.1.B~of~\cite{wu2025benefits} overcomes this second limitation, albeit with the suboptimal square-root dependence. However, this result is not directly transferable to our Gaussian mixtures model when $p\!>\!0$, since it is misspecified under the~logistic~loss.
    
    Together, the square-root rate,  the misspecification of our model under  logistic loss,  and the fact that the unrestricted logistic risk minimiser under our model is nonlinear when $p>0$ imply that existing calibration results yield not only  suboptimal rates, but rates that do  not   vanish as $n\!\rightarrow\!\infty$. 

We  overcome these limitations via a new  calibration result in  Lemma~\ref{cor:calibration}. Its proof, provided in Appendix~\ref{sec:stronger_calibration_proof}, leverages the local $\alpha$-strong convexity of $\mc{L}$ around $\wpop$ (see Lemma~\ref{lemma: local strong convexity}).
 
\begin{restatable}[Linear-rate calibration for linear classifiers]{lemma}{calilemma}
\label{cor:calibration}
    Let  $\bw\in \bbR^d$ satisfy $\|\bw\|\le 4\|\bw^*\|+1$ and 
    \begin{align}\label{eq:calibration_assumption}
        \mc L(\bm w)- \mc L(\wpop)\leq \frac{\alpha (\rho^*)^2\|\wb\|_{\bSigma}^2}{8}\,,
    \end{align}
    where $\alpha=\frac{1}{2}\ell''\big(3(4\|\bw^*\|+1)\sqrt{\|\bSigma+\bmu\bmu^\top\|}\big)$. 
    Then,
    \begin{align*}
        \ee(\bm w)-\ee(\bw^*)\leq \frac{8(1-2p)}{\sqrt{2\pi}\alpha (\rho^*)^2 \|\wb\|_{\bSigma}}\left(\mc L(\bm w)- \mc L(\wpop)\right).
    \end{align*}
\end{restatable}

The right hand side of \eqref{eq:calibration_assumption} is constant in $n$, so this condition is satisfied whenever the excess logistic risk  vanishes as $n\to \infty$. Thus, bounding   the excess zero-one risk reduces to bounding the excess logistic risk. We now focus on this step, highlighting the technical ingredients that depart~from~prior~work. 

\paragraph{Bounding  excess logistic risk using effective rank and  localisation:} Since $\norm{\bm w_\tau}\leq 4\norm{\wb}+1$ (Lemma \ref{ESproperties} (i)), it suffices to bound  the excess logistic risk $\mc L(\bw_\tau)-\mc L(\wpop)$ of the early-stopped GD iterate $\bw_\tau$ over the ball $\mathcal{B}=\{\bw:\norm{\bw}\leq 4\norm{\wb}+1\}$.

For this, \citep[Theorem 3.2]{wu2025benefits} uses a standard risk decomposition combined with Rademacher complexity, yielding a bound of the form $k/n+\sqrt{\sum_{i>k}\lambda_i/n}$. We improve this in two ways: (i) We replace  hard thresholding at the $k$-th eigenvalue with the soft effective rank $r_{\operatorname{eff}}(\eta_n)$ to obtain   finer spectral control (see Appendix \ref{appx:effective_rank}). In \cite{wu2025benefits}, the bias-variance tradeoff is governed by the discrete parameter $k$, whereas in our bound it is governed by the continuous parameter $\eta_n$,  allowing for  more effective optimisation  of the bound by tuning $\eta_n$. (ii) We perform a refined risk deviation analysis %
by localising to low-excess-risk regions,  restricting $\mathcal{B}$ to the shell $\{\bm w: \bm w\in \mathcal{B}, \mc L(\bw)-\mc L(\wpop)\leq r\}$ for a fixed $r>0$, and extending the resulting bounds to $\mc B$ via a peeling argument over $r$ (see Appendix~\ref{appx:localisation_concentration}). Together, these techniques allow us to remove the square-root dependence in the bound.

\subsection{Statistical Lower Bound}
\label{subsec: stat lower bound}

\label{sec: statistical lower bound}
To establish the statistical optimality of early-stopped GD, we provide a   statistical lower bound for  
\emph{arbitrary}
scores, a strictly larger class than  the linear scores obtained by logistic regression. 

We work in the Gaussian mixture model  \eqref{eq:GMM}, where the distribution of $(\bm X, Y)$ is parametrised by the centre $\bm \mu$, the noise covariance $\bm \Sigma$, and the label-flip probability $p$. For the purpose of the statistical lower bound, we  fix $\bm \Sigma$ and $p$,  and denote the resulting distribution of $(\bm X,Y)$ by $ P_{\bmu}$. Then  the $n$ i.i.d.\  samples $S:=(\bm X_i, Y_i)_{i=1}^n$ are distributed as $P^{\otimes n}_{\bmu}$.  For any procedure $\widehat{h}$ that takes $S$ as input, we   denote its output (a learned score) by $\wh h_S:\mathbb{R}^d\rightarrow\mathbb{R}$, which is  measurable with respect to  $S$. 

Our lower bound   involves 
an eigenvalue count  $N(\wt \eta)$, defined via the hard-thresholding operator $N(\cdot)$  as the number of eigenvalues of $\bSigma$ larger than a fixed $\wt \eta>0$:
\begin{align*}
    N(\wt \eta)=\big|\{i:i\in[d],\,\lambda_i\geq \wt\eta\}\big|.%
\end{align*}
This hard-thresholding operator $N(\cdot)$ can be viewed as the discrete analogue of the soft-thresholding operator $r_{\operatorname{eff}}(\cdot)$ in Theorem \ref{thm: upper bound final}; their outputs correspond to hard~and~soft~effective~dimension~notions.  %

\begin{restatable}[Statistical lower bound]{thm}{lowerbound}
\label{thm: stat lower bound}
    There exists a universal constant $\c>0$ such that the following holds: Fix $n$ and let $\wt \eta_n\in (0,\lambda_2/2]$ be small enough so that $N(2\wt \eta_n)>9$. Let $\wh h$ be any procedure. Then, there exists $\bmu\in\mathbb{R}^d$ with $\norm{\bSigma^{-1}\bmu}\leq \sqrt{2/\lambda_1}$ such that
    \begin{align*}
        \underset{S\sim P_{\bmu}^{\otimes n}}{\mathbb{E}}\left[\ee(\wh h_S)-\ee(\wb)\right]\geq \c (1-2p)\min\left\{\frac{2\wt \eta_n}{\lambda_1},\frac{N(2\wt\eta_n)-1}{n}\right\}.
    \end{align*}
\end{restatable}

 This  lower bound applies to \emph{any} procedure that outputs a score function $\wh h_S$ based on the samples $S$. In particular, this includes the linear score induced by the early-stopped iterate $\bw_\tau$ of Theorem \ref{thm: upper bound final}. As is standard in   learning theory, a lower bound in expectation can be converted into a  statement in probability via  Markov's inequality \citep{bach2024learning}, making it    directly comparable to Theorem~\ref{thm: upper bound final}.   
 
 To clarify the role of the tuning parameter $\wt \eta_n$, note that $\wt{\eta}_n$ captures a tradeoff between the two terms in the minimum, one increasing  and the other   decreasing in $\wt\eta_n$. Although the effective rank $r_{\operatorname{eff}}(\eta_n)$ in the upper bound (Theorem \ref{thm: upper bound final}) and the eigenvalue count  $N(\wt \eta_n)$ in the lower bound (Theorem \ref{thm: stat lower bound}) are not directly comparable, we show   in Section \ref{sec:fast and cont} that, under a common spectral regularity assumption, 
 they are of the same order, making the two bounds comparable. In fact, with an appropriate choice of   tuning parameters satisfying $\wt \eta_n= \eta_n/2$, the bounds match as functions of $n$ 
 (see Theorem~\ref{thm: matching bounds}).

We provide the proof of Theorem~\ref{thm: stat lower bound} in Appendix \ref{app: stat lower bound proof}.  The proof proceeds in two steps.  We first establish  a triangle-inequality-type bound  relating the excess zero-one risks of two   centres $\bmu_1$ and $\bmu_2$ to their distance in $\bSigma$-geometry (Lemma \ref{lemma: triangle e}). We then  construct a packing within  a suitable subspace of $\{\bmu:\norm{\bSigma^{-1}\bmu}\leq \sqrt{2/\lambda_1}\}$ (Lemma \ref{lemma: packing}), and use it to perform a 
reduction to hypothesis testing.

 \subsection{Statistical Optimality for Fast and Continuously Decaying Covariance Spectrum}

\label{sec:fast and cont}

 In this section, we combine Theorems \ref{thm: upper bound final} and \ref{thm: stat lower bound} to show that  
 early-stopped GD is statistically optimal  for the FCD spectral class. We begin by defining a FCD spectrum, following \citep[Assumption 3]{wu2025risk}.

 \begin{assumption}[Fast and Continuously Decaying  spectrum]
 \label{assumption: fast}
     For all sufficiently small $\eta>0$, the spectrum of $\bSigma$ satisfies
     \begin{align}
     \label{eq:fcd}
         \frac{1}{\eta}\sum_{i:\lambda_i<\eta}\lambda_i\lesssim N(\eta).\tag{FCD}
     \end{align}
 \end{assumption}

 Canonical examples of such spectral decays include polynomial ($\lambda_i\asymp i^{-2a}$ with $a>1/2$) and exponential ($\lambda_i\asymp e^{-bi}$ with $b>0$) decay.  In linear regression, these spectra are known to prevent benign overfitting \citep{bartlett2020benign}. While this observation does not directly transfer to classification, it motivates  early-stopped GD as an alternative to interpolation in these regimes. We next show that early-stopped GD is, in fact, statistically optimal for FCD spectra, and achieves vanishing excess zero-one risk:

\begin{restatable}[Statistical optimality of early-stopped GD for FCD spectrum]{thm}{matchingbounds}
\label{thm: matching bounds}
     Define 
     \begin{align*}
         \eta_n^*:=\sup\{\eta>0:N(\eta)\geq n\eta\},
     \end{align*}
     and assume that \eqref{eq:fcd} holds on $(0,2\eta_n^*]$. Then, $r_{\operatorname{eff}}(\eta)\asymp N(\eta)$ for all $\eta\in(0, \eta_n^*]$.
     
    Moreover, if $\eta_n^*\rightarrow0$ as $n\rightarrow\infty$,  and $\log\log(n)/n\lesssim \eta_n^*$,
     then %
     the upper bound in Theorem \ref{thm: upper bound final} with $\eta_n=\eta_n^*$ matches the statistical lower bound in Theorem \ref{thm: stat lower bound} with $\wt\eta_n=\eta_n^*/2$, with common rate $\asymp \eta_n^*$.
 \end{restatable}

 The implications of Theorem \ref{thm: matching bounds} are twofold. First, under Assumption \ref{assumption: fast}, it identifies an interval $(0,\eta_n^*]$ on which the    effective rank     $r_{\operatorname{eff}}(\eta)$ and the   eigenvalue count   $N(\eta)$, 
 appearing in  the upper and lower bounds of Theorems \ref{thm: upper bound final} and \ref{thm: stat lower bound}, respectively,  are of the same order for all  $\eta$. In this regime, these two notions of effective dimensions agree up to constants, reconciling the difference 
 between the bounds. Second, it establishes that the optimal tuning parameters satisfy $\eta_n=\eta^*_n$ and $ \wt \eta_n=\eta_n^*/2$. 
 With this choice, the upper and lower  bounds match, giving the common rate $\asymp \eta_n^*$. This  rate is analytically tractable for common spectral decays, as we show next in Corollary \ref{cor: polyexp}. The proofs of Theorem \ref{thm: matching bounds} and Corollary \ref{cor: polyexp} are provided  in Appendix \ref{appx: fcd spectra}.  

\begin{restatable}[Minimax rates for polynomially and exponentially-decaying spectra]{cor}{corpolyexp}
\label{cor: polyexp}
    For polynomial decay $\lambda_i\asymp i^{-2a}$ with $a>1/2$, there exists $\bmu\in\mathbb{R}^d$ with $\norm{\bSigma^{-1}\bmu}\leq \sqrt{2/\lambda_1}$ such that
     \begin{align*}
         \ee(\bw_\tau)-\ee(\wb)\asymp\eta_n^*\asymp n^{-\frac{2a}{2a+1}}\ .
     \end{align*}
     For exponential decay $\lambda_i\asymp e^{-b i}$ with $b>0$, there exists $\bmu\in\mathbb{R}^d$ with $\norm{\bSigma^{-1}\bmu}\leq \sqrt{2/\lambda_1}$ such that
     \begin{align*}
         \ee(\bw_\tau)-\ee(\wb)\asymp\eta_n^*\asymp \frac{\log n}{n}\ .
     \end{align*}
\end{restatable}

\section{Interpolation Lower Bound}
\label{subsec:interpolation}

To   show the benefit of early-stopped GD over interpolation,  
we prove a lower bound for \emph{all} interpolating linear classifiers, extending \cite[Theorem 4.2]{wu2025benefits} by removing  any \emph{sparsity} assumption on the Bayes  direction $\bw^*$, so the result holds uniformly over all   $\bw^*$. This generalisation is essential for enabling a meaningful comparison 
with 
our minimax rates in Corollary~\ref{cor: polyexp}: the worst case $\bmu$ does not guarantee a sparse $\wb$, so the   interpolation lower bound  of \cite{wu2025benefits} that assumes sparsity  does not apply. 

We now state this result formally. Its full proof appears in Appendix~\ref{app: interpolation lower bound proof}, with a sketch in Appendix~\ref{interpolation proof sketch}.

\begin{restatable}[Interpolation lower bound]{thm}{interpolationthm}
\label{thm:interpolation}
 Suppose   $\|\bw^*\|_{\bSigma}\asymp 1$. 
 Fix any  $\C\in(0,1)$ and   $\delta\in(0,1)$, 
     and let $\s\in \{1,\ldots,n\}$ be arbitrary. There exist  constants $\C_1, \C_2, \C_3>0$ such that  if 
    \begin{align}
        n\ge \C_1\cdot \left(\s \log \tfrac{ \sqrt{\s}}{\C}  + \s\log\log \tfrac{2\s}{\delta} + \log\tfrac{4}{\delta}\right)\quad \text{and}\quad  d-\s\ge \C_2n \,,\label{eq:n_large_interpolation_thm}
    \end{align} 
     then, defining   
  the set of interpolators $\mc{I}\!:=\!\{\bw\in \bbR^d\!:\!
  \min_{i}Y_i\bX_i^\top \bw\!>\!0\}$,   
 with probability at~least~$1\!-\!\delta$,
    \begin{align}\label{eq:interpolation_final}
       \forall \bw\in \mc{I}, \; \mc{E}(\bw)-\mc{E}(\bw^*) \ge  \frac{\C_3\cdot(1-2p)n}{\left(\sqrt{n}+\sqrt{d-\s} + \sqrt{2\log(2/\delta)}\right)^2}\,.
    \end{align}
\end{restatable}

The tuning parameter $k$ trades off sample size requirement in \eqref{eq:n_large_interpolation_thm}  and the tightness of the lower bound in \eqref{eq:interpolation_final}: larger values of $k$ require larger $n$ but yield tighter lower bounds.   
In the high-dimensional regime  $n/d\to 0$,   choosing   $k$ such that $k/d\to 0$ (e.g., $k=1$) gives  $\mc{E}(\bw)-\mc{E}(\bw^*)\gtrsim n/d$, matching the excess mean-squared error lower bound for interpolators in regression   \cite{muthukumar2020harmless}. 

In addition, when $d\!\asymp\! n\log n$,  Theorem~\ref{thm:interpolation}  implies that  the max-margin classifier $\bm w_{\infty}$, i.e., the interpolator that GD converges to in direction, satisfies the following lower bound: %
$$ \mc{E}(\bw_{\infty})-\mc{E}(\bw^*)\gtrsim \frac{1}{\log n}.$$
By contrast, Corollary~\ref{cor: polyexp} shows that early-stopped GD achieves the   minimax rates $n^{-2a/(2a+1)}$ 
and $\log n/n$ under polynomial and exponential spectral decays, both of which vanish strictly faster than $1/\log n$. Taken together,  
early stopping is sufficient for minimax optimality in these canonical  FCD spectral regimes, and  necessary when additionally, $d$ is mildly overparameterised, e.g.,~
$d\!\asymp\! n\log n$.

\section{Conclusion}
\label{conclusion}

We show that early-stopped GD on the logistic loss for Gaussian mixture classification is minimax-optimal   for covariance spectra with  fast and continuous decay,  including polynomial and exponential   decays (Theorem~\ref{thm: matching bounds}). The corresponding minimax rates are explicit (Corollary~\ref{cor: polyexp}) and closely matched by empirical performance of GD (Figure~\ref{fig: rates}). 
We further show that  early stopping can be substantially more sample-efficient than   interpolation in common overparameterised  regimes (Theorem~\ref{thm:interpolation}), establishing a statistical separation between the two, and thus the necessity of early stopping to achieve minimax optimality.

The stopping time studied in our work is oracle-based. This is standard in statistical analyses of early stopping \citep{Buhlmann2003boosting,yao2007early,lin2017optimal,pillaud2018statistical,wu2025benefits}, where the stopping rule commonly identifies the statistically-optimal scale given by a bias-variance decomposition. Thus, our stopping time should be viewed as an existence result, rather than an algorithm ready to be deployed in practice. Existing data-dependent rules are best understood in regression settings such as non-parametric least squares \citep{raskutti2014early}, where the losses used in testing and training coincide. Extending such guarantees to noisy classification with logistic loss presents additional challenges. For instance, our theory aims to minimise the zero-one risk, but the logistic loss is used instead as a surrogate.

\section*{Acknowledgments and Disclosure of Funding}

A.\ Buna and S.\,X.\ Liu worked closely  and discussed every aspect of this work in depth. A.\ Buna led the development of the upper bound, statistical lower bound and statistical optimality results, and carried out the numerical experiments. S.\,X.\ Liu  wrote the large majority of the main text, proved the interpolation lower bound, led   comparisons with prior work, and checked and refined all  proofs. P.\ Rebeschini helped formulate the problem, supervised the project and helped refine the paper.

\begingroup
\sloppy
A.\ Buna was supported by EPSRC
through the StatML CDT programme [project reference EP/S023151/1]. S.\,X.\ Liu and P.\ Rebeschini were supported by UK Research and Innovation (UKRI) under the UK government’s Horizon Europe funding guarantee [grant number EP/Y028333/1].  The authors declare no competing interests.\par
\endgroup

\bibliographystyle{plainnat}
\bibliography{references}

\newpage
\appendix
\addtocontents{toc}{\protect\setcounter{tocdepth}{3}}
\section*{Technical Appendices and Supplementary Material} 

\tableofcontents

\section*{Additional Notation}

We will use additional notation throughout the supplementary material. For a positive integer $k$, we will denote $[k]=\{1,2,\ldots,k\}$. Given a symmetric positive definite matrix  $\mathbf M\in\mathbb{R}^{d\times d}$, we use $\lambda_i(\mathbf M)$ to denote its $i$-th largest eigenvalue. For vectors $\bm x,\bm y\in \mathbb{R}^d$, we denote their $\mathbf{M}$-induced inner product as $\langle\bx, \by\rangle_{\bSigma} = {\bx^\top \bSigma \by}$. We  use $f_{X}$ to denote the PDF of the random variable $X$, and    $\phi$ and $\Phi$ to denote the PDF and CDF, respectively, of $\mc{N}(0,1)$. Throughout our proofs, we will denote the sigmoid function by
\begin{align*}
    \sigma(z)=\frac{1}{1+\exp(-z)}\,,
\end{align*}
and the population covariance of the covariates by
\begin{align*}
    \bm \Gamma = \operatorname{Cov}(\bm X)=\operatorname{Cov}(\wt{Y}\bmu+\bvareps)=\bm \Sigma+\bmu\bmu^\top.
\end{align*}

\section{Gaussian Mixtures Model Properties}
\label{app: gmm properties}

In this section, we collect  some useful properties of the Gaussian mixtures model introduced in \eqref{eq:GMM}. 

\begin{lemma}[Conditional probabilities and excess zero-one risk]
\label{lemma: posteriors and excess}
     For Gaussian mixtures, we have for $y\in\{\pm1\}$ and $\bx\in\bbR^d$,
     \begin{align*}
        \mathbb{P}(\wt Y=y\mid \bm X = \bm x)&=\sigma(2 y\bm x^\top \wb),\\
        \mathbb{P}(Y=y\mid \bm X = \bm x)&=p+(1-2p)\sigma(2y\bm x^\top \wb).
     \end{align*}
    Moreover, for any $\bm w\in\mathbb{R}^d$, the excess zero-one risk can be expressed as
    \begin{align}
        \ee(\bm w)-\ee(\wb)=(1-2p)\mathbb{E}\left[|\tanh(\bm X^\top \wb)|\mathbbm{1}\{\operatorname{sign}(\bm X^\top \bm w)\neq\operatorname{sign}(\bm X^\top \wb)\}\right].\label{eq:excess_zero_one}
    \end{align}
\end{lemma}

\begin{proof}
    Although these results are well-known, we prove them here for completeness. Start by noting that for $\wt{y}\in \{\pm1\}$,
    \begin{align*}
        \bm X\mid \widetilde Y=\widetilde{y}\sim \mc{N}(\widetilde y \bm \mu,\bm \Sigma).
    \end{align*}
     Together with the fact that $\mathbb{P}(\widetilde{Y}=\pm 1)=1/2$, the Bayes' rule gives:
    \begin{align*}
        \mathbb{P}(\widetilde{Y}=1\mid\bm X=\bm x)=\frac{f_{\bm X\mid \widetilde{Y}=1}(\bm x)}{f_{\bm X\mid \widetilde{Y}=1}(\bm x)+f_{\bm X\mid \widetilde{Y}=-1}(\bm x)}=\sigma(2\bm x^\top \bm\Sigma^{-1}\bm \mu)=\sigma(2\bm x^\top\wb).
    \end{align*}
     Since $\sigma(-t)=1-\sigma(t)$, $\P(\wt{Y}=y\mid \bX=\bx) = \sigma(2y\bx^\top \bw^*)$. 
  Furthermore, since $\wt Y$ is flipped with probability $p\in [0,0.5)$ to obtain $Y$, we have 
    \begin{align*}
        \mathbb{P}(Y=y\mid \bm  X= \bm x)=p+(1-2p)\sigma(2y\bm x^\top \wb)\,.
    \end{align*}
    Finally, conditioning on $\bX$ and computing the excess zero-one risk gives
    \begin{align*}
        \ee(\bm w)-\ee(\wb)=\mathbb{E}\left[|2\mathbb{P}(Y=1\mid \bm X)-1|\mathbbm{1}\{\operatorname{sign}(\bm X^\top \bm w)\neq\operatorname{sign}(\bm X^\top \wb)\}\right],
    \end{align*}
    and \eqref{eq:excess_zero_one} follows from the identity $\tanh(z)=2\sigma(2z)-1$.
    
\end{proof}

\begin{lemma}[Subgaussian covariates]
\label{lemma: subgauss covariates}
    For any $\bm w\in\mathbb{R}^d$, $\bm X^\top \bm w$ is subgaussian with variance proxy $\normg{\bm w}^2$\ . Consequently, for any $t>0$,
    \begin{align}
        \mathbb{P}(|\bm X^\top \bm w|>t)\leq 2\exp\left(-\frac{t^2}{2\normg{\bm w}^2}\right).\label{eq:subgaussian_concentration}
    \end{align}
\end{lemma}

\begin{proof}
    Fix any $\bm w\in\mathbb{R}^d$ and any $t\in\mathbb{R}$. Then, by the independence of $\wt Y$ and $\bm \varepsilon\sim \mathcal{N}(\bzero,\bm \Sigma)$, we have
    \begin{align*}
        \mathbb{E}[\exp (t\bm X^\top \bm w)]&=\mathbb{E}[\exp(t\wt Y\bmu^\top \bm w)]\mathbb{E}[\exp(t\bm\varepsilon^\top \bm w)]=\cosh(t\bmu^\top\bm w)\exp\left(\frac{t^2}{2}\bm w^\top\bm\Sigma\bm w\right)\\
        &\leq \exp\left(\frac{t^2}{2}(\bm w^\top \bm \Sigma \bm w+(\bmu^\top \bm w)^2)\right)=\exp\left(\frac{t^2}{2}\bm w\bm^\top \bGamma\bm w\right)=\exp\left(\frac{t^2}{2}\normg{\bm w}^2\right).
    \end{align*}
    Thus, $\bm X^\top \bm w$ is subgaussian with variance proxy $\normg{\bm w}^2$\ . The concentration inequality \eqref{eq:subgaussian_concentration} follows from a Chernoff bound.
\end{proof}

\section{Proofs Omitted from Section \ref{problem setup}}
\label{app: problem setup}
\subsection{Population Logistic Risk Minimiser}
\label{sec:proof_logistic_minimiser}

We restate Lemma~\ref{lemma: w_pop} and detail its proof. 
\lemmawpop*

\begin{proof}
    Recall that $\wb = \bm \Sigma^{-1}\bm \mu$. 
    Fix an arbitrary $\bm w\in\mathbb{R}^d$ and decompose it as 
    \begin{align*}
        \bm w = \rho \wb + \bm u, 
    \end{align*}
    where 
\begin{align*}
    \rho=\bm w^\top \bm \mu/(\bw^*)^\top\bm \mu\in \bbR\quad \text{and}\quad \bm u=\bw-\rho\bw^*
\end{align*}
By construction, $\bu^\top \bm \mu=0$, so $\bw$ is split into a component along the signal direction $\bw^*$ and a component $\bu$ orthogonal to $\bmu$. 

    \paragraph{Step 1:} We first prove that each minimiser of  $\mc L(\cdot)$  satisfies  $\bm u=\bm 0_d$, i.e., it lies along $\bw^*$. For this, let $T^*=\bm X^\top \wb$ and $T_{\bm u}=\bm X^\top \bm u$. These random variables are independent. Indeed,  since  $\bw^*=\bSigma^{-1}\bmu$ and $\bm u^\top \bm \mu =0$, we have 
    \begin{align}
        T^*&=(\widetilde{Y}\bm \mu+\bm\varepsilon)^\top \bm \wb=\widetilde{Y}\normso{\bm \mu}^2+\bm\varepsilon^\top \wb, \label{eq:T_star_def}\\
        T_{\bm u}&=(\wt Y\bm \mu+\bm\varepsilon)^\top \bm u=\bm \varepsilon^\top \bm u.
        \label{eq:Tu_def}
    \end{align}
    These are independent because $\wt Y$ is independent of $\bm\varepsilon$ and the Gaussians $\bm\varepsilon^\top\wb$ and $\bm\varepsilon^\top \bm u$ are uncorrelated:
    \begin{align*}
        \operatorname{Cov}(\bm\varepsilon^\top\wb,\bm\varepsilon^\top  \bm u)=\bm u^\top \bm \Sigma \wb=0.
    \end{align*}
    Now, for $t,s\in\mathbb{R}$, define
    \begin{align}\label{eq:eta_rho_def}
        \zeta(t)=p+(1-2p)\sigma(2t),\quad\quad\vartheta_t(s)=\zeta(t)\ell(s)+(1-\zeta(t))\ell(-s)\,,
    \end{align}
    where recall $\sigma(z)={1}/(1+\exp(-z))$ denotes the sigmoid function and $\ell(z)=\log(1+\exp(-z))$    denotes the logistic loss. 
    Then, using the decomposition of $\bm w$ and the tower law, we can write 
    \begin{align*}
        \mc L(\bm w)&=\mathbb{E}[\ell(Y\bm X^\top \bm w)]=\mathbb{E}[\ell(Y\bm X^\top (\rho \wb +\bm u))]=\mathbb{E}[\ell(Y(\rho T^*+T_{\bm u}))]\\
        &=\mathbb{E}\left[\mathbb{E}[\ell(Y(\rho T^*+T_{\bm u}))\mid \bm X]\right]=
        \mathbb{E}[\vartheta_{T^*}(\rho T^*+T_{\bm{u}})]\,,
    \end{align*}
    where the last step applies $\P(Y=1\mid \bX)=\zeta(T^*)$ due to Lemma \ref{lemma: posteriors and excess}. 
    Since $\ell''>0$, the function $\vartheta_t$ is strictly convex for any $t\in \bbR$. Hence, conditioning on $T^*$ and applying the tower law, we can write 
    \begin{align*}
        \mc L(\bm w)&=\mathbb{E}[\vartheta_{T^*}(\rho T^*+T_{\bm{u}})]=\mathbb{E}\left[\mathbb{E}[\vartheta_{T^*}(\rho T^*+T_{\bm{u}})\mid T^*]\right]\\
        &\stackrel{\text{(i)}}{\geq} \mathbb{E}\left[\vartheta_{T^*}(\rho T^*+\mathbb{E}[T_{\bm u}\mid T^*])\right] \stackrel{\text{(ii)}}{=}\mathbb{E}[\vartheta_{T^*}(\rho T^*)]=\mc L(\rho\wb)\,,
    \end{align*}
    where (i) is   Jensen's inequality applied to the strictly convex function $\vartheta_{T^*}$,  and (ii) uses   the independence of $T^*$ and $ T_{\bm u}$, together with $\mathbb{E}[T_{\bm u}]=0$. 
    Equality in (i)  holds if and only if $T_{\bm u}=0$, which implies $\bm u=\bm 0_d$. 
    
    \paragraph{Step 2:} We have shown every minimiser of $\mc{L}(\cdot)$ in $\bbR^d$ takes the form $\rho \bw^*$, so it remains to minimise over $\rho$. For this, we study the function
    \begin{align*}
        f(\rho)=\mc L(\rho \wb).
    \end{align*}
    To begin,     by noticing  $T^*=\bX^\top \bw^*= (\wt{Y}\bmu+\bvareps)^\top \bw^*= \wt Y \normso{\bm \mu}^2+\bm \varepsilon^\top \wb$ and      $\bvareps^\top\bw^*\perp  \wt Y$, we have 
    \begin{align}\label{eq:E_T_star_Y}
        \E[T^*\wt{Y}]=\bm \mu^\top\bm \Sigma^{-1}\bm \mu = \|\bmu\|_{\bSigma^{-1}}^2\,.
    \end{align}
    Thus, we can write
    \begin{align}
        \mathbb{E}[T^*\zeta(T^*)]&\stackrel{\text{(i)}}{=}\tfrac{1}{2}\mathbb{E}[T^*(2\zeta(T^*)-1)]\nonumber\\
        &\stackrel{\text{(ii)}}{=}\tfrac{1}{2}\mathbb{E}[T^*\mathbb{E}[Y\mid \bm X]]=\tfrac{1}{2}\mathbb{E}[T^* Y]=\tfrac{1-2p}{2}\mathbb{E}[T^*\wt Y]\nonumber\\
        &\stackrel{\text{(iii)}}{=}
        \tfrac{1-2p}{2}\normso{\bm \mu}^2\,,\label{eq:E_T_eta}
    \end{align}
    where (i) uses $\mathbb{E}[T^*]=0$, (ii) holds because  $\P(Y=1\mid \bX)=\zeta(T^*)$ from Lemma \ref{lemma: posteriors and excess} implies $\mathbb{E}[Y\mid \bm X]=2\mathbb{P}(Y=1\mid \bm X)-1=2\zeta(T^*)-1$, and  (iii) applies \eqref{eq:E_T_star_Y}. 
    
    Let $\vartheta_t(s), \zeta(t)$ be defined as in \eqref{eq:eta_rho_def}, then note that the derivative $\vartheta'_t(s)=\sigma(s)-\zeta(t)$. Using this and applying \eqref{eq:E_T_eta}, we obtain
    \begin{align}
    \label{f' alpha}
        f'(\rho)&=\frac{\mathrm{d}}{\mathrm{d} \rho}\E[\vartheta_{T^*}(\rho T^*)]=\E[T^* \vartheta_{T^*}'(\rho T^*)]
        =\mathbb{E}[T^*(\sigma(\rho T^*)-\zeta(T^*))]\nonumber\\
        &=\mathbb{E}[T^*\sigma(\rho T^*)]-\mathbb{E}[T^*\zeta(T^*)]\nonumber\\
        &=\mathbb{E}[T^*\sigma(\rho T^*)]-\tfrac{1-2p}{2}\normso{\bm\mu}^2\,.
    \end{align}
    Setting $f'(\rho^*)=0$ and rearranging, any stationary point $\rho^*$ of $f(\cdot)$ satisfies
    \begin{align*}
        \E[T^*\sigma(\rho^* T^*)]=\tfrac{1-2p}{2}\normso{\bm\mu}^2\,.
    \end{align*}
    
 \paragraph{Step 3:} The final step is to show that $f(\cdot)$ has a \emph{unique} stationary point $\rho^*\in (0,2]$ and that it is a minimiser.
 Since $\sigma'(z)>0$ for all $z\in \bbR$ and $T^*\ne 0$ almost surely, we have 
    \begin{align*}
        f''(\rho)=\mathbb{E}[(T^*)^2\sigma'(\rho T^*)]>0\,,
    \end{align*}
    meaning $f$ is strictly convex. 
    Hence, $\rho^*$ is the unique minimiser of $f(\cdot)$. Finally, note that $\rho^*\in (0,2]$ because $f'(0)<0$ and $f'(2)\geq 0$. Indeed, Eq.~\eqref{f' alpha} implies that 
    \begin{align*}
        f'(0)=\tfrac{1}{2}\mathbb{E}[T^*]-\tfrac{1-2p}{2}\normso{\bm \mu}^2=-\tfrac{1-2p}{2}\normso{\bm \mu}^2<0\,,
    \end{align*}
    and
    \begin{align*}
        f'(2)&=\mathbb{E}[T^*\sigma(2 T^*)]-\tfrac{1-2p}{2}\normso{\bm \mu}^2\\
        &\stackrel{\text{(i)}}{=}\tfrac{1}{2}\mathbb{E}\left[T^*({2\sigma(2T^*)-1})\right]-\tfrac{1-2p}{2}\normso{\bm \mu}^2\\
        &\stackrel{\text{(ii)}}{=}\tfrac{1}{2}\mathbb{E}[T^*\wt Y]-\tfrac{1-2p}{2}\normso{\bm \mu}^2\\
        &\stackrel{\text{(iii)}}{=}\tfrac{1}{2}\normso{\bm \mu}^2-\tfrac{1-2p}{2}\normso{\bm \mu}^2=p\normso{\bm \mu}^2\geq 0\,,
    \end{align*}
    where (i) uses $\E[T^*]=0$, (ii) holds because  
     $\P(\wt{Y}=1\mid \bX)=\sigma(2T^*)$  from Lemma \ref{lemma: posteriors and excess} implies  $\mathbb{E}[\wt Y\mid \bm X]=2\P(\wt{Y}=1\mid \bX)-1=2\sigma(2T^*)-1$, and  (iii) applies \eqref{eq:E_T_star_Y}. 
\end{proof}

\subsection{Optimisation Lemmas}\label{sec:optimisation_lemmas}
As detailed in Section~\ref{problem setup}, we train a linear classifier on the Gaussian mixture  samples by minimising the empirical logistic risk  $\wh{\mc{L}}(\cdot)$ via early-stopped GD. To analyse the optimisation, we first establish the Lipschitz and smoothness  properties of $\wh{\mc{L}}(\cdot)$,   then characterise the stopping time $\tau$ (Lemma~\ref{lem:existence_tau}) and the early-stopped GD iterate $\bw_\tau$ (Lemma~\ref{ESproperties}).  Throughout this section, we will denote
\begin{align*}
    \beta:=\frac{\norm{\mathbf{X}_{1:n}}^2}{n}.
\end{align*}

\begin{lemma}[Lipschitz and smoothness]
\label{lemma:lle gradient hessian}
The empirical logistic risk $\lle(\cdot)$ is $\sqrt{\beta}$-Lipschitz and $\beta$-smooth, i.e.,
\begin{align*}
    \underset{\bm w\in\mathbb{R}^d}{\sup} \norm{\nabla\lle(\bm w)}\leq \sqrt{\beta}\quad\quad\text{and}\quad\quad \underset{\bm w\in\mathbb{R}^d}{\sup}\norm{\nabla^2 \lle(\bm w)}\leq {\beta}.
\end{align*}

\end{lemma}

\begin{proof}
Recall $\|\mathbf{X}_{1:n}\|$ denotes the spectral norm of $\mathbf{X}_{1:n}$.    Let $\bm a(\bm w)\in\mathbb{R}^n$ be the vector whose $i$-th coordinate is defined as $[\bm a(\bw)]_i=Y_i/(1+\exp(Y_i\bm X_i^\top \bm w))$. Then we have $\nabla \lle(\bm w)=-\frac{1}{n}\mathbf{X}_{1:n}^\top \bm a(\bm w)\in \bbR^d$. Since the entry-wise absolute value $|[\bm a(\bm w)]_i| \leq 1$ for all $i\in[n]$, we have $\norm{\bm a(\bm w)}\leq\sqrt{n}$ and so
    \begin{align*}
        \norm{\nabla \lle(\bm w)}\leq \frac{\norm{\mathbf{X}_{1:n}}\norm{\bm a(\bm w)}}{n}\leq \frac{\norm{\mathbf{X}_{1:n}}}{\sqrt{n}}=\sqrt{\beta}.
    \end{align*}
    For the Hessian, note that $\nabla^2 \lle(\bm w)=\frac{1}{n}\mathbf{X}_{1:n}^\top \bm D(\bm w)\mathbf{X}_{1:n}\in \bbR^{d\times d}$, where $\bm D(\bw)$ is the $n\times n$ diagonal matrix whose $i$-th diagonal entry is $[\bm D(\bw)]_{ii}=\ell''(Y_i\bm X_i^\top \bm w)$. %
    Since $0\leq \ell''(z)\leq 1/4$ for all $z\in \bbR$, we have 
    \begin{align*}
        \norm{\nabla^2\lle(\bm w)}\leq \frac{\norm{\mathbf{X}_{1:n}}\norm{\bm D(\bm w)}\norm{\mathbf{X}_{1:n}}}{n}\leq \frac{\norm{\mathbf{X}_{1:n}}^2}{4n}\leq \beta.
    \end{align*}
\end{proof}

We next restate Lemmas~\ref{lem:existence_tau} and \ref{ESproperties} and provide their proofs, respectively. 
\tauexists*

\begin{proof}
   When $n\leq d$, we have $\operatorname{rank}(\mathbf{X}_{1:n})=n$ almost surely,  since the covariates  $\bm X=\wt Y\bm \mu+\mc{N}(0,\bm\Sigma)$  have a continuous density on  $\mathbb{R}^d$ and are therefore linearly independent almost surely. Then, linear separability of the training samples follows from \cite[Proposition 2.2]{wu2025benefits}. On linearly-separable data, the implicit bias of GD  under logistic loss \cite{soudry2018implicit, shamir2021gradient}  implies    $\|\bm w_t\|\to \infty$  as $t\to \infty$, so $\tau$  is finite.
\end{proof}

\wtauproperties*

\begin{proof}
    \textbf{Proof of part (i): } By the definition of the stopping time $\tau:=\inf\{t\geq 1:\norm{\bm w_t}\geq 4\norm{\wb}\}$, we have $\norm{\bm w_{\tau-1}}<4\norm{\wb}$. Using the GD update rule, Lemma \ref{lemma:lle gradient hessian}, and the choice of step size  $\gamma\leq \min(1/\beta, 1)$ in Section~\ref{problem setup}, 
    we can write
    \begin{align*}
        \norm{\bm w_{\tau}-\bm w_{\tau-1}}=\gamma \norm{\nabla \lle(\bm w_{\tau-1})}\leq \gamma\sqrt{\beta}\le \sqrt{\gamma}\cdot \sqrt{\gamma\beta}\leq \sqrt{\gamma}\leq 1\,.
    \end{align*}
    Hence, by the triangle inequality,
    \begin{align*}
        \norm{\bm w_\tau}\leq \norm{\bm w_{\tau-1}}+\norm{\bm w_{\tau}-\bm w_{\tau-1}}<4\norm{\wb}+1.
    \end{align*}

       \paragraph{Proof of part (ii):} Suppose, towards a contradiction, that $\lle(\bm w_{\tau})>\lle(\bm w^*_{\mc{L}})$. Since $\lle(\cdot)$ is $\beta$-smooth and the step size $\gamma\leq 1/\beta$, Lemma 3.3 in \cite{wu2025benefits} with the comparator $\bm w^*_{\mc{L}}$ gives
    \begin{align*}
        \frac{\norm{\bm w_{\tau}-\bm w_{\mc{L}}^*}^2}{2\gamma \tau}+\lle(\bm w_{\tau})\leq \lle(\bm w_{\mc{L}}^*)+\frac{\norm{\bm w_{\mc{L}}^*}^2}{2\gamma \tau}.
    \end{align*}
    Our assumption $\lle(\bm w_{\tau})>\lle(\bm w^*_{\mc{L}})$ then forces  $\norm{\bm w_\tau-\bm w_{\mc{L}}^*}< \norm{\bm w^*_{\mc{L}}}$. But  by the definition of the stopping time $\tau$, the triangle inequality, and $\norm{\bm w^*_{\mc{L}}}=\|\rho^*\bw^*\|\leq 2\norm{\wb}$ due to Lemma \ref{lemma: w_pop},  we arrive at a contradiction:
    \begin{align*}
        4\norm{\wb}\leq \norm{\bm w_\tau}\leq \norm{\bm w_\tau-\bm w^*_{\mc{L}}}+\norm{\bm w^*_{\mc{L}}}<2\norm{\bm w^*_{\mc{L}}}\leq 4\norm{\bm \wb}.
    \end{align*}
\end{proof}

\section{Proof of  the Early-Stopped GD Upper Bound}
\label{app: upper bound proof}
We prove  the excess zero-one  risk upper bound of Theorem \ref{thm: upper bound final} for the early-stopped iterate $\bm w_\tau$  via a sequence of lemmas (Sections~\ref{sec:stronger_calibration_proof}--\ref{sec:empirical_excess_risk_deviation}), culminating in the proof of Theorem~\ref{thm: upper bound final} at the end of this section (Section~\ref{sec:theorem_1_proof}).

Throughout, for symmetric matrices $\bm A, \bm B\in \bbR^{d\times d}$,  we   write $ \bm A \succ \bm B$ (resp.\ $\succeq$) to mean that $\bm A-\bm B$ is positive  definite (resp.\ positive semidefinite), i.e., $\bm h^\top (\bm A- \bm B)\bm h> 0$ (resp. $\ge0$) for all $\bm h\in \bbR^d\setminus\{\bm 0\}$.
\subsection{Zero-One to Logistic Calibration via \texorpdfstring{$\bm \Sigma$}{}-geometry}\label{sec:stronger_calibration_proof}

In this section, we prove the calibration result in Lemma~\ref{cor:calibration}, which upper bounds the excess zero-one risk   by the excess logistic risk. 
The key technique is to introduce the $\bSigma$-geometry distance $\|\bw-\wpop\|_{\bSigma}$ as a bridge between the two risks:  Lemma~\ref{lemma: conditional excess risk bound} upper bounds the   excess zero-one risk by $\|\bw-\wpop\|_{\bSigma}$, while Lemma~\ref{lemma: local strong convexity} allows us to lower bound   the excess logistic risk by $\|\bw-\wpop\|_{\bSigma}$. Combining the two yields Lemma~\ref{cor:calibration}.

To begin, we   show in Lemma~\ref{lemma: conditional excess risk bound} that if  $\bm w\in\mathbb{R}^d$ is sufficiently close to  $\wpop$ in the $\bSigma$-norm, then its excess zero-one risk $\mc{E}(\bw)-\ee(\wb)$ can be upper bounded by $\|\bw-\bw^*_{\mc{L}}\|_{\bSigma}$.

\begin{lemma}[$\bSigma$-geometry calibration of excess zero-one risk]
\label{lemma: conditional excess risk bound}
    Suppose $\bm w\in\mathbb{R}^d$ satisfies
    \begin{align}
        \norm{\bm w-\bm w^*_{\mc{L}}}_{\bm \Sigma}\leq \frac{\rho^*\normso{\bmu}}{2}.\label{eq:w_wpop_norm_bound}
    \end{align}
    Then,
    \begin{align}\label{eq:excess_risk_ell2_norm}
        \ee(\bm w)-\ee(\wb)\leq \frac{4(1-2p)}{\sqrt{2\pi}(\rho^*)^2\normso{\bmu}}\norms{\bm w-\wpop}^2.
    \end{align}
\end{lemma}

\begin{proof}
    Recall  $\bw^*=\bm\Sigma^{-1}\bmu$. As in Lemma \ref{lemma: w_pop}, decompose  $\bm w=\rho \wb+\bm u$ with $\bm u^\top \bmu=0$. 
    We first show  that $\rho>0$. Since $\wpop=\rho^*\bw^*$, we have 
    \begin{align}
    \label{eq: norm w - w pop}
        \norms{\bm w-\wpop}^2=\norms{(\rho-\rho^*)\wb+\bm u}^2=(\rho-\rho^*)^2\normso{\bmu}^2+\norms{\bm u}^2\,,
    \end{align}
    which implies  $\norms{\bm w-\wpop}\geq |\rho-\rho^*|\normso{\bmu}$. Combining with the assumption in  \eqref{eq:w_wpop_norm_bound}  gives  $|\rho-\rho^*|\leq \rho^*/2$. Since  $\rho^*>0$  (c.f.\ Lemma~\ref{lemma: w_pop}), it follows  that
    \begin{align}
    \label{bound alpha}
        \rho\geq \frac{\rho^*}{2}>0\,.
    \end{align}
    
The decomposition $\bw=\rho\bw^*+\bu$ implies  $\bm X^\top \bm w=\rho T^*+T_{\bm u}$, where  $T^*=\bm X^\top\wb$ and $T_{\bm u}=\bm X^\top \bm u$    are independent and each can be expanded as in \eqref{eq:T_star_def}--\eqref{eq:Tu_def}. Applying  the excess   risk expression  \eqref{eq:excess_zero_one} of Lemma \ref{lemma: posteriors and excess} and applying $\tanh(z)\leq |z|$ yields
    \begin{align*}
        \ee(\bm w)-\ee(\wb)&=(1-2p)\mathbb{E}\left[|\tanh(\bm X^\top \wb)|\mathbbm{1}\{\operatorname{sign}(\bm X^\top \bm w)\neq\operatorname{sign}(\bm X^\top \wb)\}\right]\\
        &\leq(1-2p)\mathbb{E}\left[|\bm X^\top \wb|\mathbbm{1}\{\operatorname{sign}(\bm X^\top \bm w)\neq\operatorname{sign}(\bm X^\top \wb)\}\right]\\
        &=(1-2p)\mathbb{E}\left[|T^*|\mathbbm{1}\{\operatorname{sign}(\rho T^*+T_{\bm u})\neq\operatorname{sign}(T^*)\}\right].
    \end{align*}
   Since $\rho>0$, we have the following event inclusion:
    \begin{align*}
        \{\operatorname{sign}(\rho T^*+T_{\bm u})\neq\operatorname{sign}(T^*)\}\subseteq\left\{|T^*|\leq \frac{|T_{\bm u}|}{\rho}\right\}\,.
    \end{align*}
    
    Note that conditionally on $\wt{Y}$,  $T^*\mid \wt{Y} \sim \mathcal{N}(\wt{Y}\|\bmu\|_{\bSigma^{-1}}^2,\normso{\bmu}^2)$, so the conditional density satisfies $f_{T^*\mid \wt{Y}}(t)\leq 1/(\sqrt{2\pi}\normso{\bmu})$. Since this bound is uniform over $\wt{Y}\in \{\pm1\}$, the marginal density of $T^*$ satisfies the same bound $f_{T^*}(t)\leq 1/(\sqrt{2\pi}\normso{\bmu})$. Therefore, we have
    \begin{align*}
        \ee(\bm w)-\ee(\bw^*)&\leq (1-2p)\mathbb{E}\left[|T^*|\mathbbm{1}\{\operatorname{sign}(\rho T^*+T_{\bm u})\neq\operatorname{sign}(T^*)\}\right]\\
        &\leq (1\!-\!2p)\mathbb{E}\left[|T^*|\mathbbm{1}\left\{|T^*|\!\leq\! \frac{|T_{\bm u}|}{\rho}\right\}\right]
        \!=\!(1\!-\!2p)\mathbb{E}\left[\mathbb{E}\left[|T^*|\mathbbm{1}\left\{|T^*|\!\leq\! \frac{|T_{\bm u}|}{\rho}\right\}\biggl\vert T_{\bm u}\right]\right]\\
        &=(1-2p)\mathbb{E}\left[\int_{-|T_{\bm u}|/\rho}^{|T_{\bm u}|/\rho} |t|f_{T^*}(t)\de t\right]
        \leq \frac{1-2p}{\sqrt{2\pi}\normso{\bmu}}\mathbb{E}\left[\int_{-|T_{\bm u}|/\rho}^{|T_{\bm u}|/\rho} |t|\de t\right]\\
        &=\frac{1-2p}{\sqrt{2\pi}\normso{\bmu}}\mathbb{E}\left[\frac{T_{\bm u}^2}{\rho^2}\right]
        =\frac{(1-2p)}{\sqrt{2\pi}\normso{\bmu}}\cdot \frac{\norms{\bm u}^2}{\rho^2}\,,
    \end{align*}
    where the last step holds because  $T_{\bm u}\sim \mc N(0,\norms{\bm u}^2)$. 
    Combining this with the bound $\norms{\bm u}^2\leq \norms{\bm w-\wpop}^2$ implied by (\ref{eq: norm w - w pop}), and with (\ref{bound alpha}), we arrive at the result in \eqref{eq:excess_risk_ell2_norm}.

\end{proof}

\bigskip

The next step is to link the excess logistic risk to $\|\bw- \wpop\|_{\bSigma}^2$. 
Since $\wpop$ minimises  $\mc L(\cdot)$, we  do so by establishing  local strong convexity of $\mc{L}(\cdot)$ around $\wpop$.  
By Lemma \ref{ESproperties}, we know  the early-stopped iterate $\bw_\tau$ satisfies $\|\bw_\tau\|\le  4\|\bw^*\|+1$. We therefore work over the ball
\begin{align}\label{eq:def_Ball_R}
\mc B=\{\bm w\in\mathbb{R}^d:\norm{\bm w}\leq R\}\,,\quad \text{where}\quad R:=4\norm{\wb}+1\,.
\end{align}
Clearly $\wb\in\mc B$ and, by Lemma \ref{lemma: w_pop}, $\wpop\in\mc B$ as well. As we  show in Lemma~\ref{lemma: local strong convexity} below, the local strong-convexity parameter of $\mc{L}(\cdot)$ is given by
\begin{align}
\label{eq: curvature}
    \alpha = \tfrac{1}{2}\ell''(3R\sqrt{\norm{\bm \Gamma}}),\quad\quad\text{where}\quad  \bm\Gamma=\operatorname{Cov}(\bm X)=\bmu\bmu^\top +\bm\Sigma.
\end{align}

\begin{lemma}[Local strong convexity of logistic risk]
\label{lemma: local strong convexity}
    For every $\bm w\in\mc B$ it holds that $\nabla^2 \mc L(\bm w)\succeq \alpha\bm \Gamma$. Consequently,
    \begin{align*}
        \mc L (\bm w)-\mc L(\wpop)\geq \frac{\alpha}{2}\norm{\bm w-\wpop}_{\bm \Gamma}^2\quad\text{for all } \bm w\in\mc B.
    \end{align*}
\end{lemma}
    
\begin{proof}
\textbf{Step 1: }    We first show that for every $\bm w\in\mc B$, $\nabla^2 \mc L(\bm w)\succeq \alpha\bm \Gamma$. Note that $\ell''(z)=e^z/(1+e^z)^2$ is even, so 
    \begin{align*}
        \nabla^2 \mc L(\bm w)=\mathbb{E}[\ell''(Y\bm X^\top \bm w)\bm X\bm X^{\top}]=\mathbb{E}[\ell''(\bm X^\top \bm w)\bm X\bm X^\top]\,.
    \end{align*}
    Hence, for any $\bm h\in\mathbb{R}^d$,
    \begin{align*}
        \bm h^\top \nabla^2\mc L(\bm w) \bm h=\mathbb{E}[\ell''(\bm X^\top\bm w)(\bm X^\top \bm h)^2].
    \end{align*}
   To show $\nabla^2 \mc L(\bm w)\succeq \alpha\bm \Gamma$, it suffices to show that
    \begin{align*}
        \mathbb{E}[\ell''(\bm X^\top\bm w)(\bm X^\top \bm h)^2]\geq \alpha\bm h^\top \bm \Gamma \bm h\,.
    \end{align*}
    Towards this objective, using the fact that $\ell''$ is even and decreasing on $[0,\infty)$, we have
    \begin{align*}
        \ell''(\bm X^\top \bm w)\geq \ell''(3R\sqrt{\norm{\bm \Gamma}})\mathbbm{1}\{|\bm X^\top \bm w|\leq 3R\sqrt{\norm{\bm \Gamma}}\},
    \end{align*}
    so it follows that
    \begin{align*}
        \mathbb{E}[\ell''(\bm X^\top\bm w)(\bm X^\top \bm h)^2]
        &\geq \ell''(3R\sqrt{\norm{\bm \Gamma}})\cdot \mathbb{E}[(\bm X^\top \bm h)^2\mathbbm{1}\{|\bm X^\top \bm w|
        \leq 3R\sqrt{\norm{\bm \Gamma}}\}] \\
        &= 2\alpha\cdot \mathbb{E}[(\bm X^\top \bm h)^2\mathbbm{1}\{|\bm X^\top \bm w|\leq 3R\sqrt{\norm{\bm \Gamma}}\}].
    \end{align*}
    It remains to show that 
    $\mathbb{E}[(\bm X^\top \bm h)^2\mathbbm{1}\{|\bm X^\top \bm w|\leq 3R\sqrt{\norm{\bm \Gamma}}\}]\ge \frac{1}{2}\bm h^\top \bm \Gamma \bm h$, or, equivalently,
    \begin{align}
    \label{eq: goal}
        \mathbb{E}[(\bm X^\top \bm h)^2\mathbbm{1}\{|\bm X^\top\bm w|> 3R\sqrt{\norm{\bm \Gamma}}\}]
        &=   \E[(\bX^\top \bm h)^2]-\mathbb{E}[(\bm X^\top \bm h)^2\mathbbm{1}\{|\bm X^\top \bm w|\leq 3R\sqrt{\norm{\bm \Gamma}}\}]\nonumber\\
        &=\bm h^\top\bGamma\bm h-\mathbb{E}[(\bm X^\top \bm h)^2\mathbbm{1}\{|\bm X^\top \bm w|\leq 3R\sqrt{\norm{\bm \Gamma}}\}]\nonumber\\
        &\leq\tfrac{1}{2}\bm h^\top \bm \Gamma \bm h. 
    \end{align}
    We do this as follows: starting with Cauchy-Schwarz,
    \begin{align}
    \label{eq: cs}
        \mathbb{E}[(\bm X^\top \bm h)^2\mathbbm{1}\{|\bm X^\top \bm w|> 3R\sqrt{\norm{\bm \Gamma}}\}]\leq\sqrt{\mathbb{E}[(\bm X^\top \bm h)^4]}\sqrt{\mathbb{P}(|\bm X^\top \bm w|> 3R\sqrt{\norm{\bm \Gamma}})}.
    \end{align}
    We bound each of the terms under the square roots. For the expectation, note that, since $\bm \varepsilon\sim \mc N(\bzero,\bSigma)$, $\wt Y^2=1$ and $\bvareps\perp \wt{Y}$, we have
    \begin{align}
    \label{eq:expectation}
        \mathbb{E}[(\bm X^\top \bm h)^4]&=\mathbb{E}[(\wt Y\bmu^\top \bm h+\bm\varepsilon^\top \bm h)^4]=(\bmu^\top \bm h)^4+6(\bmu^\top \bm h)^2(\bm h^\top \bm\Sigma\bm h)+3(\bm h^\top \bm\Sigma \bm h)^2 \nonumber\\
        &\leq 3((\bmu^\top \bm h)^2+(\bm h^\top \bm\Sigma\bm h))^2=3(\bm h^\top \bm \Gamma \bm h)^2.
    \end{align}%
    For the probability term, using Lemma \ref{lemma: subgauss covariates} and the fact that $\norm{\bm w}\leq R$,
    \begin{align}
    \label{eq:probability}
        \mathbb{P}(|\bm X^\top \bm w|>3R\sqrt{\norm{\bm \Gamma}})\leq 2\exp\left(-\frac{9R^2{\norm{\bm \Gamma}}}{2\bm w^\top \bm \Gamma \bm w}\right)\leq 2\exp\left(-\tfrac{9}{2}\right).
    \end{align}
    Plugging (\ref{eq:expectation}) and (\ref{eq:probability}) into (\ref{eq: cs}), we arrive at
    \begin{align*}
        \mathbb{E}[(\bm X^\top \bm h)^2\mathbbm{1}\{|\bm X^\top \bm w|> 3R\sqrt{\norm{\bm \Gamma}}\}]\leq (\sqrt{3}\bm h^\top \bm\Gamma \bm h) \cdot \left(\sqrt{2}\exp\left(-\tfrac{9}{4}\right)\right)\leq \tfrac{1}{2}\bm h^\top \bm \Gamma \bm h,
    \end{align*}
    which establishes (\ref{eq: goal}) and hence for every $\bm w\in\mc B$, $\nabla^2 \mc L(\bm w)\succeq \alpha\bm \Gamma$.

  \paragraph{Step 2:}  For the excess logistic risk lower bound, using the second-order Taylor expansion of $\bm w\in \mc{B}$ around $\wpop\in \mc B$, there is some $\bm \xi\in \mc{B}$ on the line connecting $\bm w$ and $\wpop$ such that
    \begin{align*}
        \mc L(\bm w)-\mc L(\wpop)&\geq \nabla \mc L(\wpop)^\top(\bm w-\wpop)+\tfrac{1}{2}(\bm w-\wpop)^\top\nabla^2\mc L(\bm \xi)(\bm w-\wpop)\\
        &\geq \tfrac{1}{2}(\bm w-\wpop)^\top\nabla^2\mc L (\bm \xi)(\bm w-\wpop) \ge \frac{\alpha}{2}\|\bw-\wpop\|_{\bGamma}^2\,.
    \end{align*}
    where the first term vanishes by
    the first-order optimality condition   $\nabla\mc{L}(\wpop)=\bzero$ (c.f. Lemma~\ref{lemma: w_pop}).
\end{proof}

We now restate  Lemma~\ref{cor:calibration} and prove it by combining Lemmas~\ref{lemma: conditional excess risk bound}--\ref{lemma: local strong convexity} above.
\calilemma*

\begin{proof}
    Recall  $\bm \Gamma=\bmu\bmu^\top+\bm \Sigma$ so   $\bm \Gamma\succeq \bm \Sigma$. Together with Lemma \ref{lemma: local strong convexity}, we have
    \begin{align*}
        \norm{\bm w-\bm w^*_{\mc{L}}}_{\bm \Sigma}^2\leq \norm{\bm w-\bm w^*_{\mc{L}}}_{\bm \Gamma}^2\leq \frac{2}{\alpha}(\mc L(\bm w)- \mc L(\wpop)).
    \end{align*}
    Using the assumption in \eqref{eq:calibration_assumption}, this implies
    \begin{align*}
        \norm{\bm w- \wpop}_{\bm \Sigma}\leq \frac{\rho^*\norms{\wb}}{2}=\frac{\rho^*\normso{\bmu}}{2}.
    \end{align*}
    We can now apply Lemma \ref{lemma: conditional excess risk bound} to conclude:
    \begin{align*}
        \mc{E}(\bw)-\ee(\wb) \le \frac{4(1-2p)}{\sqrt{2\pi}(\rho^*)^2\normso{\bmu}}\norms{\bm w-\wpop}^2\le \frac{4(1-2p)}{\sqrt{2\pi}(\rho^*)^2\normso{\bmu}} \frac{2}{\alpha}\left(\mc L(\bm w)- \mc L(\wpop)\right).
    \end{align*}
\end{proof}

By Lemma \ref{ESproperties},   the early-stopped iterate $\bm w_{\tau}$ lies in the ball $\mc B$, so  Lemma \ref{cor:calibration} reduces the problem of   bounding  the population excess zero-one risk to   bounding   the population excess logistic risk $\mc{L}(\bw_\tau)-\mc{L}(\wpop)$, which admits the decomposition below in \eqref{eq:risk_decomposition}. 
By Lemma~\ref{ESproperties} (ii), Term I of \eqref{eq:risk_decomposition} is non-positive. Hence, it remains to control the uniform deviation between the empirical and population excess logistic risks, i.e., Term II of \eqref{eq:risk_decomposition},
and we do so using empirical process theory. 

\begin{align}
    \mc{L}(\bw_\tau)-\mc{L}(\wpop)=\underbrace{\left(\wh{\mc{L}}(\bw_\tau)-\wh{\mc{L}}(\wpop)\right)}_{\text{Term I}\le0\text{ by Lemma~\ref{ESproperties} (ii)}}+\underbrace{\left(\mc{L}(\bw_\tau) - \mc{L}(\wpop)\right) -\left(\wh{\mc{L}}(\bw_\tau)-\wh{\mc{L}}(\wpop)\right)}_{\text{Term II: controlled by empirical process theory}}\,.\label{eq:risk_decomposition}
\end{align}

To this end, a direct uniform convergence bound over the entire ball $\mc{B}$ would be too loose. To obtain  a sharper upper bound in terms of  the \emph{effective rank} of $\bm \Sigma$, we exploit the fact that $\bw_\tau$ has small excess risk, confining it to a thin shell around $\wpop$ within $\mc{B}$.   This localisation yields a smaller function class with better-controlled complexity, 
which we detail in Section~\ref{appx:localisation_concentration}.

For clarity,  we first recall in Section~\ref{appx:effective_rank} the definition of the effective rank and establish a key  spectral property of the empirical covariance of the covariates $(\bX_i)_{i=1}^n$ in terms of the effective rank.  We then present the localisation and concentrations proofs in Section \ref{appx:localisation_concentration}.  

\subsection{Effective Rank and Empirical Covariance}\label{appx:effective_rank}
We define the effective rank and establish two  spectral properties in terms of the effective rank (Lemmas~\ref{lemma: effective ranks}--\ref{lemma: E cond}) which will be used in the sequel.
\subsubsection{Effective Rank of Population Covariances}
Recall that in our Gaussian mixture model,  the population covariance of the noise variable $\bvareps$ is $\bSigma$, while that of the covariates $\bX$ is $\bGamma=\bSigma+\bmu\bmu^\top$. Both matrices induce geometries that play a role  in our analysis, so we introduce the effective rank for each and use $\bSigma$ or $\bm \Gamma$ as a superscript to avoid any confusion.    For any $\eta>0$, define the regularised matrices  $\bm \Sigma_\eta:=\bm \Sigma+\eta \bm I_d$ and $\bm \Gamma_\eta:=\bm \Gamma+\eta \bm I_d$, and  define
    \begin{align*}
    r_{\operatorname{eff}}^{\bm \Sigma}(\eta)& :=\operatorname{tr}(\bSigma \bSigma_{\eta}^{-1})=\operatorname{tr}(\bm \Sigma(\bm \Sigma+\eta \bm I_d)^{-1})=\sum_{i=1}^d\frac{\lambda_i(\bm \Sigma)}{\lambda_i(\bm \Sigma)+\eta}\,,\\
    r_{\operatorname{eff}}^{\bm \Gamma}(\eta)&:=\operatorname{tr}(\bm \Gamma \geta^{-1})=\operatorname{tr}(\bm \Gamma(\bm \Gamma+\eta \bm I_d)^{-1})=\sum_{i=1}^d\frac{\lambda_i(\bm \Gamma)}{\lambda_i(\bm \Gamma)+\eta}.
\end{align*}

Note that throughout the main text  we only work with the effective rank associated with $\bSigma $, and so we write $r_{\operatorname{eff}}(\eta)$ to refer to $r_{\operatorname{eff}}^{\bSigma}(\eta)$ for brevity. In the appendices, both $r_{\operatorname{eff}}^{\bSigma}(\eta)$ and $r_{\operatorname{eff}}^{\bGamma}(\eta)$ appear, so we retain the superscripts for clarity. 

We show in Lemma~\ref{lemma: effective ranks} that the two effective ranks differ by at most 1. 

\begin{lemma}[Effective ranks of $\bSigma$ and $\bGamma$]
\label{lemma: effective ranks}
    For any $\eta>0$,
    \begin{align*}
        r_{\operatorname{eff}}^{\bm \Sigma}(\eta)\leq r_{\operatorname{eff}}^{\bm \Gamma}(\eta)\leq r_{\operatorname{eff}}^{\bm \Sigma}(\eta)+1.
    \end{align*}
\end{lemma}

\begin{proof}
    It suffices to show $0\leq r_{\operatorname{eff}}^{\bm \Gamma}(\eta)-r_{\operatorname{eff}}^{\bm \Sigma}(\eta)\leq 1$. 
  Note the  identities $\bm \Sigma\bm\Sigma_\eta^{-1}=\bm I_d-\eta\bm\Sigma_{\eta}^{-1}$ and $\bm \Gamma\geta^{-1}=\bm I_d-\eta\geta^{-1}$. Then the Sherman-Morrison formula applied to $\bSigma_{\eta}\succ 0$ gives
    \begin{align*}
        (\bm \Sigma_\eta+\bmu\bmu^\top)^{-1}=\bm \Sigma_{\eta}^{-1}-\frac{\bm \Sigma_{\eta}^{-1}\bmu\bmu^\top \bm\Sigma_\eta^{-1}}{1+\bmu^\top \bm \Sigma_{\eta}^{-1}\bmu},
    \end{align*}
    so we have  
    \begin{align*}
        r_{\operatorname{eff}}^{\bm \Gamma}(\eta)-r_{\operatorname{eff}}^{\bm \Sigma}(\eta)&=\operatorname{tr}(\bm \Gamma\geta^{-1}-\bm\Sigma\bm\Sigma_{\eta}^{-1})\\
        &=\operatorname{tr}(\bm I_d-\eta\geta^{-1}-\bm I_d+\eta\bm\Sigma_\eta^{-1})\\
        &=\eta\operatorname{tr}(\bm\Sigma_{\eta}^{-1}-(\bm \Sigma_\eta+\bmu\bmu^\top)^{-1})\\
        &=\frac{\eta\operatorname{tr}(\bm\Sigma_{\eta}^{-1}\bmu\bmu^\top \bm\Sigma_{\eta}^{-1})}{1+\bmu^\top \bm\Sigma_{\eta}^{-1} \bmu}\\
        &=\frac{\eta\bmu^\top \bm\Sigma_{\eta}^{-2}\bmu}{1+\bmu^\top \bm\Sigma_{\eta}^{-1}\bmu}.
    \end{align*}
    Since $\bm \Sigma_{\eta}^{-1}, \bm \Sigma_{\eta}^{-2}\succ 0$, we clearly have $r_{\operatorname{eff}}^{\bm \Gamma}(\eta)-r_{\operatorname{eff}}^{\bm \Sigma}(\eta)\geq 0$. Next, to show $ r_{\operatorname{eff}}^{\bm \Gamma}(\eta)-r_{\operatorname{eff}}^{\bm \Sigma}(\eta)\le 1$, it suffices to show
    \begin{align*}
        {\eta\bmu^\top \bm\Sigma_{\eta}^{-2}\bmu}\leq {\bmu^\top \bm \Sigma_{\eta}^{-1}\bmu}.
    \end{align*}
    This holds because $\bm \Sigma_{\eta}\succeq \eta \bm I_d$, so that $\bm\Sigma_{\eta}^{-2}\preceq \eta^{-1}\bm \Sigma_{\eta}^{-1} $, i.e., $\eta \bm \Sigma_{\eta}^{-2}\preceq\bm \Sigma_{\eta}^{-1}$.
\end{proof}

\subsubsection{Spectral Control of the Empirical Covariance via Effective Rank}
Now, define the empirical covariance of the covariates as
\begin{align*}\bm{\widehat{\Gamma}}:=\frac{1}{n}\sum_{i=1}^n \bm X_i\bm X_i^\top,
\end{align*}
and define the good event
\begin{align}\label{eq:def_E_cov_event}
    E_{\operatorname{cov}}:=\{\lvert\lvert{\geta^{-1/2}\widehat{\bGamma}\geta^{-1/2}}\rvert\rvert\leq 2\}\,,
\end{align}
on which   the empirical covariance is well-behaved in that it does not introduce spurious high-variance directions, and is spectrally controlled relative to the regularised population covariance $\bGamma_\eta$. 
We  show in Lemma~\ref{lemma: E cond} that the event $E_{\operatorname{cov}}$ holds with high probability, whenever  the  sample size $n$ is sufficiently large relative to $r_{\operatorname{eff}}^{\bGamma}(\eta)$. 

\begin{lemma}[Spectral control of the empirical covariance]
\label{lemma: E cond}
    There exists a universal constant $\c>0$ such that the following holds: %
    Fix $\eta>0$. For every $\delta\in(0,1)$, if
    \begin{align*}
        n\geq \C(r_{\operatorname{eff}}^{\bm \Gamma}(\eta)+\log(1/\delta))\,,
    \end{align*}
    then we have $\mathbb{P}(E_{\operatorname{cov}})\geq 1-\delta$.
    
\end{lemma}

\begin{proof}
    We split this proof into two steps:

    \paragraph{Step 1:} Let $\bm Z_i=\geta^{-1/2} \bm X_i$ for $1\leq i\leq n$, so their empirical covariance takes the form
    \begin{align*}
        \frac{1}{n}\sum_{i=1}^n \bm Z_i\bm Z_i^\top = \geta^{-1/2}\widehat{\bGamma}\geta^{-1/2}\,,
    \end{align*}
    and define its population covariance as
    \begin{align*}
        \bm \Upsilon=\mathbb{E}[\bm Z_i\bm Z_i^\top]=\geta^{-1/2}\bm \Gamma \geta^{-1/2}=\bGamma\geta^{-1}.
    \end{align*}
    where note $\norm{\bm \Upsilon}\leq 1$ and $\operatorname{tr}(\bm\Upsilon)=r_{\operatorname{eff}}^{\bm \Gamma}(\eta)$. 

    In this step of the proof, we will apply  \cite[Theorem 9]{koltchinskii2017concentration} to bound the spectral norm deviation of the sample covariance from the population covariance $\|\bGamma_\eta^{-1/2}\wh{\bGamma}\bGamma_\eta^{-1/2} - \bm \Upsilon\|$. To verify the conditions in the hypothesis of this theorem, note that the $\bm Z_i$'s are:
    \begin{itemize}[leftmargin=*]
        \item weakly-square integrable, i.e., $\mathbb{E}[(\bm Z_i^\top \bm u)^2]<\infty$ for all $\bm u\in\mathbb{R}^d$, which holds as $\norm{\bm \Upsilon}\leq 1$;
        \item centred and have covariance $\bm\Upsilon$ and empirical covariance $\geta^{-1/2}\widehat{\bm \Gamma}\geta^{-1/2}$;
        \item subgaussian in the sense of \cite[Definition 2]{koltchinskii2017concentration}, by Lemma \ref{lemma: subgauss covariates} applied with $\bw=\bGamma_{\eta}^{-1/2}\bu$;
        \item pregaussian, i.e., there exists a centred Gaussian random vector with the same covariance $\bm \Upsilon$. Indeed $\mc{N}(\bzero,\bm \Upsilon)$ satisfies this condition.
    \end{itemize}
    Then,  by \cite[Theorem 9]{koltchinskii2017concentration},  there exists a universal constant $\c'>0$ such that, for all $t\ge1$, with probability at least $1-e^{-t}$
    \begin{align*}
        \norm{\geta^{-1/2}\widehat{\bm \Gamma}\geta^{-1/2} - \bm\Upsilon}\leq \c' \norm{\bm \Upsilon}\max\left(\sqrt{\frac{\widetilde r(\bm \Upsilon)}{n}},\frac{\widetilde r(\bm \Upsilon)}{n},\sqrt{\frac{t}{n}},\frac{t}{n}\right),
    \end{align*}
    where, using the notation from \cite[Definition 1]{koltchinskii2017concentration}, 
    \begin{align*}
        \widetilde{r}(\bm \Upsilon)=\frac{\operatorname{tr}(\bm \Upsilon)}{\norm{\bm \Upsilon}}=\frac{r_{\operatorname{eff}}^{\bm \Gamma}(\eta)}{\norm{\bm \Upsilon}}.
    \end{align*}
    Since $\norm{\bm \Upsilon}\leq 1$, we have 
    \begin{align*}
        \norm{\geta^{-1/2}\widehat{\bm \Gamma}\geta^{-1/2} - \bm\Upsilon}\leq \c' \max\left(\sqrt{\frac{r_{\operatorname{eff}}^{\bm \Gamma}(\eta)}{n}},\frac{r_{\operatorname{eff}}^{\bm \Gamma}(\eta)}{n},\sqrt{\frac{t}{n}},\frac{t}{n}\right).
    \end{align*}
   \paragraph{Step 2:} This step establishes  $\P(\lvert\lvert{\geta^{-1/2}\widehat{\bGamma}\geta^{-1/2}}\rvert\rvert\leq 2)\ge 1-\delta$. Combining the last display with the triangle inequality and setting $t=\log(1/\delta)$, we obtain with probability at least $1-\delta$, 
    \begin{align*}
        \norm{\geta^{-1/2}\widehat{\bm \Gamma}\geta^{-1/2}}&\leq \norm{\bm\Upsilon}+\norm{\geta^{-1/2}\widehat{\bm \Gamma}\geta^{-1/2} - \bm \Upsilon}\\
        &\leq 1+ \c' \max\left(\sqrt{\frac{r_{\operatorname{eff}}^{\bm \Gamma}(\eta)}{n}},\frac{r_{\operatorname{eff}}^{\bm \Gamma}(\eta)}{n},\sqrt{\frac{\log(1/\delta)}{n}},\frac{\log(1/\delta)}{n}\right)\,.
    \end{align*}
   Hence,   if  $n\geq \C(r_{\operatorname{eff}}^{\bm \Gamma}(\eta)+\log(1/\delta))$ for a sufficiently large universal constant $\c$, then 
    \begin{align*}
        \c' \max\left(\sqrt{\frac{r_{\operatorname{eff}}^{\bm \Gamma}(\eta)}{n}},\frac{r_{\operatorname{eff}}^{\bm \Gamma}(\eta)}{n},\sqrt{\frac{\log(1/\delta)}{n}},\frac{\log(1/\delta)}{n}\right)\leq 1,
    \end{align*}
    and this concludes the proof.
\end{proof}

\subsection{Localisation and Concentrations}\label{appx:localisation_concentration}

For brevity, we   adopt standard notation from  empirical process theory. For each $\bw\in \bbR^d$, define $\gw:\mathbb{R}^d\times\{\pm 1\}\rightarrow\mathbb{R}$ by
\begin{align*}
    g_{\bm w}(\bm x,y):=\ell(y\bm x^\top \bm w)-\ell(y\bm x^\top \wpop),
\end{align*}
and let $P$ and $P_n$ denote the population and empirical expectation operators, so that
\begin{align*}
    P\gw&:=\mathbb{E}[g_{\bm w}(\bm X,Y)]=\mc L(\bm w)-\mc{L}(\wpop),\\
    P_n\gw&:=\frac{1}{n}\sum_{i=1}^n \gw(\bm X_i,Y_i)=\lle(\bm w)-\lle(\wpop).
    \end{align*}
Then, 
\begin{align*}
    (P-P_n)\gw&=P\gw-P_n\gw =\left(\mc{L}(\bw)-\mc{L}(\wpop)\right) - \left(\wh{\mc{L}}(\bw) - \wh{\mc{L}}(\wpop)\right)\,
\end{align*}
is the term that we need to bound in the risk decomposition in \eqref{eq:risk_decomposition}. 
For any $r>0$, we  work over the localised function class %
\begin{align*}
    \mc{G}(r)=\{\gw: \bm w\in \mathcal{B}\text{ and } P\gw\leq r\}\,,
\end{align*}
which contains $g_{\wpop}$ for every $r>0$ because $Pg_{\bw_{\mc{L}}^*}=0$.  This class consists of functions indexed by parameters $\bw$ lying in a localised shell within the ball $\mc{B}$.

Notation-wise, recall the ball   $\mc{B}=\{\bw\in \bbR^d: \|\bw\|\le R\}$ with   $R=4\norm{\wb}+1$ from \eqref{eq:def_Ball_R}, and  the  local strong-convexity parameter $\alpha=\frac{1}{2}\ell''(3R\sqrt{\norm{\bm \Gamma}})$ of $\mc L(\cdot)$ over $\mathcal{B}$ from   Eq.~\eqref{eq: curvature}.

\subsubsection{Useful Properties of   $\mc{G}(r)$}

We first present some useful properties of  the function class $\mc{G}(r)$, which we will use in our proofs later in this section.

\begin{lemma}
\label{lemma: simple facts}
    If $\gw\in \mc{G}(r)$, then $\norm{\bm w-\wpop}\leq 2R$ and
    \begin{align*}
        \norm{\bm w-\wpop}^2_{\bm\Gamma}\leq \frac{2r}{\alpha}.
    \end{align*}
\end{lemma}

\begin{proof}
    The first inequality is simply by the fact that $\bm w,\wpop\in \mc B$ and the triangle inequality. The second inequality is obtained by rearranging the conclusion of Lemma \ref{lemma: local strong convexity} and using $P\gw\leq r$.
\end{proof}

\begin{lemma}
\label{lemma: gw variance}
    For every $\gw\in \mc{G}(r)$, $|\gw(\bm X,Y)|\leq |\bm X^\top (\bm w-\wpop)|$.
    As a consequence, 
    \begin{align*}
        \mathbb{E}[\gw(\bm X,Y)^2]\leq \frac{2r}{\alpha}.
    \end{align*}
\end{lemma}

\begin{proof}
    The first inequality is a consequence of the fact that $\ell$ is $1$-Lipschitz and $Y\in\{\pm 1\}$, whereas for the second inequality we can apply Lemma \ref{lemma: local strong convexity} to obtain
    \begin{align*}
        \mathbb{E}[\gw(\bm X,Y)^2]\leq \mathbb{E}[(\bm X^\top (\bm w-\wpop))^2]=\norm{\bm w-\wpop}_{\bm \Gamma}^2\leq \frac{2 Pg_{\bw}}{\alpha}\leq \frac{2r}{\alpha}.
    \end{align*}
\end{proof}

\subsubsection{Empirical Excess Logistic Risk Concentration on $\mc{G}(r)$}

On the event $E_{\operatorname{cov}}$, defined  in~\eqref{eq:def_E_cov_event} and characterised in Lemma~\ref{lemma: E cond}, we provide the first  uniform concentration bound: it controls how much  the empirical   excess logistic risk $P_ng_{\bw}$ deviates from the population   excess logistic risk  $Pg_{\bw}$ uniformly over all  functions $g_{\bw}\in \mc{G}(r)$. 
We introduce an auxiliary parameter $\lambda>0$ that we will later optimise over.

\begin{lemma}[Concentration on $\mc{G}(r)$]
\label{lemma: uniform shell bound}
    Assume $E_{\operatorname{cov}}$ holds. Fix $r,n,\lambda>0$ and %
    define
    \begin{align*}
        \varrho_r=\frac{2r}{\alpha}+4\eta R^2.
    \end{align*}
    Then, for any $\delta\in(0,1)$, with probability at least $1-\delta$,
    \begin{align*}
        \underset{\gw\in \mc{G}(r)}{\sup}(P-P_n)\gw\leq \frac{r_{\operatorname{eff}}^{\bm \Gamma}(\eta)}{\lambda n}+4\lambda\varrho_r+5\sqrt{\frac{\varrho_r\log(1/\delta)}{n}}
    \end{align*}
\end{lemma}

\begin{proof}The proof follows a standard  truncation argument. We clip each function $g_{\bw}\in \mc{G}(r)$ at a threshold $m_r$ into a bounded part $\overline{g_{\bw}}$ and a tail residual $s_{\bw}$. The bounded part is controlled via Rademacher complexity and a Bernstein-type inequality (Steps 1-2), while the tail is bounded in Step 3.  

    Introduce the clipping threshold
    \begin{align*}
        m_r=\sqrt{\frac{n\varrho_r}{\log(1/\delta)}},
    \end{align*}
    and use this to clip each function $g_{\bw}$ into $\overline\gw$ with a  tail residue $s_{\bw}$, respectively defined as
    \begin{align*}
        \overline \gw:=\max(-m_r,\min(\gw,m_r)),\quad \text{and}\quad s_{\bm w}:=\gw - \overline{\gw}\quad \text{with }|s_{\bm w}|\leq\tw:=(|\gw|-m_r)_+\ .
    \end{align*}
     Denote the corresponding  class of clipped functions  by
     \begin{align*}
         \overline{\mc{G}}(r)=\{\overline \gw:\gw\in \mc{G}(r)\}.
     \end{align*}

    \paragraph{Step 1: Bounding $\mathbb{E}\texttt{Rad}_n\overline{\mc{G}}(r)$:} 
    Recall that for a class $\mathcal{F}$ of functions $f:\mathbb{R}^d\times \{\pm 1\}\rightarrow\mathbb{R}$, its Rademacher complexity with respect to the training samples $(\bm X_i, Y_i)_{i=1}^n$ is defined as
    \begin{align*}
        \texttt{Rad}_n\mathcal{F}:=\underset{f\in\mathcal{F}}{\sup}\;\frac{1}{n}\sum_{i=1}^n \epsilon_i f(\bm X_i,Y_i),
    \end{align*}
    where $\epsilon_i$ are $\pm 1$ Rademacher random variables, 
    and let $\bm \epsilon=(\epsilon_1, \dots, \epsilon_n)$.
    
    Since the map $g\mapsto \max(-m_r,\min(g,m_r))$ is 1-Lipschitz, a standard contraction result (for example, Theorem A.6 in \cite{Bartlett_2005}) gives that, conditional on the samples, the expectations over the Rademacher variables $\bm \epsilon$ satisfy
    \begin{align*}
        \mathbb{E}_{\bm \epsilon} \texttt{Rad}_n \overline{\mc{G}}(r)\leq \mathbb{E}_{\bm\epsilon} \texttt{Rad}_n \mc{G}(r)\,.
    \end{align*} 
    Taking an expectation over all samples,
    \begin{align*}
        \mathbb{E} \texttt{Rad}_n \overline{\mc{G}}(r)\leq \mathbb{E} \texttt{Rad}_n \mc{G}(r),
    \end{align*}
    and we next upper bound the right hand side.

    To do so,   note that, since $\wpop$ does not depend on $g_{\bw}$, we have
    \begin{align*}
        \texttt{Rad}_n \mc{G}(r)&=\underset{\gw\in \mc{G}(r)}{\sup}\frac{1}{n}\sum_{i=1}^n \epsilon_i\left(\ell(Y_i\bm X_i^\top \bm w)-\ell(Y_i\bm X_i^\top \wpop)\right)\\
        &=-\frac{1}{n}\sum_{i=1}^n \epsilon_i\ell(Y_i\bm X_i^\top \wpop) +\underset{\gw\in \mc{G}(r)}{\sup}\frac{1}{n}\sum_{i=1}^n \epsilon_i\ell(Y_i\bm X_i^\top \bm w),
    \end{align*}
    where  taking expectation over $\bm\epsilon$ (and then over the samples) makes the first term vanish, resulting in
    \begin{align*}
        \mathbb{E}\texttt{Rad}_n \mc{G}(r)=\mathbb{E}\underset{\gw\in \mc{G}(r)}{\sup}\frac{1}{n}\sum_{i=1}^n \epsilon_i\ell(Y_i\bm X_i^\top \bm w).
    \end{align*}
    Now, since $\ell$ is $1$-Lipschitz, applying contraction again and using the fact that $\epsilon_i Y_i$ are still Rademacher random variables, we obtain
    \begin{align*}
        \mathbb{E}\texttt{Rad}_n \mc{G}(r)\leq \mathbb{E}\underset{\gw\in \mc{G}(r)}{\sup}\frac{1}{n}\sum_{i=1}^n \epsilon_iY_i\bm X_i^\top\bm w=\mathbb{E}\underset{\gw\in \mc{G}(r)}{\sup}\frac{1}{n}\sum_{i=1}^n \epsilon_i\bm X_i^\top \bm w.
    \end{align*}
    Next, we can subtract  $\sum_{i=1}^n\epsilon_i\bm X_i^\top \wpop$ because in expectation this is $0$, so  we arrive at
    \begin{align*}
        \mathbb{E}\texttt{Rad}_n \mc{G}(r)\leq \mathbb{E}\underset{\gw\in \mc{G}(r)}{\sup}\frac{1}{n}\sum_{i=1}^n \epsilon_i\bm X_i^\top (\bm w-\wpop)=\mathbb{E}\underset{\gw\in \mc{G}(r)}{\sup} \bm V^\top(\bm w-\wpop).
    \end{align*}
    where we have introduced $\bm V=\frac{1}{n}\sum_{i=1}^n \epsilon_i\bm X_i\in \bbR^d$.

    Next, recall $\geta=\bm \Gamma+\eta \bm I_d$ and note that for any $\bm h\in\mathbb{R}^d$ and any $\lambda>0$,
    \begin{align*}
        \bm V^\top \bm h=\bm V^\top \bm h-\lambda \norm{\geta^{1/2}\bm h}^2+\lambda\norm{\geta^{1/2}\bm h}^2\leq \underset{\bm u\in\mathbb{R}^d}{\sup}\left(\bm V^\top \bm u-\lambda\norm{\geta^{1/2}\bm u}^2\right)+\lambda\norm{\geta^{1/2}\bm h}^2,
    \end{align*}
    therefore, setting $\bm h=\bm w-\wpop$ and taking supremum over $g_{\bw}\in \mc{G}(r)$ yields
    \begin{align*}
        \underset{\gw\in \mc{G}(r)}{\sup}\bm V^\top (\bm w-\wpop)\leq \underset{\bm u\in\mathbb{R}^d}{\sup}\left(\bm V^\top \bm u-\lambda\norm{\geta^{1/2}\bm u}^2\right)+\lambda\underset{\gw\in \mc G(r)}{\sup}\norm{\geta^{1/2}(\bm w-\wpop)}^2.
    \end{align*}
    Optimising over $\bm u$ in the first supremum we get
    \begin{align*}
        \underset{\bm u\in\mathbb{R}^d}{\sup}\left(\bm V^\top \bm u-\lambda\norm{\geta^{1/2}\bu}^2\right)=\frac{1}{4\lambda}\norm{\geta^{-1/2}\bm V}^2,
    \end{align*}
    whose expectation is
    \begin{align*}
        \frac{1}{4\lambda}\mathbb{E}\norm{\geta^{-1/2}\bm V}^2&=\frac{1}{4\lambda}\mathbb{E}[\bm V^\top \geta^{-1}\bm V]=\frac{1}{4\lambda}\mathbb{E}[\operatorname{tr}(\bm V\bm V^\top \geta^{-1})]=\frac{1}{4\lambda}\operatorname{tr}(\mathbb{E}[\bm V\bm V^\top]\geta^{-1})\\
        &\stackrel{\text{(i)}}{=}\frac{1}{4\lambda n}\operatorname{tr}(\bm \Gamma \geta^{-1})\stackrel{\text{(ii)}}{=}\frac{r_{\operatorname{eff}}^{\bm \Gamma}(\eta)}{4\lambda n}\,,
    \end{align*}
    where (i) uses $ \E[\bm V\bm V^\top]=\bGamma/n$ and (ii) applies the definition $r_{\operatorname{eff}}^{\bGamma}(\eta):=\operatorname{tr}(\bGamma\bGamma_{\eta}^{-1})$.  
    Finally, for the second supremum, using the fact that $\geta=\bm \Gamma+\eta \bm I_d$,
    \begin{align}
    \label{eq: norm bw-w}
        \lambda\underset{\gw\in \mc G(r)}{\sup}\norm{\geta^{1/2}(\bm w-\wpop)}^2&=\lambda\underset{\gw\in \mc G(r)}{\sup}(\norm{\bm w-\wpop}_{\bm \Gamma}^2+\eta \norm{\bm w-\wpop}^2)\nonumber \\
        &\leq\lambda\left(\frac{2r}{\alpha}+4\eta R^2\right)=\lambda \varrho_r,
    \end{align}
    where   the inequality holds because of the two inequalities from Lemma \ref{lemma: simple facts}. Putting everything together, we arrive at
    \begin{align}
    \label{eq: e rademacher}
        \mathbb{E}\texttt{Rad}_n\overline{\mc{G}}(r)\leq \frac{r_{\operatorname{eff}}^{\bm \Gamma}(\eta)}{4\lambda n}+\lambda \varrho_r\,.
    \end{align}

    \paragraph{Step 2: High-probability bound for $\overline{g_{\bw}}$:} Each $\overline \gw$ takes values in $[-m_r,m_r]$, an interval of  length  $2m_r$. Additionally, clipping cannot increase the variance, hence by Lemma \ref{lemma: gw variance},
    \begin{align*}
        \underset{\overline \gw\in\overline{\mc{G}}(r)}{\sup}\operatorname{Var}(\overline \gw)\leq \underset{\gw\in\mc{G}(r)}{\sup}\operatorname{Var}(\gw)\leq \underset{\gw\in \mc{G}(r)}{\sup}\mathbb{E}[(\gw)^2]\leq\frac{2r}{\alpha}\leq \varrho_r.
    \end{align*}
    We can now apply Theorem 2.1 in \cite{Bartlett_2005} (with $\mathcal{F}=\overline{\mc{G}}(r), x=\log(1/\delta),\alpha=1,a=-m_r,b=m_r, r=\varrho_r$) to get that, with probability at least $1-\delta$,
    \begin{align*}
        \underset{\overline \gw\in\overline{\mc{G}}(r)}{\sup}(P-P_n)\overline \gw\leq 4\mathbb{E} \texttt{Rad}_n\overline{\mc{G}}(r)+\sqrt{\frac{2\varrho_r\log(1/\delta)}{n}}+\frac{8}{3}m_r\frac{\log(1/\delta)}{n}.
    \end{align*}
    Using the bound (\ref{eq: e rademacher}) together with $m_r\log(1/\delta)/n=\sqrt{\varrho_r\log(1/\delta)/n}$, we have that, with probability at least $1-\delta$,
    \begin{align}
        \label{eq: clipped hp}
        \underset{\overline \gw\in\overline{\mc{G}}(r)}{\sup}(P-P_n)\overline{g_{\bw}}\leq\frac{r_{\operatorname{eff}}^{\bm \Gamma}(\eta)}{\lambda n}+4\lambda \varrho_r+(\sqrt{2}+8/3)\sqrt{\frac{\varrho_r\log(1/\delta)}{n}}.
    \end{align}

    \textbf{Step 3: Bounding the tail residual:} Starting from $(|g|-m_r)_+\leq g^2/4m_r$ valid for all $g\in\mathbb{R}$, Lemma \ref{lemma: gw variance} gives
    \begin{align}
    \label{eq: pop tail}
         \underset{\gw\in \mc{G}(r)}{\sup} P(|\gw|-m_r)_+=\underset{\gw\in \mc{G}(r)}{\sup} P\tw\leq\frac{\varrho_r}{4m_r}=\frac{1}{4}\sqrt{\frac{\varrho_r\log(1/\delta)}{n}}.
    \end{align}
    On the empirical side, the same inequality and Lemma \ref{lemma: gw variance} gives
    \begin{align*}
        \tw(\bm X_i,Y_i)=(|\gw(\bm X_i,Y_i)|-m_r)_+\leq \frac{\gw(\bm X_i,Y_i)^2}{4m_r}\leq \frac{(\bm w-\wpop)^\top \bm X_i\bm X_i^\top (\bm w-\wpop)}{4m_r},
    \end{align*}
    so that
    \begin{align*}
        P_n\tw\leq \frac{1}{4m_r}(\bm w-\wpop)^\top \widehat \Gamma (\bm w-\wpop).
    \end{align*}
    Now, since $(\bm w-\wpop)^\top \widehat{\bm \Gamma}(\bm w-\wpop)\leq \norm{\geta^{1/2}(\bm w-\wpop)}^2\norm{\geta^{-1/2}\widehat{\bm \Gamma}\geta^{-1/2}}$, with the assumption that the event $E_{\operatorname{cov}}$ holds and the fact that, by  (\ref{eq: norm bw-w}), $\norm{\geta^{1/2}(w-\wpop)}^2\leq \varrho_r$,
    \begin{align}
        \label{eq: emp tail}
        \underset{\gw\in \mc{G}(r)}{\sup}P_n\tw\leq \frac{2\varrho_r}{4m_r}=\frac{1}{2}\sqrt{\frac{\varrho_r\log(1/\delta)}{n}}.
    \end{align}

    \textbf{Step 4: Conclusion:} Using the decomposition $\gw=\overline \gw+s_{\bm w}$ with $|s_{\bm w}|\leq \tw$, we have
    \begin{align*}
        (P-P_n)\gw=(P-P_n)\overline\gw+Ps_{\bm w}-P_ns_{\bm w}\leq (P-P_n)\overline\gw+P\tw+P_n\tw.
    \end{align*}
    Taking a supremum over $\gw\in \mc{G}(r)$ and using (\ref{eq: clipped hp}), (\ref{eq: pop tail}), and (\ref{eq: emp tail}), we get that, with probability $1-\delta$,
    \begin{align*}
        \underset{\gw\in\mc{G}(r)}{\sup}(P-P_n)\gw\leq \frac{r_{\operatorname{eff}}^{\Gamma}(\eta)}{\lambda n}+4\lambda\varrho_r+(\sqrt{2}+35/12+1/2)\sqrt{\frac{\varrho_r\log(1/\delta)}{n}}.
    \end{align*}
    The conclusion follows by noting that $\sqrt{2}+35/12+1/2<5$.
\end{proof}

\bigskip

We now can pick $\lambda$ conveniently in Lemma \ref{lemma: uniform shell bound} and process it further to have only $n^{-1}$ rates:

\begin{cor}
\label{cor: uniform shell bound}
    Assume $E_{\operatorname{cov}}$ holds. Fix any confidence $\delta>0$ and any $r,\eta>0$.
    Then, with probability $1-\delta$,
    \begin{align*}
        \underset{\gw\in \mc{G}(r)}{\sup}(P-P_n)\gw\leq \frac{128 (r_{\operatorname{eff}}^{\bm \Sigma}(\eta)+1)}{\alpha n}+\frac{1}{8}r+\frac{3}{16}\alpha\eta R^2+600\frac{\log(1/\delta)}{\alpha n}.
    \end{align*}
\end{cor}

\begin{proof}
    First, take $\lambda=\alpha/128$ in Lemma \ref{lemma: uniform shell bound} and then use $\sqrt{a+b}\leq \sqrt{a}+\sqrt{b}$ to arrive at
    \begin{align}
    \label{ineq: cor partial upper}
        \underset{\gw\in \mc{G}(r)}{\sup}(P-P_n)\gw\leq \frac{128 r_{\operatorname{eff}}^{\bm \Gamma}(\eta)}{\alpha n}+\frac{1}{16}r+\frac{1}{8}\alpha\eta R^2+\sqrt{\frac{50r\log(1/\delta)}{\alpha n}}+\sqrt{\frac{100\eta R^2\log(1/\delta)}{n}},
    \end{align}
    Next, applying the inequality $\sqrt{ab}\leq a/16+4b$ to each of the square root terms we have
    \begin{align*}
        \sqrt{\frac{50r\log(1/\delta)}{\alpha n}}&\leq \frac{1}{16}r+200\frac{\log(1/\delta)}{\alpha n},\\
        \sqrt{\frac{100\eta R^2\log(1/\delta)}{n}}&=\sqrt{\frac{100\alpha\eta R^2\log(1/\delta)}{\alpha n}}\leq \frac{1}{16}\alpha\eta R^2+400\frac{\log(1/\delta)}{\alpha n}.
    \end{align*}
    Substituting these into (\ref{ineq: cor partial upper}), we arrive at
    \begin{align*}
        \underset{\gw\in \mc{G}(r)}{\sup}(P-P_n)\gw\leq \frac{128 r_{\operatorname{eff}}^{\bm\Gamma}(\eta)}{\alpha n}+\frac{1}{8}r+\frac{3}{16}\alpha\eta R^2+600\frac{\log(1/\delta)}{\alpha n},
    \end{align*}
    and the conclusion follows by using $r_{\operatorname{eff}}^{\bm \Gamma}\leq r_{\operatorname{eff}}^{\bm \Sigma}+1$ (c.f.\ Lemma \ref{lemma: effective ranks}).

\end{proof}

\bigskip

Now, having proved empirical logistic risks deviation bounds on the shell $\mc{G}(r)$, we next extend them to the whole of $\mc B=\{\bm w:\norm{\bm w}\leq R\}$ using a peeling argument over $r$.

\subsubsection{Empirical Excess Logistic Risk Concentration on $\mc B$}\label{sec:empirical_excess_risk_deviation}

\begin{lemma}[Concentration on $\mc{B}$]
\label{lemma: final concentration bound}
    Fix $\delta>0$ and $\eta>0$ and assume that $n\geq \c_0(r_{\operatorname{eff}}^{\bm \Sigma}(\eta)+\log(1/\delta))$ for some universal constant $\c_0>0$. With probability at least $1-2\delta$, for any $\bm w\in \mc B$:
    \begin{align*}
        (P-P_n)\gw\leq\frac{1}{4}P\gw+\c_1'\left(\frac{r_{\operatorname{eff}}^{\bm \Sigma}(\eta)+\log(\pi^2/(6\delta))+\log\log(n)+\c_2'}{n\alpha}+\alpha\eta R^2\right),
    \end{align*}
    where $\c_1'$ and $\c_2'$ have the following form:
    \begin{align*}
        \c_1' &= \frac{21251}{12},\\
        \c_2'&=1+\log\left(1+(\log_2((\log2+R\sqrt{\norm{\bm\Gamma}}) \alpha))_+ + \frac{1}{\log 2}\right).
    \end{align*}
\end{lemma}

\begin{proof}
    Given the assumption on the sample size $n$ and Lemma \ref{lemma: E cond}, we will work on the event $E_{\operatorname{cov}}$, which holds with probability at least $1-\delta$. Thus, by a Union Bound, we need show that the upper-bound on $(P-P_n)\gw$ from the Lemma holds with probability $1-\delta$.

    \textbf{Step 1: Bounding $P\gw$ for $\bm w\in\mc B$:} First of all, note that if $\bm w\in\mc B$, i.e. $\norm{\bm w}\leq R$, then
    \begin{align*}
        \norm{\bm w}_{\Gamma}\leq \sqrt{\norm{\bm w}\norm{\bm \Gamma}}\leq R\sqrt{\norm{\bm \Gamma}}:=R_{\bm \Gamma}.
    \end{align*}
    
    Let $M_{\bm \Gamma}=\log2+R_{\bm \Gamma}$ and note that $M_{\bm \Gamma}\geq \underset{\bm w\in\mc B}{\sup}P\gw=\underset{\bm w\in\mc B}{\sup}(\mc L(\bm w)-\mc L(\wpop))$. Indeed, using $0\leq \ell(z)\leq |z|+\log 2$ and Jensen's inequality,
    \begin{align*}
        \underset{\bm w\in\mc B}{\sup}(\mc L(\bm w)-\mc L(\wpop))&\leq\underset{\bm w\in\mc B}{\sup} \mc L(\bm w)=\underset{\bm w\in\mc B}{\sup} \mathbb{E}[\ell(Y\bm X^\top \bm w)]\leq \log 2+ \underset{\bm w\in\mc B}{\sup}\mathbb{E}[|Y\bm X^\top \bm w|]\\
        &\leq \log 2\!+\!\underset{\bm w\in\mc B}{\sup}\sqrt{\mathbb{E}[(\bm X^\top \bm w)^2]}
        =\log2\!+\!\underset{\bm w\in\mc B}{\sup}\norm{\bm w}_{\bm \Gamma}\leq \log2+R_{\bm \Gamma}=M_{\bm \Gamma}.
    \end{align*}

    \textbf{Step 2: Peeling:} Now, for $k\geq 0$, set $r_k=2^{-k}M_{\bm \Gamma}$ and $\delta_k=6\delta/(\pi^2(k+1)^2)$, so that $\sum_{k\geq 0}\delta_k=\delta$.
    Note that
    \begin{align*}
        \log(1/\delta_k)=2\log(k+1)+\log(\pi^2/(6\delta)).
    \end{align*}
    Applying Corollary \ref{cor: uniform shell bound} for all $r_k$ and $\delta_k$ (with the same $\eta$), we have, by a Union Bound, that with probability $1-\delta$, for any $k\geq 0$,
    \begin{align}
    \label{partial shell bound}
        \underset{\gw\in \mc{G}(r_k)}{\sup}(P-P_n)\gw\leq \frac{128 (r_{\operatorname{eff}}^{\bm \Sigma}(\eta)+1)}{\alpha n}+\frac{1}{8}r_k+\frac{3}{16}\alpha\eta R^2+600\frac{2\log(k+1)+\log(\pi^2/(6\delta))}{\alpha n}.
    \end{align}

    \textbf{Step 3: Bounding $k$:} %
    We now claim there exists $\c_3'$, depending on $R_{\bm \Gamma}$ (through $M_{\bm \Gamma})$  and $\alpha$ only, such that
    \begin{align}
    \label{ineq: shell log(k+1)}
        \text{if }r_k=2^{-k}M_{\bm \Gamma}\geq\frac{1}{n\alpha},\quad\text{then }\log(k+1)\leq \c_3'+\log\log(n).
    \end{align}
    Indeed, in this case, $k\leq \log_2(n\alpha M_{\bm \Gamma})$, so that
    \begin{align*}
        k+1\leq 1+(\log_2(M_{\bm \Gamma} \alpha))_+ +\log_2(n).
    \end{align*}
    But, as $\log_2(n)=\log(n)/\log(2)$ and, assuming for simplicity $n\geq 3>e$ (as otherwise we can replace $n$ by $n+e$ everywhere), $\log n\geq 1$,
    \begin{align*}
        k+1\leq \left(1+(\log_2(M_{\bm \Gamma} \alpha))_+ + \frac{1}{\log 2}\right)\log(n).
    \end{align*}
    By taking $\log$ again, we arrive at (\ref{ineq: shell log(k+1)}) with 
    \begin{align*}
        \c_3'=\log\left(1+(\log_2(M_{\bm \Gamma} \alpha))_+ + \frac{1}{\log 2}\right)
    \end{align*}

\bigskip

    For the remaining of the proof, we fix $\bm w\in\mc B$, define $\c_2'=\c_3'+1$, and introduce the notation
    \begin{align*}
        \Xi(\eta):=\c_1'\left(\frac{r_{\operatorname{eff}}^{\bm \Sigma}(\eta)+\log(\pi^2/(6\delta))+\log\log(n)+\c_2'}{n\alpha}+\alpha\eta R^2\right).
    \end{align*}
    Note that, up to constants, this is the same as $\Psi(\eta)$ in the hypothesis of Theorem \ref{thm: upper bound final}.
    
    \textbf{Step 4: The case $P\gw>\Xi(\eta)$:} Since $P\gw\leq M_{\bm \Gamma}=r_0$, there exists some $k\geq 0$ such that $r_{k+1}<P\gw\leq r_k$. In this case $\gw\in G(r_k)$, so that (\ref{partial shell bound}) applies. Additionally, note that
    \begin{align}
    \label{eq: r_k n mu}
        r_k\geq P\gw>\Xi(\eta)\geq \c_1'\frac{\c_2'}{n\alpha}\geq \frac{1}{n\alpha},
    \end{align}
    as $\c_1',\c_2'>1$. So, (\ref{ineq: shell log(k+1)}) also holds in this case. And, also using $r_k\leq 2P\gw$, we get:
    \begin{align*}
        (P-P_n)\gw &\leq \frac{128 (r_{\operatorname{eff}}^{\bm \Sigma}(\eta)\!+\!1)}{\alpha n}+\frac{1}{4}P\gw+\frac{3}{16}\alpha\eta R^2+600\frac{2\c_3'+2\log(\log(n))+\log(\pi^2/(6\delta))}{\alpha n}\\
        &\leq \frac{1}{4}P\gw+1200\left(\frac{r_{\operatorname{eff}}^{\bm \Sigma}(\eta)+1+\log(\pi^2/(6\delta))+\c_3'+\log\log(n)}{n\alpha}+\alpha\eta R^2\right)\\
        &\leq\frac{1}{4}P\gw+\c_1'\left(\frac{r_{\operatorname{eff}}^{\bm \Sigma}(\eta)+\log(\pi^2/(6\delta))+\log\log(n)+\c_2'}{n\alpha}+\alpha\eta R^2\right),
    \end{align*}
    where the last inequality follows from $\c_2'=\c_3'+1$ and $\c_1'\geq 1200$.
    
    \textbf{Step 5: The case $P\gw\leq \Xi(\eta)$:} In this case, let $k$ be the smallest index with $r_k\leq 2\Xi(\eta)$. If $k=0$, then $\log(k+1)\leq \c_3'+\log\log(n)$ clearly holds. Otherwise, by minimality, $r_{k-1}>2\Xi(\eta)$, and then $r_k=r_{k-1}/2>\Xi(\eta)$, so that again, by  \eqref{ineq: shell log(k+1)} and \eqref{eq: r_k n mu}, $\log(k+1)\leq \c_3'+\log\log(n)$.

    Moreover, since $P\gw\leq\Xi(\eta)<r_k$, we have $\gw\in G(r_k)$ and so (\ref{partial shell bound}) applies. Using $r_k\leq 2\Xi(\eta)$:
    \begin{align*}
        (P-P_n)\gw\leq \frac{128 (r_{\operatorname{eff}}^{\bm\Sigma}(\eta)+1)}{\alpha n}+\frac{1}{4}\Xi(\eta)+\frac{3}{16}\alpha\eta R^2+600\frac{2\c_3'+2\log\log(n)+\log(\pi^2/(6\delta))}{\alpha n}.
    \end{align*}
    By the definition of $\Xi(\eta)$ and the fact that $\c_2'=\c_3'+1>1$, we have the following:
    \begin{align*}
        \frac{128(r_{\operatorname{eff}}^{\bm \Sigma}(\eta)+1)}{\alpha n}&\leq \frac{128}{\c_1'}\Xi(\eta),\\
        \frac{3}{16}\alpha\eta R^2&\leq \frac{3}{16\c_1'}\Xi(\eta),\\
        600\frac{2\c_3'+2\log\log(n)+\log(\pi^2/(6\delta))}{\alpha n}&\leq \frac{1200}{\c_1'}\Xi(\eta),
    \end{align*}
    and thus, using the value of $\c_1'$, we have
    \begin{align*}
        (P-P_n)\gw&\leq \left(\frac{1}{4}+\frac{128}{\c_1'}+\frac{3}{16\c_1'}+\frac{1200}{\c_1'}\right)\Xi(\eta)=\left(\frac{1}{4}+\frac{21251}{16\c_1'}\right)\Xi(\eta)\\
        &= \left(\frac{1}{4}+\frac{3}{4}\right)\Xi(\eta)=\Xi(\eta)\leq \frac{1}{4}P\gw+\Xi(\eta)\\
        &= \frac{1}{4}P\gw+\c_1'\left(\frac{r_{\operatorname{eff}}^{\bm \Sigma}(\eta)+\log(\pi^2/(6\delta))+\log\log(n)+\c_2'}{n\alpha}+\alpha\eta R^2\right)\,.
    \end{align*}
    This concludes the proof.
\end{proof}

We are now ready to put everything together to prove the excess zero-one risk upper bound in Theorem~\ref{thm: upper bound final}.

\subsection{Proof of Theorem \ref{thm: upper bound final}}\label{sec:theorem_1_proof}

\upperbound*

\begin{proof}
    Fix $\eta_n>0$. Recall that $R=4\norm{\wb}+1$. Note that by Lemma \ref{ESproperties} (i) we have $\bm w_{\tau}\in \mc B$.  By Lemma \ref{lemma: w_pop}, recall that $\wpop\in \mc B$. 
    By Lemma \ref{ESproperties} (ii), $P_n g_{\bm w_\tau}\leq 0$, so we observe that
    \begin{align*}
        Pg_{\bm w_\tau}=(P-P_n)g_{\bm w_\tau}+ P_ng_{\bm w_\tau}\leq (P-P_n)g_{\bm w_\tau}
    \end{align*}
    Now, for $\c_1'$ and $\c_2'$ as in Lemma \ref{lemma: final concentration bound}, consider
    \begin{align*}
        \Xi(\eta_n)=\frac{\c_1'}{\alpha}\left(\frac{r_{\operatorname{eff}}^{\bm \Sigma}(\eta_n)+\log(\pi^2/(3\delta))+\log\log(n)+\c_2'}{n}+\alpha^2 R^2\eta_n\right).
    \end{align*}
    
    Then, by Lemma \ref{lemma: final concentration bound} (applied with $\delta/2$ instead of $\delta$), with probability at least $1-\delta$,
    \begin{align*}
        Pg_{\bm w_\tau}\leq (P-P_n)g_{\bm w_\tau}\leq \frac{1}{4}Pg_{\bm w_\tau}+\Xi(\eta_n),
    \end{align*}
    which can be rearranged into
    \begin{align}
    \label{eq: excess logistic psi}
        \mc L(\bm w_{\tau})-\mc L(\wpop)=Pg_{\bm w_\tau}\leq \frac{4}{3}\Xi(\eta_n).
    \end{align}
    Assume that the following inequality holds, which will coincide with the assumption \eqref{thm condition} for appropriate $\c_2$, $\c_3$, and $\c_4$, stated at the end of the proof: 
    \begin{align*}
        \Xi(\eta_n)\leq \frac{3}{32}\alpha(\rho^*)^2\norms{\wb}^2\ .
    \end{align*}
    Then, by (\ref{eq: excess logistic psi}) we have
    \begin{align*}
        \mc L(\bm w_{\tau})-\mc L(\wpop)\leq  \frac{\alpha}{8} (\rho^*)^2\norms{\wb}^2,
    \end{align*}
    so that we can apply Lemma \ref{cor:calibration} and (\ref{eq: excess logistic psi}) to arrive at the conclusion that, with probability $1-\delta$,
    \begin{align*}
        \ee (\bm w_{\tau})-\ee(\bw^*)\leq \frac{8(1-2p)}{(\rho^*)^2\alpha\sqrt{2\pi}\norms{\wb}}(\mc L(\bm w_\tau)-\mc L(\wpop))\leq \frac{32(1-2p)}{3(\rho^*)^2\alpha\sqrt{2\pi}\norms{\wb}}\Xi(\eta_n).
    \end{align*}
    Unfolding the definition of $\Xi(\eta_n)$ and absorbing additional constants into $\c_1'$ and $\c_2'$, we have
    \begin{align*}
        \ee (\bm w_{\tau})-\ee(\bw^*) \leq \c_1\frac{1-2p}{(\rho^*)^2\norms{\wb}}\left(\frac{r_{\operatorname{eff}}^{\bm \Sigma}(\eta_n)+\log(1/\delta)+\log\log(n)+\c_2}{n}+\c_3\eta_n\right).
    \end{align*}
    The last factor corresponds exactly to $\Psi(\eta_n)$, since $r_{\operatorname{eff}}(\eta_n)=r_{\operatorname{eff}}^{\bSigma}(\eta_n)$.
     The missing definitions of the constants in the Theorem statement are given by
    \begin{align*}
        \bm \Gamma &= \bmu\bmu^\top + \bm \Sigma,\\
        R&=4\norm{\wb}+1=4\norm{\bm \Sigma^{-1}\bmu}+1,\\
        \alpha &= \frac{1}{2}\ell''(3R\sqrt{\norm{\bm \Gamma}}),\\
        \c_1 &= \frac{170008}{9\alpha^2\sqrt{2\pi}},\\
        \c_2&=1+\log(\pi^2/3)+\log\left(1+(\log_2((\log2+R\sqrt{\norm{\bm\Gamma}}) \alpha))_+ + \frac{1}{\log 2}\right),\\
        \c_3&=\alpha^2 R^2,\\
        \c_4&=\frac{9\alpha^2}{170008}.
    \end{align*}

\end{proof}

\section{Proof of the Statistical Lower Bound}
\label{app: stat lower bound proof}
Recall that for the purposes of the statistical lower bound we keep the noise covariance $\bSigma$ and the noise-flip probability $p$ fixed, so that the class of Gaussian Mixtures introduced in (\ref{eq:GMM}) is effectively parametrised by $\bmu$ only. To ease the notation, we will  denote the Bayes linear classifier by
\begin{align*}
    \wb_{\bmu}=\bm \Sigma^{-1}\bmu,
\end{align*}
so that, with a slight abuse of notation we  write $h^*_{\bmu}(\bm x)=\operatorname{sign}(\bm x^\top\wb_{\bmu})$ and refer to $h:\mathbb{R}^{d}\rightarrow\{\pm 1\}$ as a classifier, rather than a real-valued score function.

With a slight abuse of notation, we will denote by $\mathbb{P}_{\bmu}$ and $\mathbb{E}_{\bmu}$ the probability and expectation, respectively, when the underlying randomness comes from the marginal of $\bm X$ only.

Unlike in the proof Theorem \ref{thm: upper bound final}, all spectral quantities will involve the spectrum of $\bSigma$. In particular, the geometry of $\operatorname{cov}(\bm X)=\bGamma=\bSigma+\bmu\bmu^\top$ does not enter this statistical lower bound proof.

\subsection{A Triangle Inequality for  Excess Zero-One Risk}

\begin{lemma}
\label{conversion e-> d}
For any $\bmu\in\mathbb{R}^d$ with $\snr=1$ and any $h:\mathbb{R}^d\rightarrow\{\pm1\}$,
\begin{align*}
    \ee_{\bmu}(h)-\ee_{\bmu}(\wb_{\bmu})\geq \frac{(1-2p)\sqrt{2\pi}}{16}\left(\mathbb{P}_{\bmu}(h(\bm X)\neq h^*_{\bmu}(\bm X))\right)^2.
\end{align*}
\end{lemma}

\begin{proof}
    Similarly to (\ref{eq:excess_zero_one}), we have
    \begin{align*}
        \ee_{\bmu}(h)-\ee_{\bmu}(\wb_{\bmu})=(1-2p)\mathbb{E}_{\bmu}\left[|\tanh(\bm X^\top\wb_{\bmu})|\mathbbm{1}\{h(\bm X)\neq h^*_{\bmu}(\bm X)\}\right].
    \end{align*}
    
    Consider the thresholding
    \begin{align*}
        t:=\frac{\sqrt{2\pi}}{4}\mathbb{P}_{\bmu}(h(\bm X)\neq h^*_{\bmu}(\bm X))\leq 1.
    \end{align*}
    We proceed as follows:
    \begin{align*}
        \ee_{\bmu}(h)-\ee_{\bmu}(\wb_{\bmu})&=(1-2p)\mathbb{E}_{\bmu}\left[|\tanh(\bm X^\top\wb_{\bmu})|\mathbbm{1}\{h(\bm X)\neq h^*_{\bmu}(\bm X)\}\right]\\
        &\geq(1-2p)\mathbb{E}_{\bmu}\left[|\tanh(\bm X^\top\wb_{\bmu})|\mathbbm{1}\{h(\bm X)\neq h^*_{\bmu}(\bm X)\}\mathbbm{1}\{|\bm X^\top \wb_{\bmu}|>t\}\right]\\
        &\geq (1-2p)\mathbb{E}_{\bmu}\left[\tanh(t)\mathbbm{1}\{h(\bm X)\neq h^*_{\bmu}(\bm X),|\bm X^\top \wb_{\bmu}|>t\}\right]\\
        &=(1-2p)\tanh(t)\mathbb{P}_{\bmu}(h(\bm X)\neq h^*_{\bmu}(\bm X),|\bm X^\top \wb_{\bmu}|>t),
    \end{align*}
    where the second inequality follows from the fact that $|\tanh(t)|$ is even and increasing on $[0,\infty)$. Further, by a Union Bound, we get
    \begin{align}
    \label{eq: excess zero one lower bound partial}
        \ee_{\bmu}(h)-\ee_{\bmu}(\wb_{\bmu})\geq (1-2p)\tanh(t)\left(\mathbb{P}_{\bmu}(h(\bm X)\neq h^*_{\bmu}(\bm X))-\mathbb{P}_{\bmu}(|\bm X^\top \wb_{\bmu}|\leq t)\right).
    \end{align}
    Now, note that $\bm X^\top \wb_{\bmu}\mid \wt Y\sim \mathcal{N}(\wt Y\snr^2,\snr^2)\equiv \mathcal{N}(\wt Y,1)$, so that the density of $\bm X^\top \wb_{\bmu}$ is bounded above by $1/\sqrt{2\pi}$ and thus, by the definition of $t$,
    \begin{align*}
        \mathbb{P}_{\bmu}(|\bm X^\top \wb_{\bmu}|\leq t)\leq \frac{2t}{\sqrt{2\pi}}=\frac{1}{2}\mathbb{P}_{\bmu}(h(\bm X)\neq h^*_{\bmu}(\bm X)).
    \end{align*}
    Plugging this back into (\ref{eq: excess zero one lower bound partial}) and using the fact that $\tanh(t)\geq t/2$ (because $0\leq t\leq 1$) and the definition of $t$,
    \begin{align*}
        \ee_{\bmu}(h)-\ee_{\bmu}(\wb_{\bmu})\geq \frac{(1-2p)t}{4}\mathbb{P}_{\bmu}(h(\bm X)\neq h^*_{\bmu}(\bm X))=\frac{(1-2p)\sqrt{2\pi}}{16}\left(\mathbb{P}_{\bmu}(h(\bm X)\neq h^*_{\bmu}(\bm X))\right)^2.
    \end{align*}
\end{proof}

\bigskip

\begin{lemma}%
\label{lemma: triangle d}
    For any $\bmu_1,\bmu_2\in\mathbb{R}^d$ with $\normso{\bmu_1}=\normso{\bmu_2}=1$ and any classifier $h:\mathbb{R}^d\rightarrow\{\pm 1\}$,
    \begin{align*}
        \mathbb{P}_{\bmu_1}(h(\bm X)\neq h^*_{\bmu_1}(\bm X))+\mathbb{P}_{\bmu_2}(h(\bm X)\neq h^*_{\bmu_2}(\bm X))\geq \frac{1}{\pi\sqrt{e}}\normso{\bmu_1-\bmu_2}.
    \end{align*}
\end{lemma}

\begin{proof}
    \textbf{Step 1:} First, we will prove that we can do a change of measure to the $\bmu_1,\bmu_2$-independent $Q\equiv \mc N(\bm 0_d,\bSigma)$ in the sense that for $\bmu\in\{\bmu_1,\bmu_2\}$ (with $\snr=1$),
    \begin{align}
        \label{conversion to universal Q}
        \mathbb{P}_{\bmu}(h(\bm X)\neq h^*_{\bmu}(\bm X))\geq e^{-1/2} \underset{\bm X\sim Q}{\mathbb{P}}(h(\bm X)\neq h^*_{\bmu}(\bm X)).
    \end{align}
    
    To do so, note that The Cameron-Martin Theorem, together with $\normso{\bmu_1}=1$, give that the Radon-Nikodym derivative of $\mc N(\bmu,\bSigma)$ with respect to $Q$ is
    \begin{align*}
        \frac{\text{d}\mc N(\bmu,\bSigma)}{\text{d}Q}(\bm x)=\exp\left(\bm x^\top \bSigma^{-1}\bmu-\frac{1}{2}\snr^2\right)=\exp\left(\bm x^\top \wb_{\bmu}-\frac{1}{2}\right).
    \end{align*}
    Thus, since $\bm X\mid \wt Y\sim \mc N(\wt Y\bmu,\bSigma)$, taking expectation over $\widetilde{Y}=\pm 1$, we have
    \begin{align*}
        \frac{\text{d}\mathbb{P}_{\bmu}}{\text{d}Q}(\bm x )=\frac{1}{2}\exp(-1/2)\left(\exp(\bm x^\top \wb_{\bmu})+\exp(-\bm x^\top \wb_{\bmu})\right)\geq \exp(-1/2)
    \end{align*}
    which proves (\ref{conversion to universal Q}).

    \textbf{Step 2:} Now, using (\ref{conversion to universal Q}) for $\bmu=\bmu_1,\bmu_2$ and a Union Bound, we can lower bound the left-hand of the inequality in the conclusion of the Lemma as
    \begin{align*}
        \text{LHS}&\geq\frac{1}{\sqrt{e}}\Big( \underset{\bm X\sim Q}{\mathbb{P}}\big(h(\bm X)\neq h^*_{\bmu_1}(\bm X)\big)+\underset{\bm X\sim Q}{\mathbb{P}}\big(h(\bm X)\neq h^*_{\bmu_2}(\bm X)\big)\Big)\\
        &\geq \frac{1}{\sqrt{e}}\underset{\bm X\sim Q}{\mathbb{P}}\big(h^*_{\bmu_1}(\bm X)\neq h^*_{\bmu_2}(\bm X)\big)\\
        &=\frac{1}{\sqrt{e}}\underset{\bm X\sim Q}{\mathbb{P}}\big(\operatorname{sign}(\bm X^\top \wb_{\bmu_1})\neq \operatorname{sign}(\bm X^\top \wb_{\bmu_2})\big).
    \end{align*}

    \textbf{Step 3:} Thus, it remains to show that
    \begin{align}
    \label{disagreement -> distance}
        \underset{\bm X\sim Q}{\mathbb{P}}(\operatorname{sign}(\bm X^\top \wb_{\bmu_1})\neq \operatorname{sign}(\bm X^\top \wb_{\bmu_2}))\geq \frac{1}{\pi}\normso{\bmu_1-\bmu_2}.
    \end{align}

    To do so, we will work with the whitened $\bm X\sim \mc N(\bm 0_d,\bm \Sigma)$. Recall the notation $\bm X=\bSigma^{1/2}\bm Z$ so that $\bm Z\sim \mc N(\bm 0_d,\bm I_d)$, and additionally consider $\bm v_1^*:=\bSigma^{1/2}\wb_{\bmu_1}=\bSigma^{-1/2}\bmu_1$, and $\bm v_2^*:=\bSigma^{1/2}\wb_{\bmu_2}=\bSigma^{-1/2}\bmu_2$ with $\norm{\bm v_1^*}=\norm{\bm v_2^*}=1$. So, (\ref{disagreement -> distance}) reduces to
    \begin{align}
    \label{whitened disagreement -> distance}
        \mathbb{P}(\operatorname{sign}(\bm Z^\top \bm v_1^*)\neq \operatorname{sign}(\bm Z^\top \bm v_2^*))\geq \frac{\norm{\bm v_1^*-\bm v_2^*}}{\pi},
    \end{align}
    where the probability is taken over $\bm Z\sim \mc N(\bm 0_d,\bm I_d)$ and we suppress this for convenience. To prove (\ref{whitened disagreement -> distance}) , note that
    \begin{align*}
        \mathbb{P}(\operatorname{sign}(\bm Z^\top \bm v_1^*)\neq \operatorname{sign}(\bm Z^\top \bm v_2^*))=1-2\mathbb{P}(\bm Z^\top \bm v_1^*>0,\bm Z^\top \bm v_2^*>0).
    \end{align*}
    Let $\theta\in [0,\pi]$ be the angle between $\bm v_1^*$ and $\bm v_2^*$, i.e. $\bm (v_1^*)^\top \bm v_2^*=\cos\theta$. Then, the scalar random variables $\bm Z^\top \bm v_1^*$ and $\bm Z^\top \bm v_2^*$ are both mean $0$, variance 1, and have correlation $\cos\theta$, so we can write
    \begin{align*}
        \bm Z^\top \bm v_2^*=\bm Z^\top \bm v_1^*\cos\theta+|\sin\theta| N,
    \end{align*}
    where $N\sim \mc N(0,1)$ is independent of $\bm Z^\top\bv_1^*$. Then, recalling that $\phi$ and $\Phi$ are the p.d.f. and c.d.f, respectively, of $\mc N(0,1)$, we have 
    \begin{align*}
        \mathbb{P}(\bm Z^\top \bm v_1^*>0,\bm Z^\top \bm v_2^*>0)=\int_{0}^{\infty}\phi\left(t\right)\mathbb{P}\left(N>-\frac{t\cos\theta}{|\sin\theta|}\right)\textup{d}t=\int_{0}^{\infty} \phi(t)\Phi\left(\frac{t\cos\theta}{|\sin\theta|}\right)\textup{d}t.
    \end{align*}
    Note that this also is valid at $\theta\in \{0,\pi\}$: if $\theta=0$, then $\bm Z^\top \bm v_1^*=\bm Z^\top \bm v_2^*$ and the previous equality reads
    \begin{align*}
        \mathbb{P}(\bm Z^\top \bm v_1^*>0)=\int_{0}^\infty \phi(t)\text{d}t,
    \end{align*}
    whereas for $\theta=\pi$, $\bm Z^\top \bm v_1^*=-\bm Z^\top \bm v_2^*$ and
    \begin{align*}
        \mathbb{P}(\bm Z^\top \bm v_1^*>0,\bm Z^\top \bm v_2^*>0)=0=\int_{0}^\infty \phi(t)\Phi(-\infty)\text{d}t.
    \end{align*}
    So, using the identity
    \begin{align*}
        \int_{0}^\infty\phi(t)\Phi(at)\textup{d}t=\frac{1}{4}+\frac{\arctan a}{2\pi},
    \end{align*}
    valid for any $a\in\mathbb{R}$, we get
    \begin{align*}
        \mathbb{P}(\operatorname{sign}(\bm Z^\top \bm v_1^*)\neq \operatorname{sign}(\bm Z^\top \bm v_2^*))=1-2\left(\frac{1}{4}+\frac{\arctan(\cos\theta/|\sin\theta|)}{2\pi}\right)=\frac{\theta}{\pi}.%
    \end{align*}
    We arrive at (\ref{whitened disagreement -> distance}) by using $\theta\geq 2\sin(\theta/2)$ (recall that $\theta\geq 0$) and $\norm{\bm v_1^*}=\norm{\bm v_2^*}=1$:
    \begin{align*}
        \frac{\theta}{\pi}\geq\frac{2\sin(\theta/2)}{\pi}=\frac{\sqrt{2-2\cos\theta}}{\pi}=\frac{\sqrt{2-2(\bm v_1^*)^\top \bm v_2^*}}{\pi}=\frac{\norm{\bm v_1^*-\bm v_2^*}}{\pi}.
    \end{align*}

\end{proof}

\bigskip

Putting together Lemmas \ref{conversion e-> d} and \ref{lemma: triangle d}, we arrive at:

\begin{lemma}[``Triangle inequality'' for zero-one excess risk]
\label{lemma: triangle e}
    For any $\bmu_1,\bmu_2\in\mathbb{R}^d$ with $\normso{\bmu_1}=\normso{\bmu_2}=1$ and any classifier $h:\mathbb{R}^d\rightarrow\{-1,1\}$, we have
    \begin{align*}
        \left(\ee_{\bmu_1}(h)-\ee_{\bmu_1}(\wb_{\bmu_1})\right)+\left(\ee_{\bmu_2}(h)-\ee_{\bmu_2}(\wb_{\bmu_2})\right)\geq \c_1 (1-2p)\normso{\bmu_1-\bmu_2}^2,
    \end{align*}
    where
    \begin{align*}
        \c_1=\frac{\sqrt{2}}{32\pi^{3/2}e}\,.
    \end{align*}
\end{lemma}

\begin{proof}
    Using Lemma \ref{conversion e-> d}, the inequality $a^2+b^2\geq (a+b)^2/2$, and Lemma \ref{lemma: triangle d}, we have that the left hand side $\left(\ee_{\bmu_1}(h)-\ee_{\bmu_1}(\wb_{\bmu_1})\right)+\left(\ee_{\bmu_2}(h)-\ee_{\bmu_2}(\wb_{\bmu_2})\right)$ is at least
    \begin{align*}
        &\geq \frac{(1-2p)\sqrt{2\pi}}{16}\left(\left(\mathbb{P}_{\bmu_1}(h(\bm X)\neq h^*_{\bmu_1}(\bm X))\right)^2+\left(\mathbb{P}_{\bmu_2}(h(\bm X)\neq h^*_{\bmu_2}(\bm X))\right)^2\right)\\
        &\geq \frac{(1-2p)\sqrt{2\pi}}{32}\left(\mathbb{P}_{\bmu_1}(h(\bm X)\neq h^*_{\bmu_1}(\bm X))+\mathbb{P}_{\bmu_2}(h(\bm X)\neq h^*_{\bmu_2}(\bm X))\right)^2\\
        &\geq \frac{(1-2p)\sqrt{2}}{32\pi^{3/2}e}. \normso{\bmu_1-\bmu_2}^2\ .
    \end{align*}
\end{proof}

\subsection{Information Theory Tools}

\begin{lemma}
\label{lemma: kl}
    For any $\bmu_1,\bmu_2\in\mathbb{R}^d$,
    \begin{align*}
        \operatorname{KL}(P_{\bmu_1}\parallel P_{\bmu_2})\leq \frac{1}{2}\normso{\bmu_1-\bmu_2}^2\ .
    \end{align*}
\end{lemma}

\begin{proof}
    By standard KL-facts, for any $\wt y\in\{\pm 1\}$,
    \begin{align*}
        \operatorname{KL}(\mc N(\widetilde y\bmu_1,\bSigma)\parallel \mc N(\widetilde y\bmu_2,\bSigma))=\frac{1}{2}\norm{\widetilde{y}(\bmu_1-\bmu_2)}_{\bSigma^{-1}}^2=\frac{1}{2}\norm{\bmu_1-\bmu_2}_{\bSigma^{-1}}^2.
    \end{align*}
    Hence, denoting by $\widetilde{P}_{\bmu_1}$ and $\wt P_{\bmu_2}$ the laws of $(\bX,\widetilde{Y})$ under $\bmu_1$ and $\bmu_2$, respectively, since $\widetilde{Y}$ has the same law irrespective of $\bmu_1$ or $\bmu_2$:
    \begin{align*}
        \operatorname{KL}(\widetilde{P}_{\bmu_1}\parallel \widetilde{P}_{\bmu_2})={\mathbb{E}}_{\widetilde{Y}}\left[\operatorname{KL}(\mc N(\widetilde{Y}\bmu_1,\bSigma)\parallel \mc N(\widetilde{Y} \bmu_2,\bSigma))\right]=\frac{1}{2}\norm{\bmu_1-\bmu_2}_{\bSigma^{-1}}^2.
    \end{align*}
    Finally, when including the random classification noise $\bm\varepsilon$, the data processing inequality gives
    \begin{align*}
        \operatorname{KL}({P}_{\bmu_1}\parallel {P}_{\bmu_2})\leq\operatorname{KL}(\widetilde{P}_{\bmu_1}\parallel \widetilde{P}_{\bmu_2})\leq \frac{1}{2}\norm{\bmu_1-\bmu_2}_{\Sigma^{-1}}^2\ .
    \end{align*}
\end{proof}

\bigskip

Now, recall that $\bSigma$ is a symmetric positive semidefinite matrix and let $(\bm u_i,\lambda_i)_{i=1}^d$ be its eigendecomposition, with $\lambda_1\geq\lambda_2\ldots\geq \lambda_d>0$ and $\norm{\bm u_i}=1$. We will show that the lower bound holds with respect to $\bmu\in \mc U\subseteq\mathbb{R}^d$ defined as follows:
\begin{align*}
    \mc U=\left\{\bmu:\snr=1,\norm{\bSigma^{-1}\bmu}\leq {\sqrt{\frac{2}{\lambda_1}}}\right\}.%
\end{align*}
Additionally, recall that for ${\wt\eta} \in (0,\lambda_2]$ we consider the eigenvalue counting function
\begin{align*}
    N({\wt\eta})=|\{i:1\leq i\leq d, \lambda_i\geq {\wt\eta}\}|,
\end{align*}
and note that $N({\wt\eta})\geq 2$.

We are now ready to present the packing that will be used throughout our statistical lower bound.

\begin{lemma}[Dense packing of $\mc U$]
\label{lemma: packing}
    There exists a universal constant $\frac{1}{2}<\c_2<1$ such that the following holds: 
    
    Fix ${\wt\eta}\in (0,\lambda_2/2]$ and assume $N(2{\wt\eta})\geq 9$. Consider any $\kappa>0$ with $\kappa^2\leq 2{\wt\eta}/\lambda_1$. Then, there exist a subset $\mc{U}_\kappa\subset \mc{U}$ with $\log|\mc{U}_{\kappa}|\geq \c_2(N(2{\wt\eta})-1)$ and for any $\bmu_1,\bmu_2\in \mc{U}_{\kappa}$ with $\bmu_1\neq\bmu_2$,
    \begin{align*}
        \frac{\kappa}{4}\leq\norm{\bmu_1-\bmu_2}_{\bSigma^{-1}}\leq 2\kappa.
    \end{align*}
\end{lemma}

\begin{proof}
    We start with the Gilbert-Varshamov bound, which states that there exists a constant $1>\c_2>1/2$ (in fact, it can be taken as $\c_2=(1-H_2(1/64))\log2\approx 0.61$) such that, if $N(2{\wt\eta})\geq 9$, then there is a packing subset $\mathcal{P}\subset\{\pm 1\}^{N(2{\wt\eta})-1}$ with $\log|\mathcal P|\geq \c_2(N(2{\wt\eta})-1)$ and $d_H(\bm a,\bm a')\geq (N(2{\wt\eta})-1)/64$ for any $\bm a,\bm a'\in \mathcal{P}$, where $d_H$ denotes the Hamming distance. Define the set of $N(2{\wt\eta})-1$ dimensional set of unit vectors
    \begin{align*}
        \mathcal{V}=\left\{\frac{1}{\sqrt{N(2{\wt\eta})-1}}\bm a:\bm a\in\mathcal{P}\right\}.
    \end{align*}
    Then $\log|\mc V|\geq \c_2(N(2{\wt\eta})-1)$ and for any $\bm v=\bm a/\sqrt{(N(2{\wt\eta})-1)}, \bm v'=\bm a'/\sqrt{(N(2{\wt\eta})-1)}\in \mc V$ (with $\bm a,\bm a'\in \mathcal{P}$),
    \begin{align}
    \label{norm v v'}
        \norm{\bm v-\bm v'}^2=\frac{4}{N(2{\wt\eta})-1}d_{H}(\bm a,\bm a')\geq\frac{4}{N(2{\wt\eta})-1}\cdot\frac{N(2{\wt\eta})-1}{64}=\frac{1}{16},\quad \text{i.e. } \norm{\bm v-\bm v'}\geq\frac{1}{4}.
    \end{align}
    
    We next convert elements of $\mc V\subset \mathbb{R}^{N(2{\wt\eta})-1}$ to elements of $\mathcal{U}$. Each $\bm v\in\mathcal{V}$ has $N(2{\wt\eta})-1$ coordinates, and we denote its $i$'th coordinate by $v_i$. Since ${\wt\eta}\leq \lambda_2/2$, we have
    \begin{align*}
        N(2{\wt\eta})-1=|\{i:1\leq i\leq d,\lambda_i\geq 2{\wt\eta}\}|-1=|\{i:2\leq i\leq d,\lambda_i\geq 2{\wt\eta}\}|\in\{2,\ldots,d\}.
    \end{align*}
    and, recalling the $d$ eigenvectors $\bm u_i$ of $\bSigma$, we can define
    \begin{align*}
        \bmu_{\bm v}=\sqrt{\lambda_1({1-\kappa^2})}\bm u_1+\kappa \sum_{i=2}^{N(2{\wt\eta})}{v_{i-1}}{\sqrt{\lambda_i}}\bm u_i\in\mathbb{R}^d.
    \end{align*}
    We consider all these vectors to define
    \begin{align*}
        \mc U_{\kappa}=\{\bmu_{\bm v}:\bm v\in\mc V\}.
    \end{align*}
    
    Firstly, this satisfies
    \begin{align*}
        \log|\mc U_{\kappa}|=\log|\mc V|\geq \c_2(N(2{\wt\eta})-1),
    \end{align*}
    and we next show that $\mc U_{\kappa}\subseteq \mc U$. 
    
    Indeed, for any such $\bmu_{\bm v}\in\mc U_{\kappa}$ we have %
    \begin{align*}
        \normso{\bmu_{\bm v}}^2=(1-\kappa^2)+\kappa^2\norm{\bm v}^2=1,
    \end{align*}
    and, using the fact that $\lambda_i\geq 2{\wt\eta}$ for $i\in\{2,\ldots,N(2{\wt\eta})$\} and the fact that $\kappa^2\leq 2{\wt\eta}/\lambda_1$,
    \begin{align*}
        \norm{\bSigma^{-1}\bmu_{\bm v}}^2=\frac{1-\kappa^2}{\lambda_1}+\kappa^2\sum_{i=2}^{N(2{\wt\eta})}\frac{v_{i-1}^2}{\lambda_i}\leq \frac{1}{\lambda_1}+\frac{\kappa^2}{2{\wt\eta}}\sum_{i=2}^{N(2{\wt\eta})}v_{i-1}^2=\frac{1}{\lambda_1}+\frac{\kappa^2}{2{\wt\eta}}\leq \frac{2}{\lambda_1}. 
    \end{align*}

    Finally, let $\bmu_{\bm v_1}, \bmu_{\bm v_2}\in \mc U_{\kappa}$. Note that, using (\ref{norm v v'}),
    \begin{align*}
        \normso{\bmu_{\bm v_1}-\bmu_{\bm v_2}}=\normso{\kappa\sum_{i=2}^{N(2{\wt\eta})}(v_{1,i-1}-v_{2,i-1})\sqrt{\lambda_i}\bm u_i}=\kappa \norm{\bm v_1-\bm v_2}\geq \frac{\kappa}{4},
    \end{align*}
    and, since $\bm v_1,\bm v_2$ are unit vectors, using the triangle inequality,
    \begin{align*}
        \normso{\bmu_{\bm v_1}-\bmu_{\bm v_2}}=\kappa \norm{\bm v_1-\bm v_2}\leq 2\kappa.
    \end{align*}
     This concludes the proof.
\end{proof}

The final ingredient that  we will need in the lower bound proof is:

\begin{lemma}[Fano's inequality]
    \label{lemma:fano}  Let $U$, $S$, and $\widehat{U}$ be three random variables where $U$ and $\widehat{U}$ are conditionally independent, given $S$, and $U$ and $\widehat{U}$ take values in a finite set of size $M$, and $U$ is uniform. Then:
    \begin{align*}
        \mathbb{P}(U\neq \widehat{U})\geq 1-\frac{I(U;S)+\ln 2}{\ln(M)},
    \end{align*}
    where $I(U;S)$ denotes the mutual information between $U$ and $S$.
\end{lemma}

\subsection{Proof of Theorem \ref{thm: stat lower bound}}

\lowerbound*

\begin{proof}
Fix ${\wt\eta}\in(0,\lambda_2/2]$ and recall the constant $1/2<\c_2<1$ from Lemma \ref{lemma: packing}. We will apply this Lemma with
    \begin{align}
    \label{def kappa}
        \kappa=\frac{\sqrt{\c_2}}{4}\min\left\{\sqrt{\frac{2{\wt\eta}}{\lambda_1}},\sqrt{\frac{N(2{\wt\eta})-1}{n}}\right\}.
    \end{align}
    Note that, as $\c_2< 1$, we have $\kappa^2\leq 2{\wt\eta}/\lambda_1$, so that Lemma \ref{lemma: packing} can be indeed applied. Thus, there exists a subset $\mc U_\kappa$ of $\mc U$ with $\log|\mc U_{\kappa}|\geq \c_2(N(2{\wt\eta})-1)$ and $\kappa/4\leq \normso{\bmu-\bmu'}\leq 2\kappa$ for any distinct $\bmu,\bmu'\in \mc U_{\kappa}$.

    Using Lemma \ref{lemma: triangle e}, for any $\bmu_1,\bmu_2\in\mc U_\kappa$,
    \begin{align}
    \label{eq: separable}
        \left(\ee_{\bmu_1}(h)-\ee_{\bmu_1}(\wb_{\bmu_1})\right)+\left(\ee_{\bmu_2}(h)-\ee_{\bmu_2}(\wb_{\bmu_2})\right)&\geq \c_1 (1-2p)\normso{\bmu_1-\bmu_2}^2 \nonumber\\
        &\geq \frac{\c_1(1-2p)\kappa^2}{16}\mathbbm{1}\{\bmu_1\neq\bmu_2\}.
    \end{align}

    Now, fix some $\bmu\in \mc U_{\kappa}$ and draw $S\sim P_{\bmu}^{\otimes n}$. Let $\widehat{h}$ be any procedure and $\widehat{h}_S: \bm X\in\mathbb{R}^d\mapsto Y\in\{\pm1\}$ be the induced sample-based estimator. Define the decoder
    \begin{align}
    \label{def: mu s}
        \wh\bmu_S=\underset{\bmu'\in \mc U_{\kappa}}{\arg\min}\left(\ee_{\bmu'}(\widehat{h}_S)-\ee_{\bmu'}(\wb_{\bmu'})\right),
    \end{align}
    which is a function of $S$ and an element of $\mc U_{\kappa}$. Note that, using (\ref{eq: separable}) with $\bmu_1:=\bmu$ and $\bmu_2=\wh \bmu_S$, this further implies that
    \begin{align*}
        \ee_{\bmu}(\wh h_S)-\ee_{\bmu}(\wb_{\bmu})&\geq \frac{1}{2}\left(\left(\ee_{\bmu}(\wh h_S)-\ee_{\bmu}(\wb_{\bmu})\right)+\left(\ee_{\wh\bmu_S}(\wh h_S)-\ee_{\wh\bmu_{S}}(\wb_{\wh\bmu_S})\right)\right)\\
        &\geq \frac{\c_1(1-2p)\kappa^2}{32}\mathbbm{1}\{\bmu\neq\wh\bmu_S\}.
    \end{align*}
    Thus, taking expectations over $S\sim P_{\bmu}^{\otimes n}$ and then supremum over $\bmu\in \mc U$, we arrive at
    \begin{align}
        \label{eq: pre-fano}
        \underset{\bmu\in \mc U_\kappa}{\sup}\underset{S\sim P_{\bmu}^{\otimes n}}{\mathbb{E}}\left[\ee_{\bmu}(\wh h_S)-\ee_{\bmu}(\wb_{\bmu})\right]\geq \frac{c_1(1-2p)\kappa^2}{32}\underset{\bmu\in\mc U_\kappa}{\sup}\underset{S\sim P_{\bmu}^{\otimes n}}{\mathbb{P}}(\bmu\neq \wh\bmu_S).
    \end{align}
   
    We are now ready to apply Fano's inequality (c.f. Lemma \ref{lemma:fano}). Consider a uniform random variable $\bm U$ over $\mc U_{\kappa}$ and, conditional on $\bm U=\bmu$, draw $S\sim P_{\bmu}^{\otimes n}$, and let $\wh \bmu_S$ the estimator based on the sample $S$ defined in (\ref{def: mu s}). Note that $\wh \bmu_S$ and $\bm U$ are conditionally independent given $S$. Thus,
    \begin{align*}
        \underset{\bm U, S}{\mathbb{P}}(\bm U\neq \wh \bmu_S)\geq 1-\frac{I(\bm U;S)+\log 2}{\log|\mc U_{\kappa}|}.
    \end{align*}
    By the law of total probability we have
    \begin{align*}
        \underset{\bm U,S}{\mathbb{P}}(\bm U\neq \wh \bmu_S)=\frac{1}{|\mc U_{\kappa}|}\sum_{\bmu\in \mc U_{\kappa}}\underset{S\sim P_{\bmu}^{\otimes n}}{\mathbb{P}}(\bmu \neq \wh \bmu_S),
    \end{align*}
    so that
    \begin{align}
    \label{eq: partial fano}
        \underset{\bmu\in \mc U_\kappa}{\sup}\underset{S\sim P_{\bmu}^{\otimes n}}{\mathbb{P}}(\bmu\neq \wh\bmu_S)\geq \underset{\bm U,S}{\mathbb{P}}(\bm U\neq \wh \bmu_S)\geq 1-\frac{I(\bm U;S)+\log 2}{\log|\mc U_{\kappa}|}.
    \end{align}
    We now work on the right-hand side term of (\ref{eq: partial fano}). For that, define the average law
    \begin{align*}
        \bar P=\frac{1}{|\mc U_{\kappa}|}\sum_{\bmu\in \mc U_{\kappa}}P_{\bmu}^{\otimes n}.
    \end{align*}
    Then, 
    \begin{align*}
        I(\bm U;S)&\overset{\text{(i)}}{=}\frac{1}{|\mc U_{\kappa}|}\sum_{\bmu\in \mc U_{\kappa}}\operatorname{KL}(P_{\bmu}^{\otimes n}\parallel \bar{P})\\
        &\overset{\text{(ii)}}{\leq}\frac{1}{|\mc U_{\kappa}|}\sum_{\bmu\in \mc U_{\kappa}}\frac{1}{|\mc U_{\kappa}|}\sum_{\bmu'\in \mc U_{\kappa}}\operatorname{KL}(P_{\bmu}^{\otimes n}\parallel P_{\bmu'}^{\otimes n})\\
        &{\leq} \ \underset{\substack{\bmu,\bmu'\in \mc U_{\kappa}\\ \bmu\neq \bmu'}}{\max}\operatorname{KL}(P_{\bmu}^{\otimes n}\parallel P_{\bmu'}^{\otimes n})\\
        &\overset{\text{(iii)}}{\leq}\frac{n}{2}\ \underset{\substack{\bmu,\bmu'\in \mc U_{\kappa}\\ \bmu\neq \bmu'}}{\max}\normso{\bmu-\bmu'}^2\\
        &\overset{\text{(iv)}}\leq2n\kappa^2\\
        &\overset{\text{(v)}}{\leq} \frac{1}{8}\c_2 (N(2{\wt\eta})-1)\\
        &{\leq}\frac{1}{8}\log|\mc U_{\kappa}|\,,
    \end{align*}
    where (i) uses the definition of mutual information together with conditioning on $\bm U$, (ii) uses convexity of the KL in the second argument, (iii) uses KL tensorisation and Lemma \ref{lemma: kl}, (iv) follows from the fact that $\normso{\bmu-\bmu'}\leq 2\kappa$ for any $\bmu,\bmu'\in \mc U_{\kappa}$, and (v) uses the definition of $\kappa$ in (\ref{def kappa}).
    
    Thus, plugging this into (\ref{eq: partial fano}),
    \begin{align*}
        \underset{\bmu\in \mc U_\kappa}{\sup}\underset{S\sim P_{\bmu}^{\otimes n}}{\mathbb{P}}(\bmu\neq \wh\bmu_S)\geq 1-\frac{\frac{1}{8}\log|\mc U_{\kappa}|+\log 2}{\log|\mc U_{\kappa}|}=\frac{7}{8}-\frac{\log 2}{\log|\mc U_{\kappa}|}.
    \end{align*}
    Since $\c_2>1/2$ and $N(2{\wt\eta})\geq 9$, we have $\log|\mc U_{\kappa}|\geq \c_2(N(2{\wt\eta})-1)>1/2\cdot 8=4$, and thus we arrive at
    \begin{align*}
        \underset{\bmu\in \mc U_\kappa}{\sup}\underset{S\sim P_{\bmu}^{\otimes n}}{\mathbb{P}}(\bmu\neq \wh\bmu_S)\geq \frac{7}{8}-\frac{\log 2}{4}\geq\frac{3}{8},
    \end{align*}
    and then (\ref{eq: pre-fano}) gives
    \begin{align*}
        \underset{\bmu\in \mc U_\kappa}{\sup}\underset{S\sim P_{\bmu}^{\otimes n}}{\mathbb{E}}\left[\ee_{\bmu}(\wh h_S)-\ee_{\bmu}(\wb_{\bmu})\right]\geq \frac{3\c_1(1-2p)}{256}\kappa^2.
    \end{align*}
    Finally, as $\widehat{h}$ was chosen arbitrarily and using the definition of $\kappa$ in (\ref{def kappa}),
    \begin{align*}
        \underset{\widehat{h}}{\inf}\ \underset{\bmu\in \mc U_\kappa}{\sup}\underset{S\sim P_{\bmu}^{\otimes n}}{\mathbb{E}}\left[\ee_{\bmu}(\wh h_S)-\ee_{\bmu}(\wb_{\bmu})\right]\geq \c (1-2p)\min\left\{\frac{2{\wt\eta}}{\lambda_1},\frac{N(2{\wt\eta})-1}{n}\right\},
    \end{align*}
    where we recall all universal constants:
    \begin{align*}
        \c_1&=\frac{\sqrt{2}}{32\pi^{3/2}e}\\
        \c_2&\in (1/2,1)\\
        \c &= \frac{3\c_1\c_2}{4096}.
    \end{align*}
    The conclusion of the theorem follows from the fact that $\mc U_\kappa\subseteq \{\bmu\in\mathbb{R}^d:\snr=1\}$.
\end{proof}

\section{Proofs for Fast and Continuously Decaying Spectrum}

\label{appx: fcd spectra}

In this section we provide the proofs of Theorem \ref{thm: matching bounds} and Corollary \ref{cor: polyexp}.

\matchingbounds*

\begin{proof}
    We split the proof into multiple steps:
    
    \textbf{Step 1:} We prove that $r_{\operatorname{eff}}({\eta})\asymp N({\eta})$ for $0<{\eta}\leq {\eta}_n^*$. If $\lambda_i\geq {\eta}$, then
    \begin{align*}
        \frac{\lambda_i}{\lambda_i+{\eta}}\geq\frac{1}{2},
    \end{align*}
    and so we can lower bound $r_{\operatorname{eff}}({\eta})$ as
    \begin{align*}
        r_{\operatorname{eff}}({\eta})=\sum_{i=1}^d\frac{\lambda_i}{\lambda_i+{\eta}}=\sum_{i:\lambda_i<{\eta}}\frac{\lambda_i}{\lambda_i+{\eta}}+\sum_{i:\lambda_i\geq {\eta}}\frac{\lambda_i}{\lambda_i+{\eta}}\geq 0+\sum_{i:\lambda_i\geq {\eta}}\frac{1}{2}=\frac{1}{2}N({\eta}),
    \end{align*}
    i.e. $r_{\operatorname{eff}}(\eta)\gtrsim N(\eta)$. To upper bound $r_{\operatorname{eff}}({\eta})$, note that
    \begin{align*}
        r_{\operatorname{eff}}({\eta})=\sum_{i:\lambda_i\geq{\eta}}\frac{\lambda_i}{\lambda_i+{\eta}}+\sum_{i:\lambda_i< {\eta}}\frac{\lambda_i}{\lambda_i+{\eta}}\leq\sum_{i:\lambda_i\geq {\eta}}1+ \frac{1}{{\eta}}\sum_{i:\lambda_i<{\eta}}\lambda_i=N({\eta})+\frac{1}{{\eta}}\sum_{\lambda_i<{\eta}}\lambda_i.
    \end{align*}
    Since, by the \eqref{eq:fcd} assumption, $\frac{1}{\eta}\sum_{\lambda_i<{\eta}}\lambda_i\lesssim N({\eta})$, we arrive at
        $r_{\operatorname{eff}}({\eta})\lesssim N({\eta})$.
        
    Thus, we have $r_{\operatorname{eff}}({\eta})\asymp N({\eta})$ for all $\eta\in (0,\eta_n^*]$.

\medskip

    \textbf{Step 2:} We now show that for all $0<{\wt\eta}\leq {\eta}_n^*$,
    \begin{align}
    \label{eq: n eta}
        N({\wt\eta})\lesssim N(2{\wt\eta}).
    \end{align}
    Indeed, $N({\wt\eta})-N(2{\wt\eta})=|\{i:{\wt\eta}\leq \lambda_i< 2{\wt\eta}\}|$, so that
    \begin{align*}
        {\wt\eta}(N({\wt\eta})-N(2{\wt\eta}))={\wt\eta} |\{i:{\wt\eta}\leq \lambda_i< 2{\wt\eta}\}|\leq \sum_{{\wt\eta}\leq \lambda_i<2{\wt\eta}}\lambda_i\leq \sum_{\lambda_i<2{\wt\eta}}\lambda_i\lesssim{\wt\eta} N(2{\wt\eta}),
    \end{align*}
    where the last inequality uses the \eqref{eq:fcd} assumption for $2\wt\eta\in (0,2\eta_n^*]$. Thus, (\ref{eq: n eta}) follows.

\medskip

    \textbf{Step 3:} We show that $N({\eta}_n^*)\asymp n{\eta_n^*}$, so that   $r_{\operatorname{eff}}({\eta}_n^*)\asymp n{\eta}_n^*$ also holds by Step 1. 
    
    The definition of ${\eta}_n^*=\sup\{{\eta}>0:N({\eta})\geq n{\eta}\}$ implies that $N(2\eta_n^*)<2n\eta_n^*$. Thus, according to (\ref{eq: n eta}),
    \begin{align*}
        N(\eta_n^*)\lesssim N(2\eta_n^*)\lesssim n\eta_n^*,
    \end{align*}
    Since $\eta_n^*=\sup\{{\wt\eta}>0:N({\wt\eta})\geq n{\wt\eta}\}$, we have $N(\eta_n^*)\geq n\eta_n^*$, and thus $N(\eta_n^*)\asymp n\eta_n^*$.

\medskip

    \textbf{Step 4:} We first evaluate the upper bound in Theorem \ref{thm: upper bound final} at ${\eta}=\eta_n^*$, after which we will check its two conditions that $\eta_n^*$ must satisfy.  Note that, as $r_{\operatorname{eff}}(\eta_n^*)\asymp n\eta_n^*$ and $\log\log(n)/n\lesssim \eta_n^*$ by assumption, this upper bound is
    \begin{align}
    \label{derivation for condition}
        \ee(\bw_\tau)-\ee(\wb)\lesssim \frac{r_{\operatorname{eff}}(\eta_n^*)+\log\log(n)}{n}+\eta_n^*\asymp \frac{n\eta_n^*+\log\log(n)}{n}+\eta_n^*\lesssim \eta_n^*\,.
    \end{align}
    
    Now, note that the first condition $n\gtrsim r_{\operatorname{eff}}(\eta_n)$ holds because $r_{\operatorname{eff}}(\eta_n)/n\asymp\eta_n\rightarrow 0$ as $n\rightarrow\infty$.  
    
    To verify condition (\ref{thm condition}), note that its left hand side is $\lesssim \eta_n^*$ by the derivation in \eqref{derivation for condition}, and its right hand side is independent of $n$. Since $\eta_n^*\rightarrow 0$ as $n\rightarrow\infty$ by assumption, the condition \eqref{thm condition} is also verified.

\medskip

    \textbf{Step 5:} Finally, we evaluate the statistical lower bound of Theorem \ref{thm: stat lower bound} at ${\wt\eta}=\eta_n^*/2$. For its condition $N(2\eta)>9$, note that, by the result of Step 3 and the assumption $\log\log(n)/n\lesssim \eta_n^*$,
    \begin{align*}
        N(2\eta_n^*/2)=N(\eta_n^*)\asymp n\eta_n^*\gtrsim \log\log(n), 
    \end{align*}
    so that the lower bound condition holds. At ${\wt\eta}=\eta_n^*/2$, since $N(\eta_n^*)\asymp n\eta_n^*$, the lower bound is
    \begin{align*}
        \underset{\widehat{h}}{\inf}\ \underset{\bmu}{\sup}\underset{S\sim P_{\bmu}^{\otimes n}}{\mathbb{E}}\left[\ee_{\bmu}(\wh h_S)-\ee_{\bmu}(\wb_{\bmu})\right]\gtrsim \min\left\{\frac{\eta_n^*}{2},\frac{N(\eta_n^*)}{n}\right\}\asymp \eta_n^*,
    \end{align*}
    This concludes that the upper and statistical lower bounds of Theorems \ref{thm: upper bound final} and \ref{thm: stat lower bound}, respectively, are matching with common rate $\asymp\eta_n^*$.

\end{proof}

\corpolyexp*

\begin{proof}

The proof for both spectra will verify the conditions \eqref{eq:fcd}, $\eta_n^*\rightarrow 0$ as $n\rightarrow\infty$, and $\log\log(n)/n\lesssim \eta_n^*$ of Theorem \ref{thm: matching bounds}, and then apply its conclusion. For this, we will compute in both cases $N(\eta)$ and, as a result, deduce $\eta_n^*$. Note that the existence of an appropriate $\bmu\in\mathbb{R}^d$ with $\norm{\bSigma^{-1}\bmu}\leq \sqrt{2/\lambda_1}$ for which the presented rates hold is due to Theorem \ref{thm: stat lower bound}.

\medskip

    \textbf{Polynomially decaying spectrum:} Suppose $\c_-i^{-2a}\leq \lambda_i\leq \c_+i^{-2a}$ for some constants $\c_+\geq \c_->0$. Then $i\leq (\c_-/{\wt\eta})^{1/(2a)}$ implies that $\lambda_i\geq{\wt\eta}$, i.e. $N({\wt\eta})\gtrsim {\wt\eta}^{-1/(2a)}$; and $\lambda_i\geq {\wt\eta}$ gives $i\leq (\c_+/{\wt\eta})^{1/(2a)}$, i.e. $N({\wt\eta})\lesssim {\wt\eta}^{-1/(2a)}$; combining these two we get
    \begin{align*}
        N({\wt\eta})\asymp {\wt\eta}^{-\frac{1}{2a}}.
    \end{align*}

    Next, if $\lambda_i<{\wt\eta}$, then $i>(\c_-/{\wt\eta})^{1/(2a)}$, so that
    \begin{align*}
        \sum_{i:\lambda_i<{\wt\eta}}\lambda_i\leq \c_+\sum_{i:\lambda_i<{\wt\eta}}i^{-2a}\leq \c_+\sum_{i>(c_-/{\wt\eta})^{1/(2a)}}i^{-2a}\lesssim \int_{(c_-/{\wt\eta})^{1/(2a)}}^{\infty}t^{-2a}\text{d}t\asymp {\wt\eta}^{1-\frac{1}{2a}}.
    \end{align*}
    Thus,
    \begin{align*}
        \frac{1}{{\wt\eta}}\sum_{i:\lambda_i<{\wt\eta}}\lambda_i\lesssim {\wt\eta}^{-\frac{1}{2a}}\asymp N({\wt\eta}),
    \end{align*}
    and the \eqref{eq:fcd} requirement of Theorem \ref{thm: matching bounds} is met (for any $\wt \eta>0$).

    By Step 3 of Theorem \ref{thm: matching bounds}, $N({\eta}_n^*)\asymp n{\eta}_n^*$, which gives
    \begin{align*}
        ({\eta}_n^*)^{-\frac{1}{2a}}\asymp n{\eta}_n^*,\quad\quad\text{i.e.}\quad {\eta}_n^*\asymp n^{-\frac{2a}{2a+1}},
    \end{align*}
    which is the common upper and statistical lower bound matching rate. We have ${\eta}_n^*\rightarrow 0 $ as $n\rightarrow\infty$ and $\log\log(n)/n\lesssim {\eta}_n^*$.

\medskip

    \textbf{Exponentially decaying spectrum:} Suppose $\c_-e^{-bi}\leq \lambda_i\leq \c_+e^{-bi}$ for some constants $\c_+\geq\c_->0$. Then, $i\leq \frac{1}{b}\log(\c_-/{\wt\eta})$ implies that $\lambda_i\geq {\wt\eta}$, i.e. $N({\wt\eta})\gtrsim \log(1/{\wt\eta})$; and $\lambda_i\geq {\wt\eta}$ gives $i\leq \frac{1}{b}\log(\c_+/{\wt\eta})$, i.e. $N({\wt\eta})\lesssim\log(1/{\wt\eta})$; combining these two we get
    \begin{align}
    \label{exp N eta log}
        N({\wt\eta})\asymp \log\left(\frac{1}{{\wt\eta}}\right)\quad\quad\text{for any}\quad \wt\eta>0.
    \end{align}

    Next, if $\lambda_i<{\wt\eta}$, then $i>\frac{1}{b}\log(\c_-/{\wt\eta})$, so that
    \begin{align*}
        \sum_{i:\lambda_i<{\wt\eta}}\lambda_i\leq \c_+\sum_{i:\lambda_i<{\wt\eta}}e^{-bi}\leq \c_+\sum_{i>\frac{1}{b}\log(\c_-/{\wt\eta})}e^{-bi}\lesssim\int_{\frac{1}{b}\log(\c_-/{\wt\eta})}^{\infty} e^{-bt}\text{d}t\asymp {\wt\eta}.
    \end{align*}
    Thus,
    \begin{align*}
        \frac{1}{{\wt\eta}}\sum_{i:\lambda_i<{\wt\eta}}\lambda_i\lesssim 1\lesssim \log\left(\frac{1}{{\wt\eta}}\right)\asymp N({\wt\eta}),
    \end{align*}
    and the \eqref{eq:fcd} requirement of Theorem \ref{thm: matching bounds} is met (for any $\wt \eta>0$). 

    Now, in order to identify ${\eta}_n^*$ in this case, we proceed as follows. First, note that by \eqref{exp N eta log} we have that for any $\wt\eta>0$, $\c_1\log(1/{\wt\eta})\leq N({\wt\eta})\leq \c_2 \log(1/{\wt\eta})$ for some constants $\c_2\geq \c_1>0$. Consider
    \begin{align*}
        {\eta}_n^-=\frac{\c_1}{4}\frac{\log(n)}{n}.
    \end{align*}
    Then, for $n$ larger than an appropriate universal constant, since $\c_1\log(1/{\wt\eta})\leq N({\wt\eta})$ for any $\wt \eta>0$,
    \begin{align*}
        N({\eta}_n^-)\geq \c_1\log\left(\frac{n}{(\c_1/4)\log(n)}\right)=\c_1\left(\log(n)\!-\!\log\log(n)\! + \!\log\left(\frac{4}{\c_1}\right)\right)\geq \frac{\c_1}{2}\log(n) > n{\eta}_n^-.
    \end{align*}
    Thus, by the definition of ${\eta}_n^*=\sup\{{\eta}>0:N({\eta})\geq n{\eta}\}$, we have ${\eta}_n^-\leq {\eta}_n^*$. Similarly, take
    \begin{align*}
        {\eta}_n^+=2\c_2\frac{\log(n)}{n},
    \end{align*}
    Then, for $n$ larger than an appropriate universal constant, since $N({\wt\eta})\leq \c_2 \log(1/{\wt\eta})$ for any $\wt \eta>0$,
    \begin{align*}
        N({\eta}_n^+)\leq \c_2\log\left(\frac{n}{2\c_2\log(n)}\right)=\c_2\left(\log(n) -\log\log(n)-\log(2\c_2)\right)<2\c_2\log(n) =n{\eta}_n^+.
    \end{align*}
    Thus, by the definition of ${\eta}_n^*=\sup\{{\eta}>0:N({\eta})\geq n{\eta}\}$, we have ${\eta}_n^*<{\eta}_n^+$. Thus,
    \begin{align*}
        \frac{\c_1}{4}\frac{\log(n)}{n}\leq {\eta}_n^* <2\c_2\frac{\log(n)}{n},\quad\quad\text{i.e.}\quad {\eta}_n^*\asymp\frac{\log(n)}{n},
    \end{align*}
    which is the common upper and statistical lower bound matching rate. Clearly we have ${\eta}_n^*\rightarrow 0$ as $n\rightarrow\infty$ and $\log\log(n)/n\lesssim {\eta}_n^*$.
    
\end{proof}

\section{Proof of the Interpolation Lower Bound}
\label{app: interpolation lower bound proof}

In this section, we prove Theorem~\ref{thm:interpolation}, following a  similar sequence of steps  to the proof of \cite[Theorem 4.2]{wu2025benefits}. We show that, with high probability, every interpolating linear classifier $\bx\mapsto \sign(\bx^\top\bw)$ with $\bw\in \mc{I}:=\{\bw\in \bbR^d:
  \min_{i\in[n]}Y_i\bX_i^\top \bw>0\}$ incurs an excess zero-one risk bounded below. 
  
  Throughout, we use  $ \sph^{m-1}:=\{\bx\in \bbR^m: \|\bx\|=1\}$  to denote the unit sphere in $\bbR^m$,     $(\bm a, \bm b)\in \bbR^{n_a}\times \bbR^{n_b}$ to denote the vertical concatenation of  $\bm a\in \bbR^{n_a}$ and $\bm b\in \bbR^{n_b}$, and  $\phi$  and $\Phi$ to denote the PDF and CDF, respectively,  of $\mathcal{N}(0,1)$.

  \subsection{Setup via Whitening and Rotation}

\label{interpolation proof sketch}
  
\paragraph{Whitening:} We work in a   \emph{whitened}  covariate space to simplify the presentation of the proofs. We write 
 \begin{align*}
 \bZ:=\bSigma^{-1/2}\bX\in \bbR^d,\qquad  \bvareps_0:=\bSigma^{-1/2}\bvareps\in \bbR^d,  \qquad \text{and}\qquad \bv^*:=\bSigma^{1/2}\bw^*\in \bbR^d
 \end{align*}
  for the whitened covariate,  noise, and     Bayes linear classifier, respectively. The Gaussian mixture model then takes the form
 \begin{align*}
     \bZ=\wt{Y}\bv^* +\bvareps_0 \quad \text{with}\quad \bvareps_0\sim\mc{N}(\bzero, \bI_d)\,.
 \end{align*}

 \paragraph{Rotation and contrast  with \cite{wu2025benefits}:} 
 Our interpolation  lower bound is proved after \emph{rotating} to  a convenient  coordinate system.    Unlike \cite[Theorem 4.2]{wu2025benefits}, this does \emph{not} impose any structural  assumption (e.g.,   sparsity) on $\bv^*$ itself.    Indeed, since the whitened Gaussian noise $\bvareps_0$ is rotationally invariant, for any nonzero  $\bv^*\in \bbR^d$ and \emph{any} $ \s\in[d]$, there exists  an orthogonal matrix  $\bR\in \bbR^{d\times d}$ such that 
 \begin{align*}
 \bR\bv^*=\|\bv^*\|\Big(\tfrac{1}{\sqrt{\s}}\bone_\s, \bzero_{d-\s}\Big)\,.
 \end{align*}
 Specifically, a valid choice is  $\bR=\bI_d -  \ind\{\bu\ne\bzero\}\cdot 2\bu\bu^\top/\|\bu\|^2$ where $\bu:=\bv^*-\|\bv^*\|(\frac{1}{\sqrt{\s}}\bone_\s, \bzero_{d-\s})$. 
 Thus, in the rotated coordinates, the signal is supported on the first $\s$ coordinates, while the remaining $d-\s$ coordinates are nuisance directions. Importantly, $k$   and $\bR$ are free parameters that  merely define  a  choice of basis that influences the behaviour of the lower bound, so this representation applies to \emph{any} $\bv^*$ rather than  strictly sparse  ones only. For clarity, we write
 \begin{align*}
 \bZ_{\bR}:=\bR\bZ=(\bS, \bN)\in \bbR^{\s}\times \bbR^{d-\s}\qquad \text{and} \qquad \bv_{\bR}^*:=\bR\bv^*=(\bpsi^*, \bzero_{d-k})\in \bbR^{\s}\times \bbR^{d-\s} 
 \end{align*}
  for  the signal-nuisance decomposition. For the $i$-th sample $(\bX_i, Y_i)$,  we write $\bZ_{\bR,i},\bS_i$ and $\bN_i$  for the corresponding quantities.

\paragraph{Proof sketch:} The key proof  idea is to show   that using only the $\s$ signal directions  can create large negative margins on a constant fraction of samples (Section~\ref{appx:interpolatio_signal_side}), so any interpolating classifier  must use the $d-\s$ nuisance directions to fix those bad margins (Section~\ref{appx:interpolation_nuisance_side}). This reliance on nuisance directions ultimately incurs the lower bound in Theorem \ref{thm:interpolation} (Section~\ref{appx:proof_interpolation_thm}). To begin, we prove some preliminary lemmas in Section~\ref{appx:interpolation_preliminary}.

 \subsection{Preliminary Lemmas}\label{appx:interpolation_preliminary}
\begin{lemma}[Bound on norm of signal coordinates]
\label{lem:assumption_head_bounded_norm}
For every $\delta\in(0,1)$,
\begin{align*}
   \P\left(\max_{i\in[n]}\|\bS_{i}\|\le \sqrt{\s}+ \sqrt{2\log(2n/\delta)}+ \|\bv^*\| \right)\ge 1-\tfrac{\delta}{2}\,.
\end{align*}

\end{lemma}
\begin{proof}
 Since $ \bZ_{\bR}=\wt{Y}\bv_{\bR}^* +  \bvareps_{0}$ with   $\wt{Y}\in\{\pm1\}$, $\bv^*_{\bR}=\bR\bv^*=(\bpsi^*, \bzero_{d-\s})$   and $\bvareps_{0}\sim\mc{N}(\bzero, \bI_d)$, we have
\begin{align*}
    \|\bS\|= \|\bZ_{\bR [1:\s]}\|= \|\wt{Y}\bv^*_{\bR[1:\s]} +  \bvareps_{0 [1:\s]}\|\le \|\bpsi^*\| + \|\bvareps_{0[1:\s]}\|=\|\bv^*\|+\|\bvareps_{0[1:\s]}\|,
\end{align*}
where by $\bm Z_{\bm R[1:k]}$ we denote the vector obtained from the first $k$ coordinates of $\bm Z_{\bm R}$.
The result follows from standard Gaussian concentration on $\|\bvareps_{0[1:\s]}\|$:
\begin{align*}
    \P\left(\|\bvareps_{0[1:\s]}\|> \sqrt{\s}+ \sqrt{2\log(2n/\delta)}\right) \le \tfrac{\delta}{2n},
\end{align*}
together with a union bound.
\end{proof}

The next lemma is an angular calibration result, analogous to  \cite[Lemma C.1]{wu2025benefits}, which lower bounds the excess zero-one  risk in terms of  the angle between the learned linear classifier   and the Bayes classifier. Working in the whitened space,   let $\bv=\bSigma^{1/2}\bw$ for $\bw\in \bbR^d\setminus \{\bzero\}$  and  recall    $\bv^*=\bSigma^{1/2}\bw^*$. The Euclidean angle $\angle(\bv, \bv^*)$ coincides with the $\bSigma$-angle $\angle_{\bSigma}(\bw, \bw^*)$. 

    \begin{lemma}[Angular calibration of excess zero-one risk] \label{lem:angle_calibration}Let    $\theta:=\angle(\bv, \bv^*)
\in [0, \pi]$. Then,
\begin{align*}
    \mc{E}(\bw) - \mc{E}(\bw^*)\ge \tfrac{1}{\sqrt{2\pi}} (1-2p)
   (1-\cos\theta) \norm{\bv^*}\exp (-\tfrac{1}{2}\norm{\bv^*}^2).
\end{align*}
\end{lemma}
\begin{proof}
To begin, recall that $\mc{E}(\bw)-\mc{E}(\bw^*)=(1-2p) (\P(\wt{Y}\bv^\top \bZ\le 0) - \P(\wt{Y} (\bv^*)^\top \bZ\le0))$, where  
    we  have
    \begin{align*}
       \mathbb{P}(\widetilde{Y}\bv^\top \bZ\leq 0)&=\mathbb{P}(\widetilde{Y}\bv^\top(\widetilde{Y}\bv^*+\bvareps_0)\leq 0)=\mathbb{P}\big(\bv^\top \bv^*+\mathcal{N}(\bzero,\norm{\bv}^2)\leq \bzero\big)\\
       &=\Phi\left(-\frac{\bv^\top \bv^*}{\norm{\bv}}\right)=\Phi(-\|\bv^*\|\cos\theta)\,,
    \end{align*}
   and similarly, $\P(\wt{Y}(\bv^*)^\top \bZ\le0)= \Phi\big(-\norm{\bv^*}\big)$. 
   Hence, combining the two yields
    \begin{align*}
    &\mathbb{P}(\widetilde{Y}\bv^\top \bZ\leq 0)-\mathbb{P}(\widetilde{Y}(\bv^*)^\top \bZ\leq 0)    =\Phi(-\norm{\bv^*}\cos\theta)-\Phi(-\norm{\bv^*})
    =\int_{-\norm{\bv^*}}^{-\norm{\bv^*}\cos\theta}\phi(x)\de x\\
        &\geq \int_{-\norm{\bv^*}}^{-\norm{\bv^*}\cos\theta}\underset{-\norm{\bv^*}\leq x\leq -\norm{\bv^*}\cos\theta}{\min}\phi(x)\de x 
        = \big(-\|\bv^*\| \cos\theta+\|\bv^*\|\big) \cdot \phi(-\|\bv^*\|)\,.
    \end{align*}
\end{proof}

 \subsection{Signal Component  Induces Large Negative Margins on Constant Fraction of Samples}\label{appx:interpolatio_signal_side}
 Lemma \ref{lem:pop_bad_margin}, analogous to   \cite[Lemma C.2]{wu2025benefits}, gives a uniform lower bound on the probability of  large negative margins for any fixed classifier    $\bar{\bv}\in \sph^{d-1}$, where  the bar $\bar{\cdot}$   indicates  a unit vector.

 \begin{lemma}[Bad margin occurs with constant probability]\label{lem:pop_bad_margin}
 For any  $\bar{\bv}\in \sph^{d-1}$ and    every $\C\in(0,1)$, 
\begin{align*}
    \P\left(Y \bZ_{\bR}^\top\bar{\bv}< -\C\right) \ge  \varrho(\|\bv^*\|):=2\left(\Phi(-\C) - \Phi(-1)\right)\left(\frac{1-2p}{1+\exp\big(2\|\bv^*\|(\|\bv^*\|+1)\big)}+p\right)\,.
\end{align*}
\end{lemma}

\begin{proof}  Since  $\bar{\bv}':=\bR^\top\bar{\bv}\in \sph^{d-1}$ for any $\bar{\bv}\in \sph^{d-1}$, we have       $\P(Y\bZ_{\bR}^\top \bar{\bv}<-\C)=\P(Y\bZ^\top \bar{\bv}'<-\C)$. We bound $\P(Y\bZ^\top \bar{\bv}'<-\C)$ in the rest of the proof.
  Recall $\sigma(t)=1/(1+\exp(-t))$ denotes the sigmoid function.  For the Gaussian mixture model 
$\bZ=\wt{Y} \bv^*+ \bvareps_0$, 
  we   recall from Lemma  \ref{lemma: posteriors and excess}  that 
    \begin{align*}
        \mathbb{P}(Y=y\mid \bZ)=(1-2p)\sigma(2y  \bZ^\top \bv^*)+p\qquad \text{for }y\in\{\pm1\}\,.
    \end{align*}
  Using this we can write
    \begin{align}
        &\mathbb{P}(Y\bZ^\top \bar{\bv}'<-\C)\nonumber\\
        =&\E\left[\ind\{Y=1,   \bZ^\top \bar{\bv}'<-\C\}+\ind\{Y=-1,\bZ^\top \bar{\bv}'>\C\}\right]\nonumber\\
         =&\E\left[\P(Y=1\mid \bZ)\ind\{\bZ^\top \bar{\bv}'<-\C\}+\P(Y=-1\mid \bZ)\ind\{\bZ^\top \bar{\bv}'>\C\}\right] \nonumber\\
       =&\mathbb{E}\left[\left((1-2p)\sigma (2 \bZ^\top\bv^*)+p\right)\mathbbm{1}\{\bZ^\top \bar{\bv}'<-\C\}+\left((1-2p)\sigma (-2\bZ^\top\bv^*)+p\right)\mathbbm{1}\{\bZ^\top \bar{\bv}'>\C\}\right] \nonumber\\
       \stackrel{\text{(i)}}{\ge} &\E\left[\left((1-2p) \sigma(-2|\bZ^\top \bv^*|)+ p\right) \ind\{|\bZ^\top \bar{\bv}'|>\C\}\right]\nonumber\\
        \stackrel{\text{(ii)}}{\ge} &\big((1-2p)\sigma(-\kappa)+p\big)\big(1-\underbrace{\P(|\bZ^\top \bar{\bv}'|\le \C)}_{\text{Term I}} - \underbrace{\P(|\bZ^\top \bv^*|\ge \kappa/2)}_{\text{Term II}}\big) \,\qquad \text{for any }\kappa\ge0\,,\label{eq:prob_large_negative_margin}
    \end{align}
where step (i)  holds because   $\sigma(-|t|)\le \min\big\{\sigma(t), \sigma(-t)\big\}$ for any $t\in\mathbb{R}$, and  
step (ii) holds for any constant  $\kappa\ge 0$. %
We next upper bound Term I and Term II separately and   choose $\kappa$ accordingly. 
 
To this end, first note that  for any   vector $\bu\in \bbR^d$, 
  $\bZ^\top \bu\mid \wt{Y}\sim\mathcal{N}\big(\wt{Y} \bu^\top \bv^*,\norm{\bu}^2\big)$, so we can write 
    \begin{align*}
      &  \mathbb{P}\big(|\bZ^\top \bu|\leq \C\big)\\
      &=
\E_{\wt{Y}}\left[ \P\big(|\bZ^\top \bu|\le \C\mid \wt{Y}\big)\right]\nonumber\\
& =\E_{\wt{Y}}\left[ \P\big(-\C\le \bZ^\top \bu\le \C\mid \wt{Y}\big) \right]\nonumber\\
        &=\mathbb{E}_{\wt{Y}}\left[\Phi\left(\frac{\C-\wt{Y} \bu^\top \bv^*}{\norm{\bu}}\right)-\Phi\left(\frac{-\C-\wt{Y}  \bu^\top \bv^*}{\norm{\bu}}\right)\right]\nonumber\\
        &=\frac{1}{2}\left[\Phi\left(\frac{\C- \bu^\top \bv^*}{\norm{\bu}}\right)-\Phi\left(\frac{-\C-\bu^\top \bv^*}{\norm{\bu}}\right)+\Phi\left(\frac{\C+\bu^\top \bv^*}{\norm{\bu}}\right)-\Phi\left(\frac{-\C+\bu^\top \bv^*}{\norm{\bu}}\right)\right]\nonumber\\
        &\stackrel{\text{(i)}}{=}\Phi\left(\frac{\C-\bu^\top \bv^*}{\norm{\bu}}\right)+\Phi\left(\frac{\C+\bu^\top \bv^*}{\norm{\bu}}\right)-1 \stackrel{\text{(ii)}}{\leq} 
        1-2\Phi\left(-\frac{\C}{\norm{\bu}}\right)\nonumber\,,%
    \end{align*}
    where step (i) is due to  $\Phi(t)+\Phi(-t)=1$ and step (ii) uses the fact that for any $a\ge0$, $t\mapsto \Phi(a-t)+\Phi(a+t)$ is maximised at $t=0$ for $t\in \bbR$. Then taking $\bu=\bar{\bv}'$ in the last display and using $\|\bar{\bv}'\|=1$ gives
    \begin{align}
     \text{Term I}=   \mathbb{P}(|\bZ^\top \bar{\bv}'|\leq \C)\leq %
        1-2\Phi(-\C)
        \,.\label{eq:prob_Zv<C}
    \end{align}
    To control Term II, first note that by choosing  $\kappa=2\norm{\bv^*}(\|\bv^*\|+1)$, we have  
    \begin{align*}
        \text{Term II}=\P\big(|\bZ^\top \bv^*|\ge \kappa/2\big)= \mathbb{P}\big(|\bZ^\top \bv^*|\geq\norm{\bv^*}(\|\bv^*\|+1)\big)
       =\P\left(\frac{|\bZ^\top \bv^*|}{\|\bv^*\|}\ge \|\bv^*\|+1\right)\,,
    \end{align*}
    where $\bZ^\top \bv^*/\|\bv^*\|=\wt{Y}\|\bv^*\|+G$ with $G\sim \mc{N}(0,1)$, which by the triangle inequality implies $|\bZ^\top \bv^*|/\|\bv^*\|\le |\wt{Y}| \|\bv^*\|+|G| =\|\bv^*\|+|G|. $ Therefore, we have 
    \begin{align}
        \text{Term II}\le \P(\|\bv^*\|+|G|\ge\|\bv^*\|+1)=\P(|G|\ge1)=2\Phi(-1)\,.\label{eq:prob_ZvB>kappa/2}
    \end{align}

  Finally, substituting  \eqref{eq:prob_Zv<C} and \eqref{eq:prob_ZvB>kappa/2} into \eqref{eq:prob_large_negative_margin} along with $\kappa=2\norm{\bv^*}(\|\bv^*\|+1)$,  we obtain 
  \begin{align*}
    \mathbb{P}(Y\bZ^\top \bar{\bv}'<-\C)
        &\geq\left(\frac{1-2p}{1+\exp\big(2\|\bv^*\|(\|\bv^*\|+1)\big)}+p\right)\cdot 2\left(\Phi(-\C)-\Phi(-1)\right)\,,
    \end{align*}
 where    $2\left(\Phi(-\C) - \Phi(-1)\right)> 0$ given $\C\in(0,1)$.

    \end{proof}

Lemma~\ref{lem:pop_bad_margin}, analogous to \cite[Lemma C.2]{wu2025benefits}, gives a \emph{population} guarantee for any \emph{fixed} classifier  $\bar{\bv}\in \sph^{d-1}$, Lemma~\ref{lem:emp_bad_margin} below  lifts this to a \emph{finite-sample} statement that holds uniformly over \emph{all}   signal components that an interpolator might learn. To isolate the role of the signal component, we consider classifiers 
of the form $\bar{\bv}=(\bar{\bpsi}, \bzero_{d-\s})\in \sph^{d-1}$ with $\bar{\bpsi}\in \sph^{\s-1}$. Then   $\bZ^\top_{\bR}\bar{\bv}=\bS^\top \bar{\bpsi}$, implying the margin depends only on the signal component $\bar{\bpsi}$. 
We show that, for every candidate  signal direction $\bar{\bpsi}\in \sph^{\s-1}$,  a constant fraction of samples have large negative margins. As before, the bar $\bar{\cdot}$ denotes a unit vector.

\begin{lemma}[Bad margins occur on a constant fraction of samples]
\label{lem:emp_bad_margin} 
 Fix $\C\in(0,1)$ and define $\varrho(\|\bv^*\|):=2\left(\Phi(-\C) - \Phi(-1)\right)\big\{(1-2p)/{\big[1+\exp\big(2\|\bv^*\|(\|\bv^*\|+1)\big)\big]}+p\big\}>0$. 
For every  $\delta\in(0,1)$, if 
for a sufficiently large constant $\C_1>0$,  
\begin{align}\label{eq:n_large}
        n\ge \frac{\C_1}{\varrho(\|\bv^*\|)}\cdot \left(\s \log\left(\frac{ \sqrt{\s}+\|\bv^*\|}{\C} \right)+ \s\log\log\left(\frac{\C_1\s}{\varrho(\|\bv^*\|)\delta}\right)+ \log\tfrac{2}{\delta}\right)\,,
    \end{align}
    then the following holds: 
\begin{align*}
    \P\left(\forall \bar{\bpsi}\in \sph^{\s-1},\;  \sum_{i=1}^n\ind \left\lbrace Y_i\bS_i^\top \bar{\bpsi}\le -\tfrac{\C}{2}\right\rbrace \ge \tfrac{n}{2}\varrho(\|\bv^*\|)\right)\ge 1-\delta\,.
\end{align*}

\end{lemma}
\begin{proof}
Define the  binary random variables
\begin{align*}
    \xi_i(\bar{\bpsi}) :=\ind\{Y_i\bS_i^\top\bar{\bpsi}\le -\C\}
    ,\quad \text{for}\;i\in[n]\,.
\end{align*}
By Lemma~\ref{lem:pop_bad_margin},  $\E[\xi_i(\bar{\bpsi})]\ge \varrho(\|\bv^*\|)$. So by the Chernoff bound, we have
\begin{align}
   \P\left(\sum_{i=1}^n\xi_i(\bar{\bpsi}) < \tfrac{n}{2}\varrho(\|\bv^*\|)\right)
   &\le \P\left(\sum_{i=1}^n\xi_i(\bar{\bpsi}) < \tfrac{n}{2}\E[\xi_1(\bar{\bpsi})]\right)\le \exp\left(-\tfrac{n}{8}\E[\xi_1(\bar{\bpsi})]\right)\nonumber\\
   &\le \exp\left(-\tfrac{n}{8}\varrho(\|\bv^*\|)\right)\,.\label{eq:chernoff}
\end{align}
    We now use a covering argument to turn a \emph{pointwise} probability bound which holds for a fixed $\bar{\bv}=(\bar{\bpsi}, \bzero_{d-\s})$  into a \emph{uniform} bound
    over all $\bar{\bpsi}\in \sph^{\s-1}$. Since $|\sph^{\s-1}|=\infty$, we discretise  $\sph^{\s-1}$ with a finite $\epsilon$-$\ell_2$-covering $\mc{C}$, so that a union bound can be applied. 
    Since $|\mc{C}|\le \left(\frac{3}{\epsilon}\right)^\s$, applying a union bound over  $\mc{C}$ and using \eqref{eq:chernoff}, we obtain  
\begin{align}
    \P\left(\exists \bar{\bpsi}'\in\mc{C}: \sum_{i=1}^n \xi_i(\bar{\bpsi}') <\tfrac{n}{2}\varrho(\|\bv^*\|)\right)
    &\le |\mc{C}| \exp\left(-\tfrac{n}{8}\varrho(\|\bv^*\|)\right)\nonumber\\
    &\le \exp\Big(-\tfrac{n}{8}\varrho(\|\bv^*\|)+ \s\log\tfrac{3}{\epsilon}\Big)\le \tfrac{\delta}{2}\,,\label{eq:union_bound_on_C}
\end{align}
where the last inequality holds provided  $n\ge 8\varrho(\|\bv^*\|)^{-1}\cdot \left(\s \log\frac{3}{\epsilon}+ \log\frac{2}{\delta}\right)$, which is ensured by the  sample size $n$ condition \eqref{eq:n_large}  for a sufficiently large constant $\C_1>0$.

Define  $\textsf{A}:= \sqrt{\s}+\sqrt{2\log(2n/\delta)}+\|\bv^*\|$. On the high-probability event $\{\max_{i\in[n]}\|\bS_{i}\|\le \textsf{A}\}$   from  Lemma~\ref{lem:assumption_head_bounded_norm}, we choose $\epsilon=\C/(2\textsf{A})$.
Then, for every $\bar{\bpsi}\in \sph^{\s-1}$, 
 there exists  $\bar{\bpsi}'\in\mc{C}$ 
such that for all $i\in[n]$,
    \begin{align*}
        \Big|\bS_{i}^\top (\bar{\bpsi} -\bar{\bpsi}')\Big|\le \|\bS_{i}\| \big\|\bar{\bpsi}-\bar{\bpsi}'\big\|\le\textsf{A} \epsilon\le \tfrac{\C}{2}\,,
    \end{align*}
    which implies $-\frac{\C}{2}\le Y_i\bS_i^\top \bar{\bpsi}-Y_i\bS_i^\top\bar{\bpsi}'\le\frac{\C}{2}$, 
    so whenever $Y_i\bS_i^\top \bar{\bpsi}'\le -\C$, 
    we must have $ Y_i\bS^\top \bar{\bpsi}\le \frac{\C}{2}-\C = -\frac{\C}{2}$. This means that  on the event $\{\max_{i\in[n]}\|\bS_{i}\|\le \textsf{A}\}$, 
    \begin{align}
      \sum_{i=1}^n \xi_i(\bar{\bpsi}')=   \sum_{i=1}^n\ind\{Y_i\bS_i^\top \bar{\bpsi}'\le -\C\}\le  \sum_{i=1}^n\ind\{Y_i\bS_i^\top \bar{\bpsi}\le -\tfrac{\C}{2}\}\,.\label{eq:v_vprime_bound}
    \end{align}
Therefore, using \eqref{eq:v_vprime_bound} and \eqref{eq:union_bound_on_C} we can write
\begin{align*}
  & \P\left(\forall \bar{\bpsi}\in\sph^{\s-1}, \;\sum_{i=1}^n \ind\{Y_i\bS_i^\top \bar{\bpsi}\le -\tfrac{\C}{2}\} \ge \tfrac{n}{2}\varrho(\|\bv^*\|)\right)\\
   &= 1-\P\left(\exists \bar{\bpsi}\in\sph^{\s-1}: \sum_{i=1}^n \ind\{Y_i\bS_i^\top \bar{\bpsi}\le -\tfrac{\C}{2}\} < \tfrac{n}{2}\varrho(\|\bv^*\|)\right)\\
   & \ge 1 - \P\left(\max_{i\in[n]}\|\bS_{i}\|> \textsf{A}\right)-\P\left(\exists \bar{\bpsi}'\in \mc{C}: \sum_{i=1}^n \xi_i(\bar{\bpsi}')<\tfrac{n}{2}\varrho(\|\bv^*\|)\right)\\
   &\ge 1-\tfrac{\delta}{2}-\tfrac{\delta}{2}=1-\delta\,.
\end{align*}
\end{proof}

    \subsection{Nuisance Component Cannot  Fix Large Negative Margins for Affected Samples}\label{appx:interpolation_nuisance_side}

Recall the rotated Bayes classifier admits a \emph{signal-nuisance} decomposition $\bv_{\bR}^*=(\bpsi^*, \bzero_{d-\s})\in \bbR^{d}$.  Lemma~\ref{lem:emp_bad_margin}  shows that   a learned classifier  using only its  signal component, i.e., $\bar{\bv}=(\bar{\bpsi}, \bzero_{d-\s})\in \sph^{d-1}$,  produces large negative margins on a constant fraction of   samples. To   interpolate, a \emph{general} classifier $(\bpsi, \bnu)\in \bbR^\s\times \bbR^{d-\s}$ must therefore use its nuisance component $\bnu$ to  \emph{correct} these bad  margins.  Lemma~\ref{lem:repair_bad_margin}  shows that this correction is limited: even the  best nuisance direction  $\bar{\bnu}\in \sph^{d-\s-1}$  can only achieve  margins 
\emph{no larger} than $\mathfrak{m}$  over the affected subset  of samples $\mc{J}(\bar{\bpsi})$.

\begin{lemma}[Nuisance component   cannot fix bad margins]\label{lem:repair_bad_margin}
For each $\bar{\bpsi}\in \sph^{\s-1}$,
let $ \mc{J}(\bar{\bpsi})\subset [n]$ be a random subset independent of $(\bN_i)_{i=1}^n$. %
Suppose there exists a  function $\wt{\varrho}(n,\|\bv^*\|)\ge0$ such that 
$|\mc{J}(\bar{\bpsi})|\ge \wt{\varrho}(n,\|\bv^*\|)$
for all $\bar{\bpsi}\in \sph^{\s-1}$, then for every $\delta\in(0,1)$, \begin{align*}
    \P\left(
    \forall \bar{\bpsi}\in \sph^{\s-1},\;
    \sup_{\bar{\bnu} \in \sph^{d-\s-1}} \min_{i\in\mc{J}(\bar{\bpsi})} Y_i \bN_i^\top \bar{\bnu} \le \mathfrak{m}\right)\ge 1-\delta\,.
\end{align*}
where \begin{align*}
\mathfrak{m}(n,d,\s,\|\bv^*\|)=\frac{\sqrt{n}+\sqrt{d-\s} +\sqrt{2\log (1/\delta)}}{\sqrt{\wt{\varrho}(n, \|\bv^*\|)}}\,.   
\end{align*}
    
\end{lemma}
\begin{proof}Recall we have  $\bZ_{\bR,i}=\wt{Y}_i\bv^*_{\bR} +\bvareps_{0,i} =(\bS_i, \bN_i)$ with $\bvareps_{0,i}\sim \mc{N}(\bzero, \bI_d)$, so
\begin{align*}
    (\bN_i)_{i=1}^n\stackrel{\text{i.i.d.}}{\sim} \mc{N}(\bzero, \bI_{d-\s})\,.
\end{align*}
    Collect   $\bN_1,\dots, \bN_n$   into a matrix $\mathbf{N}_{1:n}\in \bbR^{n\times (d-\s)}$,
    then $\mathbf{N}_{1:n}$ is a standard Gaussian matrix, with the standard spectral-norm upper bound  %
    \begin{align*}
        \P\left(\|\mathbf{N}_{1:n}\|\le \sqrt{n}+\sqrt{d-\s}+\sqrt{2\log(1/\delta)}\right)\ge 1-\delta\,.
    \end{align*}
    In the rest of the proof, we work on this event. 

    Fix $\bar{\bpsi}\in \sph^{\s-1}$ and write $\mc{J}\equiv\mc{J}(\bar{\bpsi})\subset[n]$ for simplicity. Fix  $\bar{\bnu}\in \sph^{d-\s-1}$, and stack $(Y_i\bN_i^\top \bar{\bnu})_{i\in \mc{J} }$ into a vector $\bA\in \bbR^{|\mc{J}|}$.  If $\min_{i\in \mc{J}} \bA_i\le0$, then since $\mathfrak{m}\ge0$, it holds   that  $\min_{i\in \mc{J}} \bA_i\le \mathfrak{m}$. If   $\min_{i\in \mc{J}} \bA_i>0$, then by Cauchy-Schwarz
    \begin{align}
        \|\bA \|\ge \sqrt{|\mc{J} |}\min_{i\in \mc{J} } \bA_i,\label{eq:aJ_norm}
    \end{align}
    and note also since  $Y_1, \dots, Y_n\in\{\pm1\}$ does not change Euclidean norms, $\|\bA\|=\|\mathbf{N}_{\mc{J}} \bar{\bnu}\|$ where $\mathbf{N}_{\mc{J}}\in \bbR^{ |\mc{J}|\times (d-\s)}$   is the submatrix of $\mathbf{N}_{1:n}$ formed by the rows indexed by  $\mc{J}$, a subset that  does not itself depend on $\mathbf N_{1:n}$.  Since \eqref{eq:aJ_norm} holds for every $\bar{\bnu}\in \sph^{d-\s-1}$, we have 
    \begin{align*}
        \sup_{\bar{\bnu}\in \sph^{d-\s-1}} \min_{i\in\mc{J}} \bA_i
        &\le \sup_{\bar{\bnu}\in \sph^{d-\s-1}}\frac{\|\mathbf{N}_{\mc{J}}\bar{\bnu}\|}{\sqrt{|\mc{J}|}}\stackrel{\text{(i)}}{\le}
         \frac{\|\mathbf{N}_{\mc{J}}\|}{\sqrt{|\mc{J}|}}
         \stackrel{\text{(ii)}}{\le}
        \frac{\|\mathbf{N}_{1:n}\|}{\sqrt{|\mc{J}|}}\\
        &\stackrel{\text{(iii)}}{\le} \frac{\sqrt{n}+\sqrt{d-\s} + \sqrt{2\log (1/\delta)}}{\sqrt{\wt{\varrho}(n, \|\bv^*\|)}}=:\mathfrak{m}\,,
    \end{align*}
    where step (i) is due to $\|\bar{\bnu}\|=1$, step (ii) holds because   $\|\mathbf{N}_{\mc{J}}\| \le \|\mathbf{N}_{1:n}\| $     for any random subset $\mc{J}(\bar{\bpsi})\subset[n]$ that is itself  independent of $\mathbf{N}_{1:n}$, and   step (iii) is by assumption  $|\mc{J}|\ge\wt{\varrho}(n, \|\bv^*\|)$. Hence, the bound in the last display applies simultaneously to  all $\mc{J}(\bar{\bpsi})\subset[n]$, and hence to  all $\bar{\bpsi}\in \sph^{\s-1}$.
\end{proof}

\subsection{Proof of Theorem \ref{thm:interpolation}}\label{appx:proof_interpolation_thm}
Combining the lemmas above, we prove the main theorem.
\interpolationthm*
\begin{proof}Fix      $\delta\in(0,1)$. The proof follows four steps.

\paragraph{Step 1:} We first recall the key  consequences of Lemmas~\ref{lem:emp_bad_margin} and \ref{lem:repair_bad_margin}. Fix $\C\in(0,1)$. For each $\bar{\bpsi}\in \sph^{\s-1}$, define $\mc{J}(\bar{\bpsi}):=\{i\in[n]: Y_i\bS_i^\top \bar{\bpsi}\le -\tfrac{\C}{2}\}$. 
    Under the sample size condition \eqref{eq:n_large_interpolation_thm} and since $\|\bv^*\|=\|\bSigma^{1/2}\bw^*\|=\|\bw^*\|_{\bSigma}\asymp 1$,   by Lemma~\ref{lem:emp_bad_margin} applied with    $\delta/2$ instead of $\delta$,   
\begin{align}\label{eq:emp_bad_margin_assumption}
    \P\Big(\forall \bar{\bpsi}\in \sph^{\s-1},\;  |\mc{J}(\bar{\bpsi})| \ge \tfrac{n}{2}\varrho(\|\bv^*\|)\Big)\ge 1-\delta/2\,,
\end{align}
where $\varrho(\|\bv^*\|)$   is defined as in Lemma~\ref{lem:emp_bad_margin} and     $\varrho(\|\bv^*\|)\asymp 1$ given $\|\bv^*\|\asymp 1$. Define 
 \begin{align*}
\mathfrak{m}(n,d,\s,\|\bv^*\|)=\frac{\sqrt{n}+\sqrt{d-\s} +\sqrt{2\log (2/\delta)}}{\sqrt{\frac{n}{2}{\varrho}( \|\bv^*\|)}}\,.
\end{align*}

Let $E_2:= \{\|\mathbf{N}_{1:n}\| \le \sqrt{n}+\sqrt{d-k} + \sqrt{2\log (2/\delta)}\}$, the spectral-norm event in the proof of Lemma~\ref{lem:repair_bad_margin} with $\delta$ replaced with $\delta/2$, so $\mathbb{P}(E_2) \ge 1-\delta/2$. The proof of Lemma~\ref{lem:repair_bad_margin} shows that on $E_2$ the bound below holds for every $\bar\bpsi$ with $|\mc{J}(\bar{\bpsi})|\ge \frac{n}{2}\varrho(\|\bv^*\|)$. Hence, on $E_1\cap E_2 $, where $E_1$ is the event inside the probability in \eqref{eq:emp_bad_margin_assumption}, 
 \begin{align}
     \forall \bar{\bpsi} \in \sph^{k-1}, \quad \sup_{\bar\bnu\in \sph^{d-k-1}} \min_{i\in \mc{J}(\bar\bpsi)} Y_i \bN_i^\top \bar\bnu \le   \mathfrak{m}\,.
 \end{align}

\paragraph{Step 2:} Define $\mc{I}_0:=\{ \bar{\bv} \in \sph^{d-1}: \min_{i\in[n]} Y_i\bZ_{\bR,i}^\top\bar{\bv}>0 \}$. The second step of the proof is to establish
\begin{align}\label{eq:overlap_bound}
    \P\left(\forall \bar{\bv}\in\mc{I}_0,\; \frac{\bar{\bv}^\top \bv_{\bR}^*}{\|\bv^*_{\bR}\|}\le {1}/{\sqrt{1+\tfrac{\C^2}{4\mathfrak{m}^2} } }
    \right)\ge 1-\delta\,.
\end{align}
 Let $E_1$ and $E_2$ denote the events defined above.
 We work on $E_1\cap E_2$ with $\P(E_1\cap E_2)\ge 1-\delta$. Fix any  $\bar{\bv} \in \mc{I}_0$ and write its signal-nuisance decomposition as $\bar{\bv}=(\bpsi, \bnu)$ with $\bpsi\in\bbR^{\s}$  and $\bnu\in \bbR^{d-\s}$. If $\bpsi=\bzero$, then $\bar{\bv}^\top \bv^*_{\bR}/\|\bv^*_{\bR}\|=0$ so the claimed upper bound in \eqref{eq:overlap_bound} is trivial. Thus,  we focus on $\bpsi\ne\bzero$ in the rest of the proof, and denote  
    $\bar{\bpsi}=\bpsi/\|\bpsi\|\in \sph^{\s-1}$. 
    
 On $E_1$, the classifier $\bar{\bv}=(\bar{\bpsi}, \bzero)\in \sph^{d-1}$  incorrectly classifies many samples with large negative margins.  Therefore, for a different, general classifier $\bar{\bv}=(\bpsi, \bnu)\in \sph^{d-1}$ to interpolate, we must  use $\bnu$ to fix the negative margins caused by $\bpsi$ (or $\bar{\bpsi}$), that is, we must have for every $i\in\mc{J}(\bar{\bpsi})$, 
    \begin{align*}
    0<Y_i\bZ_{\bR,i}^\top \bar{\bv} = Y_i\bS^\top_i\bpsi+ Y_i \bN_i^\top\bnu =  Y_i\bS^\top_i\bar{\bpsi}\|\bpsi\|+ Y_i \bN_i^\top\bnu.
    \end{align*}
    On $E_1$,  the last display  implies  
    $\tfrac{\C}{2}\|\bpsi\|< \min_{i\in \mc{J}(\bar{\bpsi})}Y_i \bN_i^\top  \bnu$. If $\bnu=\bzero$, this forces $\bpsi=\bzero$, contradicting our assumption that $\bpsi\ne\bzero$. So $\bnu\ne \bzero$, and we denote $\bar{\bnu}=\bnu/\|\bnu\|\in \sph^{d-\s-1}$. Now given $E_2$ holds, we can write
    \begin{align*}
        \tfrac{\C}{2}\frac{\|\bpsi\|}{\|\bnu\|}< \min_{i\in \mc{J}(\bar{\bpsi})}Y_i \bN_i^\top \bar{\bnu}\le \mathfrak{m}.
    \end{align*}
    Together with  $\|\bar{\bv}\|^2=\|\bpsi\|^2+\|\bnu\|^2=1$, we can solve for $\|\bpsi\|$ to obtain  $\|\bpsi\|^2\le \tfrac{4\mathfrak{m}^2}{\C^2}(1-\|\bpsi\|^2) \implies \|\bpsi\|\le 1/\sqrt{1+\tfrac{\C^2}{4\mathfrak{m}^2} }$. We complete the proof of \eqref{eq:overlap_bound} by noticing 
    \begin{align*}
        \frac{\bar{\bv}^\top\bv^*_{\bR}}{\|\bv^*_{\bR}\|}= \bpsi^\top \frac{\bpsi^*}{\|\bpsi^*\|}\le \|\bpsi\|.
    \end{align*}
    In the rest of the proof, we work on the event in \eqref{eq:overlap_bound}.

    \paragraph{Step 3:}

     The next step is to  transform   from the rotated-whitened coordinates to the original coordinates. 
To this end, 
for each $\bar{\bv}\in \mc{I}_0$ define $\bw:=\bSigma^{-1/2}\bR^\top \bar{\bv}$.  Then for every sample $i$, we have 
\begin{align*}
   & Y_i\bZ^\top_{\bR,i}\bar{\bv} = Y_i(\bR\bSigma^{-1/2}\bX_i)^\top \bar{\bv} = Y_i \bX_i^\top (\bSigma^{-1/2}\bR^\top \bar{\bv})= Y_i\bX_i^\top \bw\,,
\end{align*}
which implies 
that  the two interpolator sets $\mc{I}_0$ and $\mc{I}$ have a one-to-one correspondence (up to scaling) where
\begin{align*}
    \bar{\bv}\in \mc{I}_0\Longleftrightarrow \bw=\bSigma^{-1/2}\bR^\top \bar{\bv}\in \mc{I}\,.
\end{align*}
 Moreover, note that $\|\bw\|_{\bSigma}=\|\bar{\bv}\|=1$, and so
 \begin{align*}
      & \frac{\bar{\bv}^\top \bv^*_{\bR}}{\|\bv^*_{\bR}\|} = \frac{(\bR\bSigma^{1/2}\bw)^\top (\bR\bSigma^{1/2}\bw^*)}{\|\bR\bSigma^{1/2}\bw^*\|} = \frac{\bw^\top \bSigma \bw^*}{\|\bw^*\|_{\bSigma}} = \frac{\bw^\top \bSigma \bw^*}{\|\bw\|_{\bSigma}\|\bw^*\|_{\bSigma}} =\cos\theta\,,
 \end{align*}where  $\theta:=\angle_{\bSigma}(\bw, \bw^*)$.
This implies that the following events are equivalent
\begin{align*}
\Bigg\{\forall \bar{\bv}\in\mc{I}_0,\; \frac{\bar{\bv}^\top \bv_{\bR}^*}{\|\bv^*_{\bR}\|}\le {1}/{\sqrt{1+\tfrac{\C^2}{4\mathfrak{m}^2} } }\Bigg\}=\Bigg\{\forall \bw\in \mc{I}, \;\cos\theta \le 1/\sqrt{1+\tfrac{\C^2}{4\mathfrak{m}^2}}\Bigg\}   \,. 
\end{align*}

 \paragraph{Step 4:} The final step is to apply the calibration result from Lemma~\ref{lem:angle_calibration}, which lower bounds  the excess zero-one   risk in terms of  $1-\cos\theta$. First, note that if  $d-\s\ge \C_2n $ for a sufficiently large $\C_2>0$,  then   $\mathfrak{m}=\frac{\sqrt{n}+\sqrt{d-\s} + \sqrt{2\log(2/\delta)}}{\sqrt{\frac{n}{2}\varrho(\|\bv^*\|)}}\ge 
  \frac{1 + \sqrt{\C_2 }+ \sqrt{\frac{2}{n}\log(2/\delta)}}{ \sqrt{\frac{1}{2}\varrho(\|\bv^*\|)}} \ge 3$. Using the inequality   $1-1/\sqrt{1+\frac{1}{a}}\ge \frac{3}{8a}$ for all $a\ge3$, we obtain that 
    \begin{align*}
  1-\cos\theta \ge    1-{1}/{\sqrt{1+\tfrac{\C^2}{4\mathfrak{m}^2} }}\ge \tfrac{3}{32}\tfrac{\C^2}{\mathfrak{m}^2}=\tfrac{3\C^2}{64} \cdot \frac{n \varrho(\|\bv^*\|)}{\left(\sqrt{n}+\sqrt{d-\s} + \sqrt{2\log(2/\delta)}\right)^2}\,.
    \end{align*}
 
Applying Lemma~\ref{lem:angle_calibration}, we have
\begin{align*}
      \mc{E}(\bw)-\mc{E}(\bw^*)
      &\ge\tfrac{1}{\sqrt{2\pi}} (1-2p)
   (1-\cos\theta) \norm{\bv^*}\exp (-\tfrac{1}{2}\norm{\bv^*}^2)\\
   & \ge  \tfrac{3\C^2}{64\sqrt{2\pi}} (1-2p)
      \frac{n \varrho(\|\bv^*\|)}{\left(\sqrt{n}+\sqrt{d-\s} + \sqrt{2\log(2/\delta)}\right)^2} \norm{\bv^*}\exp (-\tfrac{1}{2}\norm{\bv^*}^2).
\end{align*}
Since $\|\bv^*\|=\|\bSigma^{1/2}\bw^*\|=\|\bw^*\|_{\bSigma}\asymp 1$ and so $\varrho(\|\bv^*\|)\asymp 1$, we obtain that there exists a constant $\C_3>0$ such that 
       \begin{align*}
      \mc{E}(\bw)-\mc{E}(\bw^*)
   & \ge  \C_3 \cdot (1-2p)\cdot \frac{n     }{\left(\sqrt{n}+\sqrt{d-\s} + \sqrt{2\log(2/\delta)}\right)^2}\,.
       \end{align*}
\end{proof}

\end{document}